%% file: main.tex
\documentclass[11pt]{asiresearch}
\usepackage[T1]{fontenc}
\usepackage{lmodern}
\usepackage{XCharter}
\usepackage{algorithm}
\usepackage{algpseudocode}
\usepackage{needspace}

\DeclareMathOperator{\KL}{KL}
\DeclareMathOperator{\UKL}{UKL}
\DeclareMathOperator{\sg}{sg}

\newcommand{\pistar}{\pi^\star_{t,k}}
\newcommand{\pitrainer}{\pi^{\mathrm{trainer}}}
\newcommand{\pisampler}{\pi^{\mathrm{sampler}}}
\newcommand{\pisamp}{\pisampler_{\theta_{t-k}}}
\newcommand{\score}{\nabla_\theta\log\pi_\theta}
\title{\LARGE\bfseries On KL-Regularized Policy Optimization}
\author{\textbf{Yifan Zhang}\\[1.5mm]Princeton University\\[0.5mm]\texttt{yifzhang@princeton.edu}}
\date{September 18, 2026\footnote{Revised: October 5, 2026}}
\hypersetup{pdftitle={On KL-Regularized Policy Optimization},pdfauthor={Yifan Zhang}}

\begin{document}
\maketitle
\begin{abstract}
Asynchronous reinforcement learning (RL) for large language model (LLM) agents trains one policy on trajectories generated by another: rollouts come from stale checkpoints, and the inference engine's probabilities differ from the trainer's even at identical parameters. Standard remedies either clip importance ratios, which biases the update, or, as in GRPO, sample a group of responses per prompt, which is costly when episodes are long. We propose KL-Regularized Policy Optimization (KLPO), a framework that anchors the KL regularizer at the sampler. The regularized improvement step then has a closed-form Gibbs solution, and KLPO fits its log-ratio optimality condition by least squares on the sampler's own trajectories, so the sampler probability enters through a log-ratio and no importance weights are needed. Profiling out the regression intercept replaces the intractable log-partition function with the signal's sampler mean plus a sampler-to-trainer KL divergence. For token-level policy mirror descent targets, we show that the resulting gradient can be computed from terminal returns without a critic, via sampler-centered scores or a single trajectory residual, even under stochastic tool outputs. We further prove that independent Monte Carlo estimates of the KL term keep these gradients unbiased, derive the exact KL gap of cheaper top-$K$ and binary approximations, and show that SPPO, GPO, REBEL, and BPO arise as special cases of KLPO. The result is a critic-free update that uses one rollout per prompt and requires neither a learned normalizer nor a group of responses.
\end{abstract}

\projectpage{https://github.com/yifanzhang-pro/KLPO}

\input{contents/report}

\bibliographystyle{plainnat}
\vspace{5ex}

\bibliography{reference}

\clearpage
\appendix

\renewcommand{\appendixpagename}{\centering \huge Appendix}
\appendixpage
\counterwithin{theorem}{section} 

\startcontents[section]
\begingroup
\footnotesize
\printcontents[section]{l}{1}{\setcounter{tocdepth}{1}}
\endgroup
\clearpage

\input{contents/appendix}

\end{document}

%% file: contents/report.tex
\section{Introduction}
\label{sec:introduction}
Reinforcement learning (RL) for language agents is asynchronous by necessity. Episodes last minutes, interleave tool calls, and finish at unpredictable times, so workers keep generating with the last checkpoint they received while the learner trains on trajectories as they arrive. Every update therefore involves two policies: the trainer $\pitrainer_{\theta_t}$ that the learner improves at step $t$, and the sampler $\pisamp$ that generated the data, a checkpoint $k$ steps old. They differ for two reasons. Staleness makes the lag $k$ positive and different for every trajectory, with no fixed bound when tool latency varies. Execution makes their probabilities differ even at $k=0$, because the inference engine applies its own sampling transforms, quantization, and numerical kernels. Agentic training is off-policy by construction, and the trainer--sampler gap is a permanent feature of the system.

Most policy-optimization objectives are written for a single policy and treat the second as a correction. Token-level importance ratios grow more extreme with lag and accumulate along long trajectories, so they are clipped, as in PPO \citep{schulman2017ppo}, which biases the update. A learned critic adds a second model that must track a moving policy. Group-relative estimators such as GRPO \citep{shao2024deepseekmath} avoid the critic by sampling several responses per prompt, which is costly when a single episode is long.

We build the objective around both policies by anchoring the KL regularizer at the sampler, the policy that collected the data. The regularized improvement step then has the closed-form Gibbs solution that underlies REPS, MPO, policy mirror descent, and RPG \citep{peters2010reps,abdolmaleki2018mpo,lan2023pmd,rpg2026}. Its logarithm is an equation between the trainer-to-sampler log-ratio and an improvement signal, which we fit by least squares on the sampler's own trajectories. The log-ratio that an importance-weighted method must exponentiate and clip becomes the regression variable, so the gradient carries no importance weight, and the regression target is exactly the policy selected by the regularized objective. Because only the probabilities the sampler actually used enter, execution mismatch is covered along with staleness, and replay across checkpoints becomes a mixture of well-posed regressions, one per sampler version. We call the resulting family of objectives KL-Regularized Policy Optimization (KLPO); Figure~\ref{fig:overview} traces its token-level instance from objective to update.

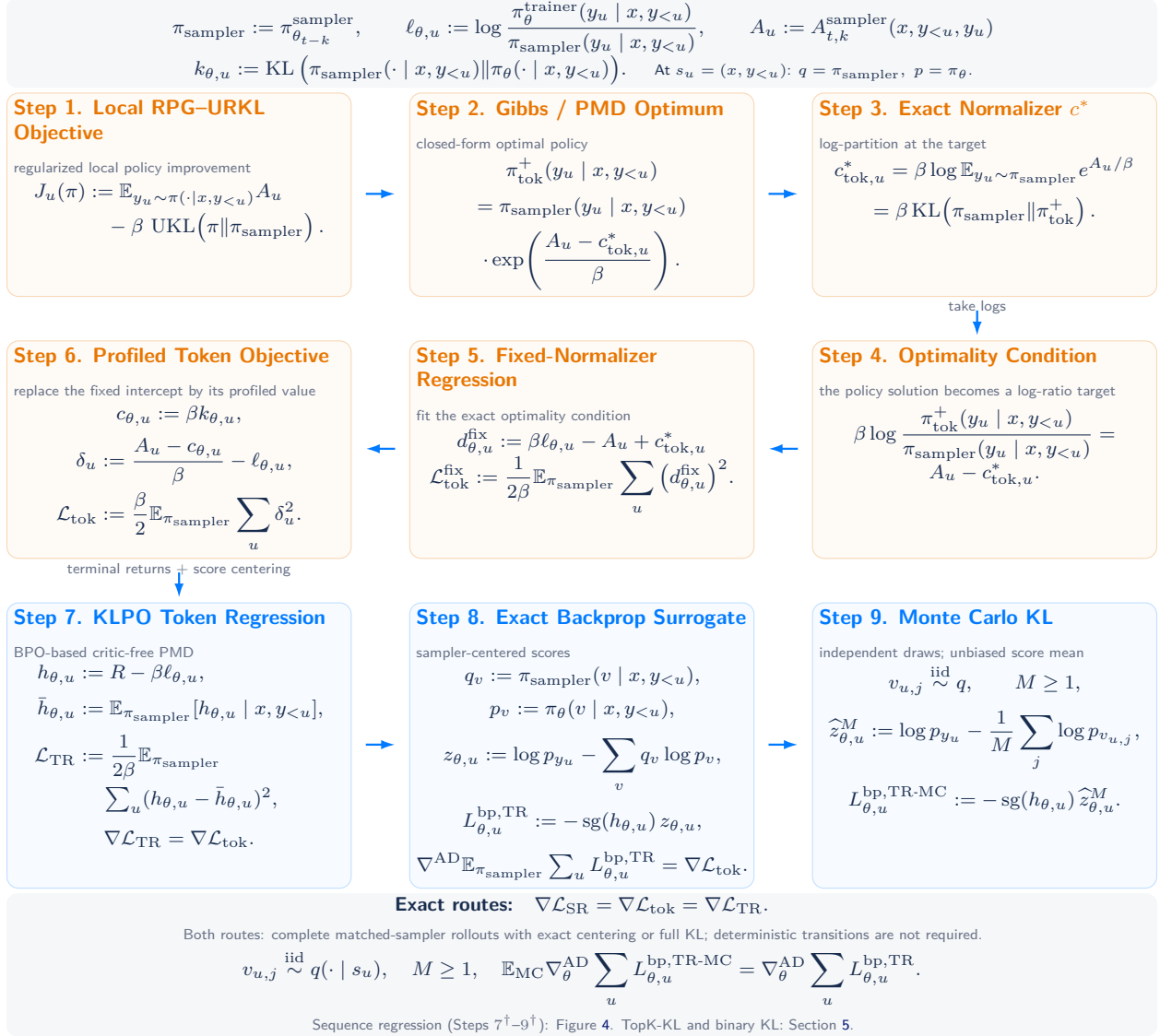
\begin{figure}[!t]
 \centering
 \input{figures/overview}
 \caption{\textbf{Token-level KLPO, from objective to critic-free update.} Steps~1--4 solve the local KL-regularized improvement problem anchored at the sampler \citep{rpg2026} and rewrite its solution as a log-ratio condition. Steps~5--6 fit that condition by regression, as in SPPO and GPO \citep{sppo2024,gpo2025}, and profile the unknown normalizer into a local KL. Steps~7--9 compute the gradient from terminal returns: Step~7 builds on BPO's critic-free reformulation of PMD \citep{bpo2026}, Step~8 centers the trainer's score under the sampler \citep{scorecentering2026}, and Step~9 replaces the exact score mean with $M\geq1$ independent sampler draws. Figure~\ref{fig:score-alternative} gives the sequence-regression alternative to Steps~7--9.}
 \label{fig:overview}
\end{figure}

Our contributions are as follows.
\begin{itemize}
 \item \textbf{Regression and profiling} (Section~\ref{sec:framework}). For any signal and granularity, regressing the log-ratio optimality condition identifies the sampler-anchored Gibbs target without importance weights; its gradient is the regularized-return gradient with the likelihood ratio removed. Profiling the intercept replaces the intractable log-partition function with the signal's sampler mean plus a sampler-to-trainer KL divergence.
 \item \textbf{A common view of prior objectives} (Section~\ref{sec:instances}). With exact normalizers, SPPO and GPO are response-level KLPO, REBEL is its pairwise profiled form, and BPO's exact trajectory loss equals the profiled token-level objective. Each practical loss differs from its KLPO instance by an identified approximation.
 \item \textbf{Critic-free token-level gradients} (Section~\ref{sec:granularity}). The gradient of the profiled token-level policy mirror descent objective follows from terminal returns, either by centering the trainer's score under the sampler or from one baseline-free trajectory residual, and both routes hold with stochastic tool outputs.
 \item \textbf{Estimators for asynchronous training} (Sections~\ref{sec:kl-estimators} and~\ref{sec:agentic}). Independent Monte Carlo estimates of the local KL keep both gradients unbiased, and top-$K$ and binary approximations fall short of the KL by exactly the divergence inside the merged tail. Replay across versions is a fixed mixture of per-version regressions, so token-level KLPO trains from one rollout per prompt with no critic, learned normalizer, or response group.
\end{itemize}

\section{Background}
\label{sec:klpo}
\label{sec:setup}
We index optimizer steps by $t$ and positions within a response by $u$. Each response is generated by a single checkpoint, whose version $v_i$ is stored with it, so at step $t$ it has lag $k_i(t)=t-v_i$; reuse increases the lag but leaves the sampler $\theta_{v_i}$ unchanged. The learner thus holds the trainer $\pitrainer_{\theta_t}$ and receives data from samplers $\pisamp$ with lags $0\leq k\leq K_t$, and synchronous training is the case $k=0$ with matched execution. Sampler probabilities are those actually used to generate the data, including temperature and other sampling transforms, and are logged at collection time. The sampler also serves as the reference distribution of the KL regularizer, so no separate reference model is needed, and the trainer is the only policy that is differentiated. Every quantity below is built from this pair (Figure~\ref{fig:asynchronous}).

\begin{figure}[t]
 \centering
 \input{figures/asynchronous}
 \caption{\textbf{Two policies in asynchronous agentic RL.} Workers generate complete trajectories with stale checkpoints $\pisampler_{\theta_{t-k}}$, on runtimes whose probabilities can differ from the trainer's, and return them with their collection version, sampler log-probabilities, and the records the KL estimator needs. The learner updates $\pitrainer_{\theta_t}$ by fitting each trajectory against its own collection sampler.}
 \label{fig:asynchronous}
\end{figure}
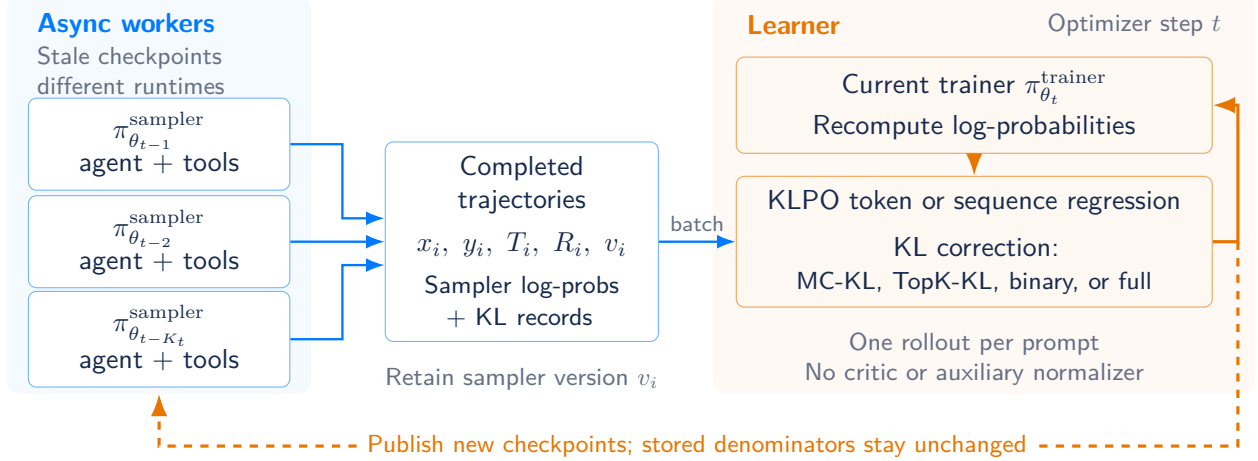

Sections~\ref{sec:framework}--\ref{sec:kl-estimators} fix one collection version $(t,k)$, write $\pi_\theta$ for $\pitrainer_\theta$, and evaluate gradients at $\theta_t$; Section~\ref{sec:agentic} combines versions. A prompt $x\sim\mathcal D$ is followed by a complete response $y=(y_1,\ldots,y_T)\in\mathcal Y_x$, including its termination, with $\pi(y\mid x)=\prod_{u=1}^{T}\pi(y_u\mid x,y_{<u})$ for both policies. We write $s_u=(x,y_{<u})$ for the history before token $u$ and $a_u=y_u$ for the token. In agentic tasks, $s_u$ also contains every tool output observed before $y_u$, and $T$ counts only the tokens generated by the policy.

\begin{assumption}[Finite Positive Model]
\label{ass:finite}
The vocabulary is finite and responses have bounded length. The reward $R(x,y)$ is a bounded deterministic function of the prompt and the complete response, such as a verifier score. The trainer and the sampler are normalized and positive on a common feasible support.
\end{assumption}
\begin{assumption}[Matched Collection]
\label{ass:matched}
For a fixed collection version, prompts are drawn from $\mathcal D$ and complete responses from $\pisamp$, without outcome-dependent filtering. Every sampler-dependent quantity refers to that same sampler: the denominator of each log-ratio, the distribution used for centering or KL estimation, and the continuation that produces each return.
\end{assumption}
\begin{assumption}[Deterministic Transitions and Terminal Reward]
\label{ass:deterministic}
The next history is determined by the current history and action, and the reward is received only at termination.
\end{assumption}
Assumptions~\ref{ass:finite} and~\ref{ass:matched} hold throughout and ensure that the sampler visits every feasible history and action. Assumption~\ref{ass:deterministic} is used only for samplewise identities that telescope token-level quantities along a trajectory; plain text generation satisfies it, and with tools it requires tool outputs to be deterministic functions of the history. A target policy is \emph{realizable} if the model class $\{\pi_\theta\}$ contains it. Statements about exact fits assume realizability; gradient identities do not.

\input{contents/framework}

\input{contents/granularity}

\input{contents/estimators}

\input{contents/agentic}

\section{Related Work}
\label{sec:related}
\paragraph{KL-regularized policy improvement.}
Lemma~\ref{lem:common-target} is the classical solution of a relative-entropy-regularized improvement step. REPS, MPO, and AWR \citep{peters2010reps,abdolmaleki2018mpo,awr2019} fit a parametric policy to it by weighted maximum likelihood, a forward-KL projection; policy mirror descent \citep{lan2023pmd} analyzes repeated steps; and RPG \citep{rpg2026} maximizes the regularized return directly, a reverse-KL projection weighted by the trainer-to-sampler ratio. KLPO fits the same target through its log-ratio condition. The three projections share the unrestricted optimum and differ within a restricted model class (Appendices~\ref{app:rpg} and~\ref{app:related}).

\paragraph{Log-ratio regression and preference optimization.}
DPO \citep{rafailov2023dpo} reads $\beta$ times a policy log-ratio as an implicit reward. SPPO \citep{sppo2024} and GPO \citep{gpo2025} regress log-ratios onto win rates and preference scores, and REBEL \citep{rebel2024} regresses log-ratio differences onto reward differences. Section~\ref{sec:instances} places all three within KLPO through the signal $f$ rather than a scalar reward, since general preferences can be cyclic.

\paragraph{Token-level PMD and score centering.}
BPO \citep{bpo2026} starts from token-level PMD with sampler advantages. Its Equation~(25) derives the full-KL gradient of the exact trajectory loss before linearizing the residual, and its practical loss adds group normalization, a binary KL, smoothing, and clipping; our sequence regression keeps the residual, drops the prompt value, and allows random tool outputs. Score centering \citep{scorecentering2026} subtracts the sampler mean of the trainer's score and multiplies by the reward. KLPO token regression applies the same correction \citep[Equation~(4)]{scorecentering2026} with coefficient $R-\beta\ell_{\theta,u}$, and its top-$K$ variant \citep[Section~4.1]{scorecentering2026} is the derivative of TopK-KL.

\paragraph{Off-policy and critic-free RL for language models.}
Clipped importance ratios \citep{schulman2017ppo} and group-relative advantages \citep{shao2024deepseekmath} are the standard ways to handle stale data and to avoid a critic. FlashREINFORCE \citep{flashreinforce2026} studies critic-free, single-rollout asynchronous training with token importance sampling and batch reward centering. KLPO changes the objective so that the sampler probability enters a log-ratio rather than a multiplicative weight.

\paragraph{Soft values and implicit Q-functions.}
Regularized MDPs \citep{geist2019regularized} and soft Q-learning \citep{haarnoja2017softq} characterize KL-regularized optimal policies through Bellman equations, and IQL \citep{kostrikov2022iql}, IQ-Learn \citep{garg2021iqlearn}, and \citet{rafailov2024qfunction} extract policies from value functions or read policies as value functions. At an exact fit, the response-level KLPO log-ratio is the optimal sampler-regularized soft advantage (Appendix~\ref{sec:qstar}); token-level KLPO fits ordinary sampler advantages and learns no value function.

\section{Conclusion}
\label{sec:discussion}
We introduced KLPO, which anchors the KL regularizer at the sampler and fits the resulting Gibbs target by regressing its log-ratio optimality condition on the sampler's own trajectories, without importance weights. Profiling the intercept turns the unknown normalizer into a sampler mean plus a sampler-to-trainer KL divergence, and at the token level only the local KL remains. The profiled token-level gradient then follows from terminal returns, through sampler-centered scores or a single trajectory residual; both routes hold under random tool outputs, and independent Monte Carlo estimates of the KL keep them unbiased. The same principle recovers SPPO, GPO, REBEL, and BPO as special cases. Token-level KLPO trains from a single rollout per prompt, with no critic, learned normalizer, or group of responses, which makes it a natural fit for asynchronous agentic training. Next steps include large-scale empirical evaluation, finite-sample and convergence analysis, and value learning for stochastic environments (Appendix~\ref{app:learning-soft-targets}).

\section*{Acknowledgement}

We used large language models to assist with the proofs and to improve the clarity and readability of this manuscript.

%% file: figures/overview.tex
\begingroup
\input{figures/style}

\tikzset{
  klpo fig flow/.style={
    -{Latex[length=2.0mm,width=1.35mm]},
    draw=klpoBlue,
    line width=.85pt
  }
}
\newcommand{\klpostepheading}[1]{%
  {\raggedright\sffamily\bfseries\fontsize{8.5}{10}\selectfont #1\par}%
}
\tcbset{
  klpo theory step/.style={
    enhanced,colback=klpoOrange!3,colframe=klpoOrange!30,
    boxrule=.4pt,arc=1.6mm,left=.9mm,right=.9mm,
    top=.9mm,bottom=.9mm,boxsep=0pt,before skip=0pt,after skip=0pt,
    valign=top
  },
  klpo practice step/.style={
    enhanced,colback=klpoBlue!3,colframe=klpoBlue!25,
    boxrule=.4pt,arc=1.6mm,left=.9mm,right=.9mm,
    top=.9mm,bottom=.9mm,boxsep=0pt,before skip=0pt,after skip=0pt,
    valign=top
  }
}

\begin{tcolorbox}[
  enhanced,colback=klpoInk!4,colframe=klpoInk!4,boxrule=0pt,
  arc=1.6mm,left=1.2mm,right=1.2mm,top=.45mm,bottom=.45mm,boxsep=0pt,
  before skip=0pt,after skip=0pt]
\centering\sffamily\scriptsize\color{klpoInk}
\(\displaystyle
\pi_{\mathrm{sampler}}:=\pisamp,\qquad
\ell_{\theta,u}:=\log\frac{\pitrainer_\theta(y_u\mid x,y_{<u})}{\pi_{\mathrm{sampler}}(y_u\mid x,y_{<u})},\qquad
A_u:=A^{\mathrm{sampler}}_{t,k}(x,y_{<u},y_u)
\)
\\[-.2mm]
\(\displaystyle
k_{\theta,u}:=\KL\big(\pi_{\mathrm{sampler}}(\cdot\mid x,y_{<u})\|\pi_\theta(\cdot\mid x,y_{<u})\big).
\)
\quad {\tiny At \(s_u=(x,y_{<u})\): \(q=\pi_{\mathrm{sampler}},\ p=\pi_\theta\).}
\end{tcolorbox}

\vspace{.7mm}

\noindent
\begin{minipage}[c]{.30\linewidth}
\begin{tcolorbox}[klpo theory step,equal height group=klpo-overview-row-one]
\klpostepheading{\color{klpoOrange}Step 1. Local RPG--URKL Objective}
{\sffamily\tiny\color{klpoGray}regularized local policy improvement}\par
\vspace{.55mm}
\centering\scriptsize\color{klpoInk}
\(\displaystyle
\begin{aligned}
J_u(\pi)
&:=\EE_{y_u\sim\pi(\cdot\mid x,y_{<u})} A_u\\[-1pt]
&\quad-\beta\,\UKL\!\left(
\pi\|\pi_{\mathrm{sampler}}\right).
\end{aligned}
\)
\end{tcolorbox}
\end{minipage}%
\hfill
\begin{minipage}[c]{.035\linewidth}
\centering
\begin{tikzpicture}[baseline=-.5ex]
\draw[klpo fig flow] (0,0)--(.42,0);
\end{tikzpicture}
\end{minipage}%
\hfill
\begin{minipage}[c]{.30\linewidth}
\begin{tcolorbox}[klpo theory step,equal height group=klpo-overview-row-one]
\klpostepheading{\color{klpoOrange}Step 2. Gibbs / PMD Optimum}
{\sffamily\tiny\color{klpoGray}closed-form optimal policy}\par
\vspace{.55mm}
\centering\scriptsize\color{klpoInk}
\(\displaystyle
\begin{gathered}
\pi^+_{\mathrm{tok}}(y_u\mid x,y_{<u})\\[1pt]
=\pi_{\mathrm{sampler}}(y_u\mid x,y_{<u})\\[1pt]
\cdot\exp\!\left(
\frac{A_u-c^*_{\mathrm{tok},u}}{\beta}
\right).
\end{gathered}
\)
\end{tcolorbox}
\end{minipage}%
\hfill
\begin{minipage}[c]{.035\linewidth}
\centering
\begin{tikzpicture}[baseline=-.5ex]
\draw[klpo fig flow] (0,0)--(.42,0);
\end{tikzpicture}
\end{minipage}%
\hfill
\begin{minipage}[c]{.30\linewidth}
\begin{tcolorbox}[klpo theory step,equal height group=klpo-overview-row-one]
\klpostepheading{\color{klpoOrange}Step 3. Exact Normalizer \(c^*\)}
{\sffamily\tiny\color{klpoGray}log-partition at the target}\par
\vspace{.45mm}
\centering\scriptsize\color{klpoInk}
\(\displaystyle
\begin{gathered}
c^*_{\mathrm{tok},u}
=\beta\log\EE_{y_u\sim\pi_{\mathrm{sampler}}}
  e^{A_u/\beta}\\[1pt]
=\beta\KL\!\left(
\pi_{\mathrm{sampler}}\|\pi^+_{\mathrm{tok}}
\right).
\end{gathered}
\)
\end{tcolorbox}
\end{minipage}\par

\vspace{.2mm}
\noindent\hfill
\begin{minipage}[c]{.30\linewidth}
\centering
{\sffamily\tiny\color{klpoGray}take logs}\\[-1.1mm]
\begin{tikzpicture}
\draw[klpo fig flow] (0,.12)--(0,-.24);
\end{tikzpicture}
\end{minipage}
\vspace{.2mm}

\noindent
\begin{minipage}[c]{.30\linewidth}
\begin{tcolorbox}[klpo theory step,equal height group=klpo-overview-row-two]
\klpostepheading{\color{klpoOrange}Step 6. Profiled Token Objective}
{\sffamily\tiny\color{klpoGray}replace the fixed intercept by its profiled value}\par
\vspace{.35mm}
\centering\scriptsize\color{klpoInk}
\(\displaystyle
\begin{gathered}
c_{\theta,u}:=\beta k_{\theta,u},\\[1pt]
\begin{aligned}
\delta_u
&:=\frac{A_u-c_{\theta,u}}{\beta}-\ell_{\theta,u},\\[-1pt]
\mathcal L_{\mathrm{tok}}
&:=\frac{\beta}{2}\EE_{\pi_{\mathrm{sampler}}}\sum_u\delta_u^2.
\end{aligned}
\end{gathered}
\)
\end{tcolorbox}
\end{minipage}%
\hfill
\begin{minipage}[c]{.035\linewidth}
\centering
\begin{tikzpicture}[baseline=-.5ex]
\draw[klpo fig flow] (.42,0)--(0,0);
\end{tikzpicture}
\end{minipage}%
\hfill
\begin{minipage}[c]{.30\linewidth}
\begin{tcolorbox}[klpo theory step,equal height group=klpo-overview-row-two]
\klpostepheading{\color{klpoOrange}Step 5. Fixed-Normalizer Regression}
{\sffamily\tiny\color{klpoGray}fit the exact optimality condition}\par
\vspace{.5mm}
\centering\scriptsize\color{klpoInk}
\(\displaystyle
d^{\mathrm{fix}}_{\theta,u}
:=\beta\ell_{\theta,u}-A_u+c^*_{\mathrm{tok},u}
\)
\\[-.2mm]
\(\displaystyle
\mathcal L_{\mathrm{tok}}^{\mathrm{fix}}
:=\frac{1}{2\beta}\EE_{\pi_{\mathrm{sampler}}}
\sum_u\big(d^{\mathrm{fix}}_{\theta,u}\big)^2.
\)
\end{tcolorbox}
\end{minipage}%
\hfill
\begin{minipage}[c]{.035\linewidth}
\centering
\begin{tikzpicture}[baseline=-.5ex]
\draw[klpo fig flow] (.42,0)--(0,0);
\end{tikzpicture}
\end{minipage}%
\hfill
\begin{minipage}[c]{.30\linewidth}
\begin{tcolorbox}[klpo theory step,equal height group=klpo-overview-row-two]
\klpostepheading{\color{klpoOrange}Step 4. Optimality Condition}
{\sffamily\tiny\color{klpoGray}the policy solution becomes a log-ratio target}\par
\vspace{.75mm}
\centering\scriptsize\color{klpoInk}
\(\displaystyle
\beta\log
\frac{\pi^+_{\mathrm{tok}}(y_u\mid x,y_{<u})}
     {\pi_{\mathrm{sampler}}(y_u\mid x,y_{<u})}
=A_u-c^*_{\mathrm{tok},u}.
\)
\end{tcolorbox}
\end{minipage}\par

\vspace{.2mm}
\noindent
\begin{minipage}[c]{.30\linewidth}
\centering
{\sffamily\tiny\color{klpoGray}terminal returns + score centering}\\[-1.1mm]
\begin{tikzpicture}
\draw[klpo fig flow] (0,.12)--(0,-.24);
\end{tikzpicture}
\end{minipage}
\hfill\mbox{}
\vspace{.2mm}

\noindent
\begin{minipage}[c]{.30\linewidth}
\begin{tcolorbox}[klpo practice step,equal height group=klpo-overview-row-three]
\klpostepheading{\color{klpoBlue}Step 7. KLPO Token Regression}
{\sffamily\tiny\color{klpoGray}BPO-based critic-free PMD}\par
\vspace{.3mm}
\centering\scriptsize\color{klpoInk}
\(\displaystyle
\begin{gathered}
\begin{aligned}
h_{\theta,u}&:=R-\beta\ell_{\theta,u},\\[1pt]
\bar h_{\theta,u}&:=\EE_{\pi_{\mathrm{sampler}}}[h_{\theta,u}\mid x,y_{<u}],\\[1pt]
\mathcal L_{\mathrm{TR}}&:=\frac{1}{2\beta}\EE_{\pi_{\mathrm{sampler}}}\\[-1pt]
&\quad\ \textstyle\sum_u(h_{\theta,u}-\bar h_{\theta,u})^2,
\end{aligned}\\[2pt]
\nabla\mathcal L_{\mathrm{TR}}
=\nabla\mathcal L_{\mathrm{tok}}.
\end{gathered}
\)
\end{tcolorbox}
\end{minipage}%
\hfill
\begin{minipage}[c]{.035\linewidth}
\centering
\begin{tikzpicture}[baseline=-.5ex]
\draw[klpo fig flow] (0,0)--(.42,0);
\end{tikzpicture}
\end{minipage}%
\hfill
\begin{minipage}[c]{.30\linewidth}
\begin{tcolorbox}[klpo practice step,equal height group=klpo-overview-row-three]
\klpostepheading{\color{klpoBlue}Step 8. Exact Backprop Surrogate}
{\sffamily\tiny\color{klpoGray}sampler-centered scores}\par
\vspace{.3mm}
\centering\scriptsize\color{klpoInk}
\(\displaystyle
\begin{gathered}
q_v:=\pi_{\mathrm{sampler}}(v\mid x,y_{<u}),\\[1pt]
p_v:=\pi_\theta(v\mid x,y_{<u}),\\[2pt]
z_{\theta,u}:=\log p_{y_u}-\sum_vq_v\log p_v,\\[2pt]
L^{\mathrm{bp,TR}}_{\theta,u}:=-\sg(h_{\theta,u})\,z_{\theta,u},\\[2pt]
\nabla^{\mathrm{AD}}\EE_{\pi_{\mathrm{sampler}}}\textstyle\sum_uL^{\mathrm{bp,TR}}_{\theta,u}
=\nabla\mathcal L_{\mathrm{tok}}.
\end{gathered}
\)
\end{tcolorbox}
\end{minipage}%
\hfill
\begin{minipage}[c]{.035\linewidth}
\centering
\begin{tikzpicture}[baseline=-.5ex]
\draw[klpo fig flow] (0,0)--(.42,0);
\end{tikzpicture}
\end{minipage}%
\hfill
\begin{minipage}[c]{.30\linewidth}
\begin{tcolorbox}[klpo practice step,equal height group=klpo-overview-row-three]
\klpostepheading{\color{klpoBlue}Step 9. Monte Carlo KL}
{\sffamily\tiny\color{klpoGray}independent draws; unbiased score mean}\par
\vspace{.2mm}
\centering\scriptsize\color{klpoInk}
\(\displaystyle
\begin{gathered}
v_{u,j}\overset{\mathrm{iid}}{\sim}q,\qquad M\geq1,\\[2pt]
\widehat z^M_{\theta,u}:=\log p_{y_u}
-\frac1M\sum_j\log p_{v_{u,j}},\\[2pt]
L^{\mathrm{bp,TR\text{-}MC}}_{\theta,u}
:=-\sg(h_{\theta,u})\,\widehat z^M_{\theta,u}.
\end{gathered}
\)
\end{tcolorbox}
\end{minipage}\par

\vspace{.7mm}

\begin{tcolorbox}[
  enhanced,colback=klpoInk!4,colframe=klpoInk!4,boxrule=0pt,
  arc=1.6mm,left=1.2mm,right=1.2mm,top=.45mm,bottom=.45mm,boxsep=0pt,
  before skip=0pt,after skip=0pt]
\centering\sffamily\scriptsize\color{klpoInk}
{\bfseries Exact routes:}\quad
\(\displaystyle
\nabla\mathcal L_{\mathrm{SR}}
=\nabla\mathcal L_{\mathrm{tok}}
=\nabla\mathcal L_{\mathrm{TR}}.
\)
\\[.8mm]
{\sffamily\tiny\color{klpoGray}Both routes: complete matched-sampler rollouts with exact centering or full KL; deterministic transitions are not required.}
\\[.8mm]
\(\displaystyle
v_{u,j}\overset{\mathrm{iid}}{\sim}q(\cdot\mid s_u),\quad M\geq1,\quad
\EE_{\mathrm{MC}}\nabla_\theta^{\mathrm{AD}}\sum_uL^{\mathrm{bp,TR\text{-}MC}}_{\theta,u}
=\nabla_\theta^{\mathrm{AD}}\sum_uL^{\mathrm{bp,TR}}_{\theta,u}.
\)
\\[.8mm]
{\sffamily\tiny\color{klpoGray}Sequence regression (Steps \(7^\dagger\)--\(9^\dagger\)): Figure~\ref{fig:score-alternative}. TopK-KL and binary KL: Section~\ref{sec:kl-estimators}.}
\end{tcolorbox}
\endgroup

%% file: figures/asynchronous.tex
\begingroup
\input{figures/style}
\resizebox{\linewidth}{!}{%
\begin{tikzpicture}[x=1cm,y=1cm,>=Latex,font=\sffamily\small]
 \path[use as bounding box] (0,0) rectangle (16,6.15);
 \fill[klpoBlue!4,rounded corners=5pt] (0,1.05) rectangle (3.9,6.15);
 \fill[klpoOrange!4,rounded corners=5pt] (9.05,1.05) rectangle (16,6.15);
 \node[klpo heading,text=klpoBlue,font=\sffamily\bfseries\small]
   at (.25,5.78) {Async workers};
 \node[klpo note,anchor=west,align=left,text width=3.3cm]
   at (.25,5.2) {Stale checkpoints\\different runtimes};
 \foreach \yy/\idx/\lag in {4.25/1/1,3.0/2/2,1.75/3/K_t} {
   \node[klpo box,text width=3.0cm,minimum height=.82cm]
     (worker\idx) at (1.95,\yy)
     {$\pisampler_{\theta_{t-\lag}}$\quad agent + tools};
 }
 \node[klpo box,text width=3.15cm,minimum height=1.55cm]
   (replay) at (6.6,3.0)
   {Completed\\trajectories\\[4pt]
    $x_i,\ y_i,\ T_i,\ R_i,\ v_i$\\[4pt]
    {\footnotesize Sampler log-probs\\+ KL records}};
 \draw[klpo flow] (worker1.east)--(4.3,4.25)--(4.3,3.3)
   --(replay.west |- {0,3.3});
 \draw[klpo flow] (worker2.east)--(replay.west);
 \draw[klpo flow] (worker3.east)--(4.3,1.75)--(4.3,2.7)
   --(replay.west |- {0,2.7});
 \node[klpo note,anchor=north,align=center,text width=3.5cm]
   at ([yshift=-2mm]replay.south) {Retain sampler version $v_i$};
 \node[klpo heading,text=klpoOrange,font=\sffamily\bfseries\small]
   at (9.35,5.78) {Learner};
 \node[klpo note,anchor=east] at (15.7,5.78) {Optimizer step $t$};
 \node[klpo fit,text width=5.75cm,minimum height=1.0cm]
   (current) at (12.4,4.75)
   {Current trainer $\pitrainer_{\theta_t}$\\[3pt]Recompute log-probabilities};
 \node[klpo fit,text width=5.75cm,minimum height=1.55cm]
   (fit) at (12.4,3.0)
   {KLPO token or sequence regression\\[5pt]
    KL correction:\\{\footnotesize MC-KL, TopK-KL, binary, or full}};
 \draw[klpo flow] (replay.east)--node[above,klpo note,font=\sffamily\scriptsize]
   {batch}(fit.west);
 \draw[klpo update] (current.south)--(fit.north);
 \node[klpo note,align=center,text width=5.8cm] at (12.4,1.5)
   {One rollout per prompt\\No critic or auxiliary normalizer};
 \draw[klpo update] (fit.east)--(15.78,3.0)--(15.78,4.75)--(current.east);
 \draw[klpo update,dashed] (15.78,3.0)--(15.78,.38)
   --(1.95,.38)--(1.95,1.05);
 \node[fill=white,inner sep=3pt,klpo note,text=klpoOrange]
   at (8.85,.38) {Publish new checkpoints; stored denominators stay unchanged};
\end{tikzpicture}%
}
\endgroup

%% file: contents/framework.tex
\section{KL-Regularized Policy Optimization}
\label{sec:framework}
KLPO is built from one primitive, a KL-regularized improvement step taken from the policy that collected the data. We state it generically: a context $h$ is followed by an outcome $z$ drawn from a positive sampler conditional $\mu(z\mid h)$ on a finite support, such as a prompt and a complete response or a history and the next token. An improvement signal $f(h,z)$, such as a reward, an advantage, or a win rate, scores outcomes, and $\beta>0$ is fixed. Proofs for this section are in Appendix~\ref{app:common-proof}.

\subsection{Sampler-Anchored Improvement}
\label{sec:improvement}
The improvement step maximizes the expected signal while penalizing divergence from the sampler,
\begin{equation}
 \max_{\pi}\ \EE_{z\sim\pi}f(h,z)-\beta\KL\big(\pi\,\|\,\mu(\cdot\mid h)\big),
 \label{eq:framework-objective}
\end{equation}
over distributions $\pi$ on the support. This is the unnormalized reverse-KL (URKL) objective of RPG \citep{rpg2026} with the sampler as reference, the policy-improvement step of REPS and MPO \citep{peters2010reps,abdolmaleki2018mpo}, and one iteration of entropic policy mirror descent \citep{lan2023pmd}. Its solution is classical.

\begin{lemma}[Gibbs Optimality Condition]
\label{lem:common-target}
The unique maximizer of Equation~\eqref{eq:framework-objective} is
\begin{equation}
 \pi_f^\star(z\mid h)=\mu(z\mid h)e^{(f(h,z)-c_f(h))/\beta},\qquad
 c_f(h)=\beta\log\EE_{z\sim \mu(\cdot\mid h)}e^{f(h,z)/\beta},
 \label{eq:common-target}
\end{equation}
and for every distribution $\pi$ the objective equals $c_f(h)-\beta\KL(\pi\,\|\,\pi_f^\star(\cdot\mid h))$.
\end{lemma}
Taking logarithms, the maximizer satisfies one equation per outcome:
\begin{equation}
 \beta\log\frac{\pi_f^\star(z\mid h)}{\mu(z\mid h)}=f(h,z)-c_f(h).
 \label{eq:framework-condition}
\end{equation}
The left side involves only policy probabilities, and the right side is the signal shifted by a context-dependent constant. KLPO fits this equation directly.

\subsection{Regressing the Optimality Condition}
\label{sec:regression-principle}
For a trainable policy $\pi_\theta$, the residual of Equation~\eqref{eq:framework-condition} is
\[
 r_\theta(h,z)=\beta\log\frac{\pi_\theta(z\mid h)}{\mu(z\mid h)}-f(h,z)+c_f(h)
 =\beta\log\frac{\pi_\theta(z\mid h)}{\pi_f^\star(z\mid h)}.
\]
\begin{theorem}[Optimality-Condition Regression]
\label{thm:common-regression}
Let $\nu$ be a replay law over pairs $(h,z)$ and $a>0$ a weight, both held fixed during differentiation. Then
\begin{equation}
 \mathcal L_f=\frac1{2\beta}\EE_\nu[a\,r_\theta^2],\qquad
 \nabla_\theta\mathcal L_f=\EE_\nu[a\,r_\theta\nabla_\theta\log \pi_\theta].
 \label{eq:common-regression}
\end{equation}
If $\nu$ charges every pair in the support and $\pi_f^\star$ is realizable, then $\mathcal L_f=0$ exactly when $\pi_\theta=\pi_f^\star$ on the support.
\end{theorem}
Any replay law that covers the support identifies the same target, with no importance multiplier in the gradient. The regularized-return gradient differs by a single factor: differentiating Equation~\eqref{eq:framework-objective} at $\pi=\pi_\theta$ and rewriting the expectation under the sampler gives
\begin{equation}
 -\nabla_\theta\Big[\EE_{\pi_\theta}f-\beta\KL(\pi_\theta\,\|\,\mu)\Big]
 =\EE_{z\sim \mu}\Big[\frac{\pi_\theta(z\mid h)}{\mu(z\mid h)}\,r_\theta(h,z)\,\nabla_\theta\log \pi_\theta(z\mid h)\Big],
 \label{eq:framework-return-gradient}
\end{equation}
which is the regression gradient with $\nu=\mu$ and $a=1$, multiplied per sample by the likelihood ratio. The two coincide at $\pi_\theta=\mu$; within a restricted model class their minimizers can differ, because the return minimizes $\KL(\pi_\theta\|\pi_f^\star)$ and the regression minimizes $\EE_{\mu}[\log^2(\pi_\theta/\pi_f^\star)]$ (Appendix~\ref{app:rpg}).

\subsection{Profiling the Normalizer}
\label{sec:profiling}
The residual contains the log-partition $c_f(h)$, an average of $e^{f/\beta}$ over every outcome, which requires the signal for all outcomes or several outcomes at the same context. We instead treat the intercept as a free parameter and minimize over it. With
$\ell_\theta(h,z)=\log(\pi_\theta(z\mid h)/\mu(z\mid h))$, note that $\KL(\mu(\cdot\mid h)\|\pi_\theta(\cdot\mid h))=-\EE_{z\sim \mu}\ell_\theta(h,z)$.
\begin{proposition}[Profiled Regression]
\label{prop:profiling}
Fix a context $h$. Over $c\in\mathbb R$, the loss $\EE_{z\sim \mu}[(\beta\ell_\theta-f+c)^2]$ is minimized by
\begin{equation}
 \blueeqbox{\displaystyle
 c_\theta(h)=\EE_{z\sim \mu}f(h,z)+\beta\KL\big(\mu(\cdot\mid h)\,\|\,\pi_\theta(\cdot\mid h)\big)},
 \label{eq:general-profile}
\end{equation}
so the profiled loss equals $\Var_{\mu}(f-\beta\ell_\theta)/(2\beta)$. Its gradient is that of Theorem~\ref{thm:common-regression} with $c_\theta$ in place of $c_f$, because the profiled residual has zero sampler mean and the derivative of $c_\theta$ cancels. If $\pi_f^\star$ is realizable, the profiled loss vanishes exactly at $\pi_\theta=\pi_f^\star$, where $c_\theta=c_f$. In particular,
\begin{equation}
 c_f(h)=\EE_{z\sim \mu}f(h,z)+\beta\KL\big(\mu(\cdot\mid h)\,\|\,\pi_f^\star(\cdot\mid h)\big).
 \label{eq:normalizer-kl}
\end{equation}
\end{proposition}
Equation~\eqref{eq:normalizer-kl} writes the log-partition at the target as a sampler mean plus a KL divergence, and profiling replaces the unknown target in that divergence by the current trainer. The fixed-normalizer and profiled losses share their realizable minimizer, though their gradients differ away from it. For a signal with zero sampler mean, such as an advantage, the profiled intercept is exactly $\beta\KL(\mu\|\pi_\theta)$.

\subsection{Instances}
\label{sec:instances}
An instance of KLPO is a choice of context, outcome, and signal. We develop two (Table~\ref{tab:granularity} in Appendix~\ref{app:token-supplement} compares them).

\paragraph{Response level (SKLPO).}
With a prompt $x$ as context, a complete response $y$ as outcome, and $f=R$, the target is the response Gibbs policy $\pistar(y\mid x)\propto\pisamp(y\mid x)e^{R(x,y)/\beta}$ and the residual is $d_\theta=\beta\ell_\theta-R+c^*(x)$, where $\ell_\theta=\sum_u\ell_{\theta,u}$ sums the token log-ratios and $c^*(x)=\beta\log\EE_{\pisamp}[e^{R/\beta}\mid x]$. We call this regression SKLPO (sequence-level KLPO) and develop it in Appendix~\ref{sec:sequence}. Its profiled intercept, $\EE_{\pisamp}[R\mid x]+\beta\KL(\pisamp(\cdot\mid x)\|\pi_\theta(\cdot\mid x))$, needs the prompt's mean reward and a complete-response KL, which a single response cannot supply. This is the weak point of SKLPO for asynchronous training: the normalizer depends on the prompt and the sampler version, and profiling within a group fails for a group of one (Section~\ref{sec:klpo-sklpo}).

\paragraph{Token level (local PMD).}
With a history $s$ as context, the next token $a$ as outcome, and $f=A^{\mathrm{sampler}}_{t,k}$, the target is the policy mirror descent (PMD) step $\pi^+_{\mathrm{tok}}(a\mid s)\propto\pisamp(a\mid s)e^{A^{\mathrm{sampler}}_{t,k}(s,a)/\beta}$, in which the sampler is both the KL reference and the policy whose advantages are improved, the starting point of BPO \citep{bpo2026}. Advantages have zero sampler mean, so the profiled intercept is the local KL $\beta\KL(\pisamp(\cdot\mid s)\|\pi_\theta(\cdot\mid s))$, computable from the two next-token distributions. The residual still contains the advantage, which would ordinarily require a critic; Section~\ref{sec:granularity} obtains the gradient from terminal returns instead. The two targets differ: the token-level targets multiply to the response Gibbs policy only when the accumulated local normalizers are constant across responses (Appendix~\ref{app:path-normalizer}), and regressing against optimal soft values instead of sampler advantages recovers the response-level target (Appendix~\ref{sec:soft-token}).

\paragraph{Prior objectives as instances.}
Several existing objectives are KLPO instances once their population objective is separated from their practical loss (Table~\ref{tab:instances}). SPPO \citep{sppo2024} and GPO \citep{gpo2025} are response-level KLPO with a win rate and an expected preference score as signals, REBEL \citep{rebel2024} is the pairwise form of profiled response-level regression (Lemma~\ref{lem:empirical-centering}), and the exact trajectory loss of BPO \citep{bpo2026} equals the profiled token objective (Theorem~\ref{thm:token-bpo}). AWR and RPG \citep{awr2019,rpg2026} share the Gibbs target but fit it differently. Appendices~\ref{app:rpg} and~\ref{app:related} state the matching conditions.

\begin{table}[t]
\centering
\small
\renewcommand{\arraystretch}{1.15}
\begin{tabularx}{\textwidth}{@{}lll>{\raggedright\arraybackslash}X@{}}
\toprule
Objective & Level & Signal & Relation to KLPO \\
\midrule
SPPO & response & win rate & Exact update: exact-normalizer regression; the practical loss sets the normalizer to $1/2$ \\
GPO & response & preference score & Exact update: exact-normalizer regression; the practical loss sets the normalizer to zero \\
REBEL & response & reward & Pairwise form of profiled regression on two responses \\
BPO & token & sampler advantage & Exact trajectory loss equals the profiled token objective under Assumption~\ref{ass:deterministic}; the practical loss linearizes and approximates it \\
Score centering & token & reward & Same centered score, with coefficient $R$ in place of $R-\beta\ell_{\theta,u}$ \\
AWR & either & advantage or reward & Same Gibbs target, fitted by weighted maximum likelihood \\
RPG & response & reward & Same Gibbs target, fitted by maximizing the importance-weighted return \\
\bottomrule
\end{tabularx}
\caption{Prior objectives viewed through KLPO. ``Exact'' refers to a method's population objective before its practical approximations.}
\label{tab:instances}
\end{table}

%% file: contents/granularity.tex
\section{Token-Level KLPO without a Critic}
\label{sec:granularity}
This section follows Figure~\ref{fig:overview}. We write down the profiled token-level objective and then compute its gradient from terminal rewards along two routes: token regression, which centers the trainer's score under the sampler, and sequence regression, which telescopes the token residuals into one trajectory residual. Throughout, $x\sim\mathcal D$ and $Y\sim\pisamp(\cdot\mid x)$ are complete rollouts from one collection version, and
\begin{equation}
 \ell_{\theta,u}=\log\frac{\pi_\theta(a_u\mid s_u)}{\pisamp(a_u\mid s_u)},\qquad
 k_\theta(s)=\KL(\pisamp(\cdot\mid s)\|\pi_\theta(\cdot\mid s)).
 \label{eq:localnotation}
\end{equation}
We write $g_\theta(s,a)=\nabla_\theta\log\pi_\theta(a\mid s)$ for the trainer's score. The sampler's action value, value, and advantage are $Q^{\mathrm{sampler}}_{t,k}(s,a)=\EE_{\pisamp}[R\mid s,a]$, $V^{\mathrm{sampler}}_{t,k}(s)=\EE_{\pisamp}[R\mid s]$, and $A^{\mathrm{sampler}}_{t,k}=Q^{\mathrm{sampler}}_{t,k}-V^{\mathrm{sampler}}_{t,k}$, with expectations over the sampler's continuation. Proofs are in Appendix~\ref{app:main-proofs}.

\subsection{The Profiled Token Objective}
\label{sec:local-objective}
\label{sec:token-target}
\label{sec:token-fixed}
\label{sec:token-centering}
Step~1 of Figure~\ref{fig:overview} applies the RPG--URKL objective \citep{rpg2026} to the next-token distribution at a history $s$:
\begin{equation}
 J_s(\pi)
 =\EE_{a\sim\pi(\cdot\mid s)}A^{\mathrm{sampler}}_{t,k}(s,a)
   -\beta\UKL\big(\pi(\cdot\mid s)\|\pisamp(\cdot\mid s)\big).
 \label{eq:local-rpg-urkl}
\end{equation}
Both policies are normalized, so the unnormalized KL $\UKL(p\|q)=\sum_a[p_a\log(p_a/q_a)+q_a-p_a]$ equals the KL, and this is Equation~\eqref{eq:framework-objective} with $f=A^{\mathrm{sampler}}_{t,k}$. Steps~2--5 solve it by Lemma~\ref{lem:common-target} and fit the logarithmic condition with the fixed-normalizer residual $d^{\mathrm{fix}}_{\theta,u}=\beta\ell_{\theta,u}-A^{\mathrm{sampler}}_{t,k}(s_u,a_u)+c^*_{\mathrm{tok}}(s_u)$, the squared log-ratio construction of SPPO and GPO \citep{sppo2024,gpo2025} moved to tokens. GPO's practical loss drops the normalizer; Step~6 profiles it.

\begin{lemma}[Token PMD Target and Profiled Objective]
\label{prop:token-pmd}
\label{lem:token-centering}
The unique maximizer of Equation~\eqref{eq:local-rpg-urkl} and its normalizer are
\begin{align}
 \pi^+_{\mathrm{tok}}(a\mid s)
 &=\pisamp(a\mid s)e^{(A^{\mathrm{sampler}}_{t,k}(s,a)-c^*_{\mathrm{tok}}(s))/\beta},
 \label{eq:localtarget}\\
 c^*_{\mathrm{tok}}(s)
 &=\beta\log\EE_{a\sim\pisamp(\cdot\mid s)}
     e^{A^{\mathrm{sampler}}_{t,k}(s,a)/\beta}
 =\beta\KL(\pisamp(\cdot\mid s)\|\pi^+_{\mathrm{tok}}(\cdot\mid s)).
 \label{eq:localnormalizer}
\end{align}
Profiling the intercept at each history gives $c_\theta(s)=\beta k_\theta(s)$ and the profiled token objective
\begin{equation}
 \delta_u=
 \frac{A^{\mathrm{sampler}}_{t,k}(s_u,a_u)}{\beta}
 -\ell_{\theta,u}-k_\theta(s_u),\qquad
 \blueeqbox{\displaystyle
 \mathcal L_{\mathrm{tok}}
 =\frac{\beta}{2}\EE_{x,Y\sim\pisamp}\sum_u\delta_u^2}.
 \label{eq:tokenloss}
\end{equation}
Each residual has zero conditional mean, $\EE_{\pisamp}[\delta_u\mid s_u]=0$, and
\begin{equation}
 \nabla_\theta\mathcal L_{\mathrm{tok}}
 =-\beta\EE_{x,Y\sim\pisamp}\sum_u
 \delta_u g_\theta(s_u,a_u).
 \label{eq:tokengradient}
\end{equation}
If $\pi^+_{\mathrm{tok}}$ is realizable, it is the unique zero-loss policy of both $\mathcal L_{\mathrm{tok}}$ and the fixed-normalizer loss $\frac1{2\beta}\EE\sum_u(d^{\mathrm{fix}}_{\theta,u})^2$.
\end{lemma}
The profiled objective needs no normalizer, but it still needs $A^{\mathrm{sampler}}_{t,k}$, the advantage of a policy the trainer has already left. The next two subsections remove it.

\subsection{Token Regression with Sampler-Centered Scores}
\label{sec:token-sc}
Step~7 builds on BPO's critic-free reformulation of PMD \citep{bpo2026} and replaces the sampler advantage with the terminal return. Define the trainer's score centered under the sampler and a per-token return coefficient,
\begin{equation}
 \widetilde g_\theta(s,a)
 =g_\theta(s,a)-\EE_{v\sim\pisamp(\cdot\mid s)}g_\theta(s,v)
 =g_\theta(s,a)+\nabla_\theta k_\theta(s),\qquad
 h_{\theta,u}=R(x,Y)-\beta\ell_{\theta,u},
 \label{eq:token-sc-signal}
\end{equation}
let $\bar h_\theta(s)=\EE_{\pisamp}[h_\theta\mid s]$ over the next action and the sampler's continuation, and define the token-regression loss $\mathcal L_{\mathrm{TR}}(\theta)=\frac1{2\beta}\EE_{x,Y\sim\pisamp}\sum_u(h_{\theta,u}-\bar h_\theta(s_u))^2$.

\begin{theorem}[Critic-Free Token-Regression Gradient]
\label{thm:token-sc}
Under Assumptions~\ref{ass:finite} and~\ref{ass:matched}, with exact conditional score means,
\begin{equation}
 \blueeqbox{\displaystyle
 \mathcal L_{\mathrm{TR}}
 =\mathcal L_{\mathrm{tok}}
 +\frac1{2\beta}\EE_{x,Y\sim\pisamp}\sum_u
 \big(R-Q^{\mathrm{sampler}}_{t,k}(s_u,a_u)\big)^2}.
 \label{eq:token-sc-decomposition}
\end{equation}
The second term does not depend on $\theta$, so
\begin{equation}
 \blueeqbox{\displaystyle
 \nabla_\theta\mathcal L_{\mathrm{tok}}
 =\nabla_\theta\mathcal L_{\mathrm{TR}}
 =-\EE_{x,Y\sim\pisamp}\sum_u
 (R-\beta\ell_{\theta,u})\,\widetilde g_\theta(s_u,a_u)}.
 \label{eq:token-sc-gradient}
\end{equation}
If $\pi^+_{\mathrm{tok}}$ is realizable, it is the unique minimizer of both losses. Subtracting a fixed baseline $b(s_u)$ from the coefficient $R-\beta\ell_{\theta,u}$ leaves the expected gradient unchanged.
\end{theorem}
Terminal returns thus replace the sampler's critic: the continuation noise $R-Q^{\mathrm{sampler}}_{t,k}$ adds variance but no bias, provided the score is centered under the sampler that generated both the action and its continuation. The term $-\beta\ell_{\theta,u}$ makes this a regression rather than a policy gradient. It feeds the current log-ratio back into the coefficient, so the update weakens as the trainer approaches the target; with the reward alone as coefficient, the estimator is the score-centered policy gradient of \citet{scorecentering2026}.

\label{sec:token-surrogate}
Step~8 computes Equation~\eqref{eq:token-sc-gradient} by automatic differentiation, without evaluating $\bar h_\theta$, through the per-token surrogate
\begin{equation}
 L^{\mathrm{bp,TR}}_{\theta,u}
 =-\sg(h_{\theta,u})\,z_{\theta,u},\qquad
 z_{\theta,u}
 =\log\pi_\theta(a_u\mid s_u)
   -\sum_v\pisamp(v\mid s_u)\log\pi_\theta(v\mid s_u).
 \label{eq:token-sc-surrogate}
\end{equation}
With $h_{\theta,u}$ and the sampler probabilities detached, each token contributes $-h_{\theta,u}\widetilde g_\theta(s_u,a_u)$, so the expected surrogate gradient is $\nabla_\theta\mathcal L_{\mathrm{TR}}=\nabla_\theta\mathcal L_{\mathrm{tok}}$ (Proposition~\ref{prop:token-sc-surrogate} in Appendix~\ref{app:token-supplement}). One complete response per prompt suffices, provided the score mean is computed independently of the sampled action.

\subsection{Sequence Regression by Bellman Telescoping}
\label{sec:tokenbpo}
Each token residual has zero conditional mean, so the token loss equals the expected square of the summed residual. Under Assumption~\ref{ass:deterministic}, the advantages in that sum also telescope into the terminal reward, which gives the exact trajectory objective of BPO \citep[Equations~(18)--(24)]{bpo2026} (Figure~\ref{fig:score-alternative} in Appendix~\ref{app:token-supplement}).
\begin{theorem}[Population Equivalence to Exact BPO]
\label{thm:token-bpo}
Under Assumptions~\ref{ass:finite} and~\ref{ass:matched},
\begin{equation}
 \EE_{\pisamp}\big[\big(\textstyle\sum_u\delta_u\big)^2\big]
 =\EE_{\pisamp}\sum_u\delta_u^2.
 \label{eq:orthogonality}
\end{equation}
If Assumption~\ref{ass:deterministic} also holds, then
\begin{equation}
 \sum_u\delta_u
 =\frac{R-V^{\mathrm{sampler}}_{t,k}(x)}{\beta}
  -\sum_u[\ell_{\theta,u}+k_\theta(s_u)]
 =\delta_{\mathrm{BPO}},
 \qquad\text{and}\qquad
 \blueeqbox{\displaystyle
 \mathcal L_{\mathrm{tok}}
 =\frac{\beta}{2}\EE_{\pisamp}[\delta_{\mathrm{BPO}}^2]}.
 \label{eq:tokentobpo}
\end{equation}
\end{theorem}
The prompt value $V^{\mathrm{sampler}}_{t,k}(x)$ in $\delta_{\mathrm{BPO}}$ seems to require a critic, but it does not affect the gradient. Drop it:
\begin{equation}
 D_\theta(x,Y):=R(x,Y)-\beta\sum_u[\ell_{\theta,u}+k_\theta(s_u)],\qquad
 \mathcal L_{\mathrm{SR}}(\theta):=\frac1{2\beta}\EE_{x,Y\sim\pisamp}D_\theta(x,Y)^2.
 \label{eq:bpo-no-baseline}
\end{equation}
\begin{proposition}[Baseline-Free Trajectory Regression]
\label{prop:bpo-no-baseline}
Under Assumptions~\ref{ass:finite} and~\ref{ass:matched}, let $\varepsilon_u=V^{\mathrm{sampler}}_{t,k}(s_{u+1})-Q^{\mathrm{sampler}}_{t,k}(s_u,a_u)$ be the transition noise after token $u$, where $s_{u+1}$ includes any tool output that follows $a_u$ and $V^{\mathrm{sampler}}_{t,k}=R$ at termination. Then
\begin{equation}
 \blueeqbox{\displaystyle
 \mathcal L_{\mathrm{SR}}
 =\mathcal L_{\mathrm{tok}}
 +\frac1{2\beta}\EE\Big[V^{\mathrm{sampler}}_{t,k}(x)^2+\Big(\sum_u\varepsilon_u\Big)^2\Big],
 \qquad
 \nabla_\theta\mathcal L_{\mathrm{SR}}
 =\nabla_\theta\mathcal L_{\mathrm{tok}}}.
 \label{eq:bpo-baseline-identity}
\end{equation}
Under Assumption~\ref{ass:deterministic}, every $\varepsilon_u$ is zero. Replacing $R$ by $R-b(x)$ for a fixed prompt baseline replaces $V^{\mathrm{sampler}}_{t,k}(x)$ by $V^{\mathrm{sampler}}_{t,k}(x)-b(x)$ in the constant.
\end{proposition}
Sequence regression therefore needs no value estimate, no group of responses, and no deterministic transitions: random tool outputs add the noise $\varepsilon_u$ to the residual, but its contribution to the loss does not depend on $\theta$, and Assumption~\ref{ass:deterministic} adds only the samplewise telescope $\sum_u\delta_u=\delta_{\mathrm{BPO}}$. The identity holds in population, so a baseline estimated from the same trajectory does not count as fixed, and $\mathcal L_{\mathrm{SR}}$ itself need not vanish at the PMD target.

\label{sec:bpo-backward}
Step~$8^\dagger$ differentiates Equation~\eqref{eq:bpo-no-baseline}. Writing $g_{\theta,u}=g_\theta(s_u,a_u)$,
\begin{equation}
 \nabla_\theta\mathcal L_{\mathrm{SR}}
 =-\EE_{x,Y\sim\pisamp}D_\theta
   \sum_u[g_{\theta,u}+\nabla_\theta k_\theta(s_u)],
 \label{eq:bpo-exact-gradient}
\end{equation}
which each trajectory produces through the backward surrogate
\begin{equation}
 L^{\mathrm{bp,SR}}_\theta(x,Y)
 :=-\sg(D_\theta)\sum_u
 [\log\pi_\theta(a_u\mid s_u)+k_\theta(s_u)].
 \label{eq:bpo-backward-surrogate}
\end{equation}
Only $D_\theta$ is detached; gradients flow through the KL terms because $g_{\theta,u}+\nabla_\theta k_\theta(s_u)=\widetilde g_\theta(s_u,a_u)$ is the centered score. BPO's Equation~(25) is this full-KL gradient before its later linearization \citep{bpo2026}. Under Assumptions~\ref{ass:finite} and~\ref{ass:matched}, the two routes therefore agree in population,
\begin{equation}
 \blueeqbox{\displaystyle
 \nabla_\theta\mathcal L_{\mathrm{SR}}
 =\nabla_\theta\mathcal L_{\mathrm{tok}}
 =\nabla_\theta\mathcal L_{\mathrm{TR}}},
 \label{eq:bpo-sc-equivalence}
\end{equation}
but they are different estimators: token regression weights each centered score by its own coefficient $R-\beta\ell_{\theta,u}$, and sequence regression weights all of them by the shared residual $D_\theta$. Appendix~\ref{sec:sr-vs-sklpo} contrasts sequence regression with SKLPO, which also squares one residual per response but has a different target.

%% file: contents/estimators.tex
\section{Estimating the Local KL Correction}
\label{sec:kl-estimators}
\label{sec:token-update}
At every visited history, both routes need the sampler mean of the trainer's score, which equals $-\nabla_\theta k_\theta(s)$, and sequence regression also needs $k_\theta(s)$ itself. Both are sums over the sampler's next-token distribution, which is expensive to store for every generated token. We study three cheaper alternatives, each usable with either route: independent Monte Carlo draws, which preserve the gradients in expectation, and top-$K$ and binary aggregations, which change the objective by an amount we compute exactly. At a history $s_u$ we write $q(v)=\pisamp(v\mid s_u)$ and $p(v)=\pi_\theta(v\mid s_u)$.

\subsection{Monte Carlo KL}
\label{sec:mc-kl}
Monte Carlo KL (MC-KL) draws $M$ auxiliary tokens from the sampler at each visited prefix, independently of the rollout and of one another:
\begin{equation}
 \widehat k^M_{\theta,u}:=\frac1M\sum_{j=1}^M\log\frac{q(v_{u,j})}{p(v_{u,j})},
 \qquad v_{u,j}\overset{\mathrm{iid}}{\sim}q.
 \label{eq:mc-kl}
\end{equation}
For a fixed trainer, $\widehat k^M_{\theta,u}$ and its gradient are unbiased for $k_\theta(s_u)$ and $\nabla_\theta k_\theta(s_u)$; a single estimate can be negative, and clipping it at zero destroys the unbiasedness. The draws are single tokens rather than continuations, so they need no further generation or reward.

\paragraph{Token regression (Step~9).}
Given the rollout, the coefficient $h_{\theta,u}$ is independent of the auxiliary draws, so replacing the vocabulary sum in Equation~\eqref{eq:token-sc-surrogate} by the sample mean leaves the expected gradient unchanged for every $M\geq1$:
\begin{equation}
 L^{\mathrm{bp,TR\text{-}MC}}_\theta
 =-\sum_u\sg(h_{\theta,u})\Big[
 \log p(a_u\mid s_u)-\frac1M\sum_j\log p(v_{u,j}\mid s_u)\Big].
 \label{eq:mc-token-backward}
\end{equation}
Independence is essential: if the rollout action served as its own auxiliary draw, the bracket would vanish identically.

\paragraph{Sequence regression (Step~$9^\dagger$).}
Here an unbiased KL estimate is not enough: the residual is squared, so reusing the same draws inside the residual and its derivative correlates the two factors and biases the gradient (Appendix~\ref{app:mc-kl-proof}). Instead, we form one residual $D_j$ per draw index $j$ and weight the score from draw $j$ by the mean residual of the other draws:
\begin{equation}
 \begin{aligned}
 D_j&=R-\beta\sum_u\Big[\ell_{\theta,u}
                   +\log\frac{q(v_{u,j}\mid s_u)}{p(v_{u,j}\mid s_u)}\Big],\qquad
 \overline D_{-j}=\frac1{M-1}\sum_{l\ne j}D_l,\\
 L^{\mathrm{bp,SR\text{-}MC}}_\theta
 &=-\frac1M\sum_j\sg(\overline D_{-j})
       \sum_u[\log p(a_u\mid s_u)-\log p(v_{u,j}\mid s_u)].
 \end{aligned}
 \label{eq:mc-sequence-backward}
\end{equation}
This requires $M\geq2$ draws per prefix, not $M$ responses.

\begin{proposition}[Independent MC-KL Gradients]
\label{prop:mc-gradients}
Fix the trainer and a complete rollout. Averaged over independent auxiliary draws, the gradients of Equations~\eqref{eq:mc-token-backward} and~\eqref{eq:mc-sequence-backward} equal the corresponding full-KL sample gradients. The sequence surrogate is also the gradient of the U-statistic
\begin{equation}
 \widehat{\mathcal L}_{\mathrm{SR,MC}}
 =\frac1{2\beta M(M-1)}\sum_{j\ne l}D_jD_l,
 \qquad
 \EE_{\mathrm{MC}}\widehat{\mathcal L}_{\mathrm{SR,MC}}
 =\frac{D_\theta^2}{2\beta}.
 \label{eq:mc-sequence-ustat}
\end{equation}
Hence the population identities of Section~\ref{sec:granularity} hold after averaging over the auxiliary draws and the rollouts.
\end{proposition}
Since these statements condition on a trainer chosen independently of the draws, a conditionally unbiased gradient at every update uses fresh draws (Appendix~\ref{app:mc-reuse}).

\subsection{TopK-KL and Binary KL}
\label{sec:topk}
\label{sec:bpo-topk}
\label{sec:token-topk}
\label{sec:binary}
\label{sec:bpo-binary}
\label{sec:sc-binary}
Top-$K$ aggregated KL (TopK-KL) keeps the sampler's $K$ most probable tokens at each history, with their unrenormalized probabilities, and merges the rest into one tail category ($K$ is unrelated to $M$ and to the maximum lag $K_t$). With head $H_u$, $q_{\mathrm{tail}}=1-\sum_{v\in H_u}q_v$, and $p_{\mathrm{tail}}=1-\sum_{v\in H_u}p_v$,
\begin{equation}
 k^K_{\theta,u}:=\sum_{v\in H_u}q_v\log\frac{q_v}{p_v}
 +q_{\mathrm{tail}}\log\frac{q_{\mathrm{tail}}}{p_{\mathrm{tail}}}
 \label{eq:topk-kl}
\end{equation}
is the KL between the coarsened distributions, and by the chain rule for KL it falls short of the full KL by exactly the divergence inside the tail:
\begin{equation}
 \blueeqbox{\displaystyle
 k_\theta(s_u)-k^K_{\theta,u}
 =q_{\mathrm{tail}}\KL\big(q(\cdot\mid\mathrm{tail})\,\|\,p(\cdot\mid\mathrm{tail})\big)\geq0}.
 \label{eq:topk-kl-gap}
\end{equation}
TopK-KL is therefore exact when the tail holds a single token, and its gradient models the sampler's tail as proportional to the trainer's (Equation~\eqref{eq:topk-kl-derivative}). Binary KL is the coarsest case: BPO's partition \citep[Equation~(27)]{bpo2026} takes the sampled token as the head, needs nothing beyond the rollout's own records, and corrects the score by the complementary-probability factor $(1-q(a_u))/(1-p(a_u))$ rather than an importance ratio. Appendix~\ref{app:topk-binary} gives the surrogates for both routes (Propositions~\ref{prop:bpo-topk} and~\ref{prop:bpo-binary}).

Both aggregations give up the zero means on which every exact result of Section~\ref{sec:granularity} rests. Because of the tail gap, the corrected log-ratio has sampler mean $k^K_\theta-k_\theta\leq0$ instead of zero, so the update is no longer the gradient of $\mathcal L_{\mathrm{tok}}$, a reward offset can change it, the PMD target need not be stationary, and the two routes can disagree in population (Appendix~\ref{app:token-sc-topk}). Table~\ref{tab:estimators} summarizes the trade-off: MC-KL keeps the exact identities in expectation at the cost of variance and $M$ records per position, TopK-KL is deterministic and nearly exact when the sampler's mass is concentrated, and binary KL is the cheapest and the crudest.

\begin{table}[t]
\centering
\small
\renewcommand{\arraystretch}{1.2}
\begin{tabularx}{\textwidth}{@{}l>{\raggedright\arraybackslash}Xlll@{}}
\toprule
KL correction & Records per generated token & Token route & Sequence route & Exact gradient \\
\midrule
Full KL & the sampler's next-token distribution & Eq.~\eqref{eq:token-sc-surrogate} & Eq.~\eqref{eq:bpo-backward-surrogate} & yes \\
MC-KL & $M$ sampled token IDs and their sampler log-probabilities & Eq.~\eqref{eq:mc-token-backward}, $M\geq1$ & Eq.~\eqref{eq:mc-sequence-backward}, $M\geq2$ & in expectation \\
TopK-KL & top-$K$ token IDs and their sampler probabilities & Eq.~\eqref{eq:token-sc-topk} & Eq.~\eqref{eq:bpo-topk-surrogate} & if the tail is one token \\
Binary KL & nothing beyond the sampled token & Eq.~\eqref{eq:sc-binary-surrogate} & Eq.~\eqref{eq:bpo-binary-surrogate} & no \\
\bottomrule
\end{tabularx}
\caption{Estimators of the local KL correction. Every row also stores the sampled token's own sampler log-probability. ``Exact gradient'' refers to the population identity of Equation~\eqref{eq:bpo-sc-equivalence}.}
\label{tab:estimators}
\end{table}

%% file: contents/agentic.tex
\section{Learning from Lagged Samplers}
\label{sec:agentic}
We return to the asynchronous setting, where a batch $\mathcal B_t=\{(x_i,y_i,T_i,R_i,v_i)\}_{i=1}^{B}$ mixes trajectories from several sampler versions, with response $i$ generated by $\pisampler_{\theta_{v_i}}$ and stored with its version, its sampler log-probabilities, and the records its KL estimator needs. The learner trains on completed trajectories without waiting for a second response to the same prompt (Algorithm~\ref{alg:token-klpo} in Appendix~\ref{app:implementation}).

\subsection{Mixtures of Sampler Versions}
\label{sec:token-batch}
Let $\ell_{i,u,\theta}$, $k_{i,\theta}$, and $\widetilde g_{i,u,\theta}$ be the log-ratio, local KL, and centered score of token $y_{i,u}$, all computed against its own collection sampler $\pisampler_{\theta_{v_i}}$.
\begin{proposition}[Version-Conditioned Updates]
\label{prop:token-batch}
With $h_{i,u,\theta}=R_i-\beta\ell_{i,u,\theta}$ and $D_{i,\theta}=R_i-\beta\sum_u[\ell_{i,u,\theta}+k_{i,\theta}(s_{i,u})]$, the token-regression and sequence-regression batch gradients are
\begin{equation}
 \nabla_\theta\widehat{\mathcal L}^{\mathrm{bp,TR}}_{\mathrm{tok},t}
 =-\frac1B\sum_i\sum_{u=1}^{T_i}
 h_{i,u,\theta}\,\widetilde g_{i,u,\theta},
 \qquad
 \nabla_\theta^{\mathrm{AD}}\widehat{\mathcal L}^{\mathrm{bp,SR}}_t
 =-\frac1B\sum_iD_{i,\theta}\sum_{u=1}^{T_i}\widetilde g_{i,u,\theta}.
 \label{eq:token-mixture-regression}
\end{equation}
If each version contributes complete, unfiltered rollouts from its recorded sampler, both expected gradients equal the gradient of the mixture of per-version objectives $\mathcal L_{\mathrm{tok}}$, weighted by the version frequencies.
\end{proposition}
The proof is in Appendix~\ref{app:batch-proof}, and every KL estimator of Section~\ref{sec:kl-estimators} substitutes into each component. The proposition is Assumption~\ref{ass:matched} applied version by version: sampler quantities must never be mixed across versions, while trainer-dependent ones ($h_{i,u,\theta}$ or $D_{i,\theta}$, and the KL correction) are recomputed at every reuse. Dividing each response's loss by its length breaks the cancellation of cross terms, so the batch loss sums over tokens and averages over responses. Each version defines its own PMD target, so a trainer fitted to several versions settles on a compromise that depends on the distribution of lags, even with unlimited capacity (Proposition~\ref{prop:offpolicy-target} proves the response-level case).

\subsection{Consequences for Asynchronous Training}
\label{sec:no-is}
\paragraph{No ratio to clip.}
The regression gradient is the regularized-return gradient without the trainer-to-sampler ratio (Section~\ref{sec:regression-principle} and Proposition~\ref{prop:token-is}), so KLPO has nothing to clip as checkpoints age or trajectories grow. The sampler probability still enters through the logarithm, so errors in logged probabilities or stored conditionals shift the log-ratio or the KL correction. Appendix~\ref{app:flash} compares KLPO with the importance-weighted update of FlashREINFORCE \citep{flashreinforce2026}.

\paragraph{One rollout per prompt.}
\label{sec:klpo-sklpo}
With full KL or MC-KL, one complete rollout gives an unbiased stochastic gradient of $\mathcal L_{\mathrm{tok}}$ (Theorem~\ref{thm:token-sc} and Proposition~\ref{prop:mc-gradients}), with no value function, prompt normalizer, or second response. SKLPO instead needs a normalizer for every prompt and sampler version, and with one response both natural estimators degenerate (Proposition~\ref{prop:singleton}): self-centering zeroes the residual, and a log-mean-exp estimate of $c^*(x)$ returns the observed reward. Its normalizer must come from independent data or a learned predictor (Appendix~\ref{sec:normalizer}), and an error $\varepsilon(x)$ biases the gradient by $\varepsilon(x)\EE_\nu[g_\theta\mid x]$ (Equation~\eqref{eq:normalizer-gradient-error}). Table~\ref{tab:single-rollout} lists what each instance needs.

\paragraph{Interactive trajectories.}
\label{sec:interactive}
Neither gradient route needs Assumption~\ref{ass:deterministic}. Token regression uses only the sampler's own continuation values, which a complete return from that sampler estimates without bias, and stochastic tool outputs add only a $\theta$-independent term to sequence regression (Theorem~\ref{thm:token-sc} and Proposition~\ref{prop:bpo-no-baseline}). In both routes the return estimates the ordinary $Q^{\mathrm{sampler}}_{t,k}$ of its own version, not an optimal soft Q-function. The response-level target, by contrast, can reweight tool outcomes that the policy cannot control, so no policy may attain it (Appendix~\ref{app:interactive}).

%% file: contents/appendix.tex
\input{contents/sequence}

\input{contents/normalization}
\FloatBarrier

\input{contents/qstar}
\FloatBarrier

\input{contents/token-supplement}
\FloatBarrier

\section{Proofs for the Framework and Token-Level Results}
\label{app:main-proofs}
\subsection{Gibbs Condition, Regression, and Profiling}
\label{app:common-proof}
\label{app:common-principle}
\begin{proof}[Proof of Lemma~\ref{lem:common-target} and Theorem~\ref{thm:common-regression}]
Fix $h$ and drop it from the notation. Let $Z_f=\sum_z\mu(z)e^{f(z)/\beta}$ and $c_f=\beta\log Z_f$, so that $\pi_f^\star=\mu e^{(f-c_f)/\beta}$ is positive and normalized. For any distribution $\pi$,
\[
 \EE_\pi f-\beta\KL(\pi\|\mu)=c_f-\beta\KL(\pi\|\pi_f^\star).
\]
The KL is nonnegative and vanishes only at $\pi=\pi_f^\star$, which proves the lemma. The residual is $r_\theta=\beta\log(\pi_\theta/\pi_f^\star)$, so $\pi_f^\star$ attains zero loss, and positive replay mass and weights force every residual to vanish at any zero-loss fit. Finally, $\nabla_\theta r_\theta=\beta\nabla_\theta\log \pi_\theta$, and differentiating the squared loss with the replay law fixed and the weights held constant yields Equation~\eqref{eq:common-regression}.

For Equation~\eqref{eq:framework-return-gradient}, write $\ell_\theta=\log(\pi_\theta/\mu)$. Since $\EE_{\pi_\theta}\nabla_\theta\log \pi_\theta=0$, the gradient of the regularized return is $\EE_{\pi_\theta}[(f-\beta\ell_\theta)\nabla_\theta\log \pi_\theta]$. Adding the constant $-c_f$ inside the expectation changes nothing and turns the coefficient into $-r_\theta$. Changing the measure from $\pi_\theta$ to $\mu$ introduces the ratio $\pi_\theta/\mu$.
\end{proof}

\begin{proof}[Proof of Proposition~\ref{prop:profiling}]
Let $X=\beta\ell_\theta-f$. For any constant $c$, $\EE_{\mu}[(X+c)^2]=\Var_{\mu}(X)+(c+\EE_{\mu}X)^2$, so the minimizer is
\[
 c_\theta=-\EE_{\mu}X=\EE_{\mu}f-\beta\EE_{\mu}\ell_\theta=\EE_{\mu}f+\beta\KL(\mu\|\pi_\theta),
\]
and the minimum is $\Var_{\mu}(X)$. Differentiating the profiled loss gives
\[
 \nabla_\theta\frac{1}{2\beta}\EE_{\mu}[(X+c_\theta)^2]
 =\EE_{\mu}[(X+c_\theta)\nabla_\theta\log \pi_\theta]
 +\frac{\nabla_\theta c_\theta}{\beta}\EE_{\mu}[X+c_\theta],
\]
and the last expectation is zero by the choice of $c_\theta$. The profiled loss vanishes exactly when $f-\beta\ell_\theta$ is constant on the support, that is, when $\pi_\theta\propto \mu e^{f/\beta}$; normalization then gives $\pi_\theta=\pi_f^\star$. At that policy $X=-c_f$ identically, so $c_\theta=c_f$, and substituting $\pi_\theta=\pi_f^\star$ into Equation~\eqref{eq:general-profile} gives Equation~\eqref{eq:normalizer-kl}.
\end{proof}

\subsection{Token PMD Target and Profiled Objective}
\label{app:token-pmd-proof}
\begin{proof}[Proof of Equations~\eqref{eq:localtarget} and~\eqref{eq:localnormalizer} in Lemma~\ref{lem:token-centering}]
Apply Lemma~\ref{lem:common-target} at each history with $f=A^{\mathrm{sampler}}_{t,k}$. This gives Equations~\eqref{eq:localnormalizer} and~\eqref{eq:localtarget}. The log-ratio of the target then yields
\[
 \begin{aligned}
 \beta\KL(\pisamp(\cdot\mid s)\|\pi^+_{\mathrm{tok}}(\cdot\mid s))
 &=\EE_{a\sim\pisamp(\cdot\mid s)}
       [c^*_{\mathrm{tok}}(s)-A^{\mathrm{sampler}}_{t,k}(s,a)]\\
 &=c^*_{\mathrm{tok}}(s),
 \end{aligned}
\]
because the exact sampler advantage has zero conditional mean. The identity $d^{\mathrm{fix}}_{\theta,u}=\beta\log(\pi_\theta/\pi^+_{\mathrm{tok}})$ and Theorem~\ref{thm:common-regression} give the regression gradient and the unique zero-loss target.
\end{proof}

\label{app:token-centering-proof}
\begin{proof}[Proof of the remaining claims of Lemma~\ref{lem:token-centering}]
This is Proposition~\ref{prop:profiling} with $f=A^{\mathrm{sampler}}_{t,k}$ at each history. Since the advantage has zero sampler mean,
\begin{equation}
 c_\theta(s)=\EE_{a\sim\pisamp}[A^{\mathrm{sampler}}_{t,k}(s,a)-\beta\log(\pi_\theta(a\mid s)/\pisamp(a\mid s))]
 =\beta k_\theta(s).
 \label{eq:token-profilec}
\end{equation}
The zero conditional means of the advantage and of $\ell_{\theta,u}+k_\theta(s_u)$ give $\EE[\delta_u\mid s_u]=0$. Differentiating the profiled loss yields $-\beta\EE\sum_u\delta_u[g_\theta(s_u,a_u)+\nabla_\theta k_\theta(s_u)]$, and the term with $\nabla_\theta k_\theta(s_u)$ vanishes after conditioning on $s_u$, which proves Equation~\eqref{eq:tokengradient}. At zero loss, the local residual is constant over actions, and normalization fixes the intercept to $c^*_{\mathrm{tok}}$ and the policy to $\pi^+_{\mathrm{tok}}$.
\end{proof}

\subsection{Critic-Free Token Regression}
\label{app:token-sc-proof}
\begin{proof}[Proof of Theorem~\ref{thm:token-sc}]
Fix a visited history $s$ and write $q$ for its sampler. Let $g_a=\nabla_\theta\log\pi_\theta(a\mid s)$, $m_\theta=\EE_q g_a$, $h=R-\beta\log(\pi_\theta(a\mid s)/q(a\mid s))$, and $\bar h_\theta(s)=\EE_q[h\mid s]$. The law of the action and its continuation does not depend on $\theta$, so
\[
 \nabla_\theta\frac{\Var_q(h\mid s)}{2\beta}
 =-\EE_q[(h-\EE_q[h\mid s])g_a\mid s]
 =-\EE_q[h(g_a-m_\theta)\mid s].
\]
Summing over the sampler's visitation distribution, which is also fixed, proves Equation~\eqref{eq:token-sc-gradient}. Since $\EE_q[g_a-m_\theta\mid s]=0$, any fixed action-independent $b(s)$ can be subtracted from $h$. A baseline fitted to the same sampled action or continuation would need a separate argument.

Let $\xi=R-Q^{\mathrm{sampler}}_{t,k}(s,a)$. Then $\EE_q[\xi\mid s,a]=0$ and $\bar h_\theta(s)=V^{\mathrm{sampler}}_{t,k}(s)+\beta k_\theta(s)$, so $h-\bar h_\theta=\beta\delta+\xi$. The cross term vanishes after conditioning on $(s,a)$, and
\begin{equation}
 \mathcal L_{\mathrm{TR}}-\mathcal L_{\mathrm{tok}}
 =\frac1{2\beta}\EE_{x,Y\sim\pisamp}\sum_u
 \Var_{\pisamp}(R\mid x,y_{\le u}).
 \label{eq:token-sc-variance-gap}
\end{equation}
The right side does not depend on $\theta$, so Lemma~\ref{lem:token-centering} gives the minimizers; under realizability, the unique one is $\pi^+_{\mathrm{tok}}$. The token-regression loss generally does not vanish there, because terminal returns carry continuation noise.

With the stop-gradients of Equation~\eqref{eq:token-sc-surrogate}, each sample contributes $-h(g_a-m_\theta)$. A single complete rollout therefore gives a stochastic estimate of the population gradient, provided the conditional score mean is supplied independently of the sampled action. The surrogate's value is neither the variance objective nor an estimate of it. Differentiating through $h$ would add a product-rule term, and replacing $h$ by the terminal reward gives the ordinary score-centered policy gradient, without the regression feedback.

The fixed-normalizer and profiled objectives have different gradients away from the PMD target. Score centering computes the gradient of the profiled objective without a critic. It does not evaluate the unknown $c^*_{\mathrm{tok}}$, and the zero-mean identity does not make this gradient an importance-corrected gradient of the current policy's return.
\end{proof}

\subsection{The Backward Surrogate}
\label{app:token-surrogate-proof}
\begin{proof}[Proof of Proposition~\ref{prop:token-sc-surrogate}]
Because $h_{\theta,u}$ and $\pisamp$ are detached,
\[
 \begin{aligned}
 \nabla_\theta^{\mathrm{AD}}L^{\mathrm{bp,TR}}_{\theta,u}
 &=-h_{\theta,u}\left[
 \nabla_\theta\log\pi_\theta(a_u\mid s_u)
 -\sum_v\pisamp(v\mid s_u)\nabla_\theta\log\pi_\theta(v\mid s_u)
 \right]\\
 &=-h_{\theta,u}\widetilde g_\theta(s_u,a_u).
 \end{aligned}
\]
Taking the sampler expectation gives Equation~\eqref{eq:token-sc-gradient}, and Theorem~\ref{thm:token-sc} gives the remaining equalities.
\end{proof}

\subsection{Bellman Telescoping and the Prompt Baseline}
\label{app:bpo-proof}
\begin{proof}[Proof of Theorem~\ref{thm:token-bpo}]
For $v<u$, the residual $\delta_v$ is determined by $(x,y_{<u})$, so
\[
\EE_{\pisamp}[\delta_v\delta_u]
=\EE_{\pisamp}[\delta_v\EE_{\pisamp}(\delta_u\mid x,y_{<u})]=0.
\]
Padding with zero residuals after termination and expanding the square proves Equation~\eqref{eq:orthogonality}. Under Assumption~\ref{ass:deterministic}, $A^{\mathrm{sampler}}_{t,k}(x,y_{<u},y_u)=V^{\mathrm{sampler}}_{t,k}(x,y_{\le u})-V^{\mathrm{sampler}}_{t,k}(x,y_{<u})$, so the advantages along a response sum to $R-V^{\mathrm{sampler}}_{t,k}(x)$. This proves Equation~\eqref{eq:tokentobpo} and the loss identity for every smooth parameterized policy. A fixed prompt weight can be applied to both sides. The trajectory form uses only the terminal reward and the prompt value, whereas the token form uses intermediate advantages; in a finite batch, the cross-terms remain.
\end{proof}

\label{app:bpo-baseline-proof}
\begin{proof}[Proof of Proposition~\ref{prop:bpo-no-baseline}]
Fix a prompt and a collection sampler, and write $V=V^{\mathrm{sampler}}_{t,k}$ and $Q=Q^{\mathrm{sampler}}_{t,k}$. Since $V=R$ at termination,
\[
 R-V(x)=\sum_u\big[V(s_{u+1})-V(s_u)\big]=\sum_u\big[A^{\mathrm{sampler}}_{t,k}(s_u,a_u)+\varepsilon_u\big],
\]
so that
\[
 D_\theta=V(x)+\beta\sum_u\delta_u+\sum_u\varepsilon_u.
\]
Every cross term vanishes in expectation. Both $\delta_u$ and $\varepsilon_u$ have zero mean given $x$, which removes the terms with $V(x)$. For $u\le v$, $\delta_u$ is determined by $(s_v,a_v)$ and $\EE[\varepsilon_v\mid s_v,a_v]=0$; for $u>v$, $\varepsilon_v$ is determined by $s_u$ and $\EE[\delta_u\mid s_u]=0$. With Equation~\eqref{eq:orthogonality}, this gives
\[
 \EE[D_\theta^2]=\beta^2\EE\sum_u\delta_u^2+\EE\Big[V(x)^2+\Big(\sum_u\varepsilon_u\Big)^2\Big],
\]
and dividing by $2\beta$ proves Equation~\eqref{eq:bpo-baseline-identity}. The last expectation involves only the sampler and the environment, so it does not depend on $\theta$. Replacing $R$ by $R-b(x)$ replaces $V(x)$ by $V(x)-b(x)$. Under Assumption~\ref{ass:deterministic}, $V(s_{u+1})=Q(s_u,a_u)$, so every $\varepsilon_u$ is zero. All sampler quantities and $b$ stay fixed during differentiation. Together with Theorem~\ref{thm:token-sc}, this proves Equation~\eqref{eq:bpo-sc-equivalence} for every smooth parameterized trainer, without realizability, which is needed only for the claim about the common target.

Finally, $D_\theta=V(x)+\beta\delta_{\mathrm{BPO}}$ holds by the definition of $\delta_{\mathrm{BPO}}$, so $\delta_{\mathrm{BPO}}=\sum_u\delta_u+\beta^{-1}\sum_u\varepsilon_u$, and the same cross-term argument gives $(\beta/2)\EE[\delta_{\mathrm{BPO}}^2]=\mathcal L_{\mathrm{tok}}+\EE[(\sum_u\varepsilon_u)^2]/(2\beta)$.
\end{proof}

\input{contents/mc-proof}

\subsection{TopK-KL}
\label{app:bpo-topk-proof}
\begin{proof}[Proof of Proposition~\ref{prop:bpo-topk}]
At a fixed history, $\nabla p_{\mathrm{tail}}=-\sum_{v\in H}p_vg_v$. Differentiating Equation~\eqref{eq:topk-kl} with the sampler records fixed gives
\[
 \nabla k^K_\theta
 =-\sum_{v\in H}q_vg_v
 -\frac{q_{\mathrm{tail}}}{p_{\mathrm{tail}}}\nabla p_{\mathrm{tail}}
 =-\sum_{v\in H}(q_v-\rho p_v)g_v,
\]
which is Equation~\eqref{eq:topk-kl-derivative}. With the whole coefficient detached, $\nabla_\theta C_\theta=\sum_{v\in H}(q_v-\rho p_v)g_v=-\nabla_\theta k^K_\theta$. The chain rule gives
\[
 \nabla_\theta\frac{(D^K_\theta)^2}{2\beta}
 =-D^K_\theta\sum_u[g_{\theta,u}+\nabla_\theta k^K_{\theta,u}],
\]
which is the gradient of Equation~\eqref{eq:bpo-topk-surrogate}, and differentiating Equation~\eqref{eq:token-sc-topk} gives the token-route expression in Equation~\eqref{eq:sequence-token-feedback}.
\end{proof}

Equation~\eqref{eq:topk-kl-gap} is the chain rule for KL applied to the partition of the vocabulary into the head tokens and one tail. As a consequence, $\EE_{a\sim q}[\log(p_a/q_a)+k^K_\theta]=k^K_\theta-k_\theta$ is generally nonzero, and the conditional-centering, baseline-invariance, and route-equivalence arguments do not extend to a finite head. A singleton tail gives the full KL, and a singleton head consisting of the sampled action gives the binary partition. These statements concern unfloored probabilities.

\label{app:token-sc-topk}
\paragraph{Gradient error of the top-$K$ score correction.}
At a fixed history, take the head $H$ from the stored sampler with unrenormalized probabilities $q_v$. Let $T=\mathcal V\setminus H$ and $\rho=q_{\mathrm{tail}}/p_{\mathrm{tail}}$, and set $\widehat q_v=q_v$ on $H$ and $\widehat q_v=\rho p_v$ on $T$. Since $\sum_vp_vg_v=0$,
\begin{equation}
 \widehat m_\theta(s):=\sum_v\widehat q_vg_v
 =\sum_{v\in H}(q_v-\rho p_v)g_v,
 \label{eq:token-sc-head-mean}
\end{equation}
so detaching the whole coefficient in Equation~\eqref{eq:token-sc-topk} gives $\nabla_\theta C_\theta=\widehat m_\theta$. Detaching $\rho$ alone is not enough. Token regression keeps the original logged $q_{y_u}$ in $h_{\theta,u}$, even when the sampled action lies outside $H$. With rollouts from $q$, the conditional gradient error of the top-$K$ update is
\begin{equation}
 \EE_q[-h(g_a-\widehat m_\theta)\mid s]
 -\EE_q[-h(g_a-m_\theta)\mid s]
 =\bar h_\theta(s)(\widehat m_\theta(s)-m_\theta(s)).
 \label{eq:token-sc-head-error}
\end{equation}
With a fixed baseline, $\bar h_\theta$ is replaced by $\bar h_\theta-b$. The error vanishes for exact centering, and in particular when the modeled and the actual tails agree. Because it is proportional to $\bar h_\theta$, a finite head makes the update depend on reward offsets. The top-$K$ surrogate is therefore an approximation, not the gradient of the exact variance objective.

\subsection{Binary KL}
\label{app:bpo-binary-proof}
\begin{proof}[Proof of Proposition~\ref{prop:bpo-binary}]
For one sampled action, keep $q$ fixed. Differentiating the Bernoulli KL gives
\[
 \nabla_\theta k^{\mathrm{bin}}
 =\left[-q+\frac{(1-q)p}{1-p}\right]\nabla_\theta\log p,
\]
and adding $\nabla_\theta\log p$ multiplies the score by $(1-q)/(1-p)$. Therefore
\[
 \nabla_\theta\frac{\widehat D_\theta^2}{2\beta}
 =-\widehat D_\theta\sum_u\omega_{\theta,u}\nabla_\theta\log p_u,
\]
which is the gradient of the stopped surrogate. The token surrogate applies the same corrected score with coefficient $h_{\theta,u}$.
\end{proof}
These are identities for the binary objective, not for the full-KL objective. The binary-corrected log-ratio $\log(p_a/q_a)+k^{\mathrm{bin}}(a)$ generally has nonzero mean under $a\sim q$, so the baseline-invariance argument does not extend to binary KL. For example, take $q=(1/3,1/3,1/3)$, $p=(1/2,1/3,1/6)$, and a constant reward. After adding the full KL, the log-ratio has zero sampler mean. The binary-corrected log-ratio has a negative mean, because coarsening lowers the KL and the binary KL is strictly below the full KL for at least one sampled action. A reward offset can therefore change the binary objective's gradient. With a vocabulary of two tokens, the binary partition is the full vocabulary.

\subsection{Mixtures of Sampler Versions}
\label{app:version-proof}
\begin{proof}[Proof of Proposition~\ref{prop:offpolicy-target}]
For each component, $d_\theta=\beta\log(\pitrainer_\theta/\pistar)$. Positive sampling mass and outer weights force that residual to vanish on its support at zero loss. A positive combination of nonnegative component losses is zero exactly when every component is zero, so a shared trainer must equal every component target. If full-support samplers $p_1,p_2$ have the same Gibbs target for a common reward and $\beta$, then $p_1(y)e^{R(y)/\beta}/Z_1=p_2(y)e^{R(y)/\beta}/Z_2$ for every $y$. Their ratio is the constant $Z_1/Z_2$, and normalization makes it one.
\end{proof}

\label{app:batch-proof}
\begin{proof}[Proof of Propositions~\ref{prop:sequence-batch} and~\ref{prop:token-batch}]
For SKLPO, the collection scores, lengths, and fixed normalizers are constant, so $\nabla_\theta d_{i,\theta}=\beta\nabla_\theta\log\pitrainer_\theta(y_i\mid x_i)$, and differentiation gives the weighted batch gradient. For token regression, the stop-gradients in Equation~\eqref{eq:token-sc-surrogate} give $-h_{i,u,\theta}\widetilde g_{i,u,\theta}$ for each token; applying Theorem~\ref{thm:token-sc} to each collection version and averaging with fixed version frequencies gives the expected mixture gradient. For sequence regression, apply Proposition~\ref{prop:bpo-no-baseline} to each version and differentiate its surrogate. Outcome-dependent replay selection can break these identities.
\end{proof}

\section{Proofs for the Response-Level and Value Results}
\label{app:value-proofs}
\subsection{Response Targets and Positive Weights}
\label{app:sequence-proof}
\begin{proof}[Proof of Theorems~\ref{prop:gibbs} and~\ref{thm:sequence-regression}]
For normalized policies the mass correction in $\UKL$ vanishes, and Lemma~\ref{lem:common-target} with $\mu=\pisamp$ and $f=R$ gives Equations~\eqref{eq:gibbs}--\eqref{eq:gap}. Equivalently, stationarity under $\sum_y\pi(y)=1$ gives $R(y)-\beta\log(\pi(y)/\pisamp(y))=\lambda$, and normalization fixes $\lambda=c^*$. Equation~\eqref{eq:mirror} uses $\eta^{-1}=\beta$. Replacing $R$ by $R-b(x)$ multiplies $Z$ by $e^{-b(x)/\beta}$ and leaves the normalized target unchanged.

Taking the logarithm of Equation~\eqref{eq:gibbs} gives $d_\theta=\beta\log(\pi_\theta/\pistar)$, and Theorem~\ref{thm:common-regression} identifies the unique zero-loss response distribution. Equation~\eqref{eq:gap} gives the reverse-KL criterion of the regularized return, and Equation~\eqref{eq:klpo} the squared log-ratio criterion of the regression. Sharing a zero on the full simplex does not make their projections onto a restricted model class equal.
\end{proof}

\begin{proof}[Proof of Proposition~\ref{prop:positive-weights}]
At $\pistar$ every residual is zero. Conversely, positivity of the sampler and of $a$ forces every residual to vanish at zero weighted loss, so the weight leaves both the target and its normalizer unchanged. With $a=1/T$ fixed, differentiation gives
\[
 \nabla_\theta\mathcal L_a=\EE_{\pisamp}[a\,d_\theta\score(Y)]
 =\EE_{\pisamp}\left[d_\theta\frac1T\sum_u\nabla_\theta\log\pi_\theta(y_u\mid x,y_{<u})\right],
\]
which is Equation~\eqref{eq:length-weighted}.
\end{proof}

\subsection{Regularized Bellman Values}
\label{app:bellman-proof}
\begin{proof}[Proof of Proposition~\ref{prop:soft-bellman}]
At a nonterminal state, condition on the first action and the successor state in Equation~\eqref{eq:qstar-objective}. Once the optimal continuation values are fixed by backward induction, the remaining problem is
\[
 \max_{\pi(\cdot\mid s)}\sum_a\pi(a\mid s)
 \left[\EE_P V^\star_{\beta,t,k}(S')-\beta\log\frac{\pi(a\mid s)}{\pisamp(a\mid s)}\right].
\]
Lemma~\ref{lem:common-target} gives the value and policy in Equation~\eqref{eq:qstar-bellman}, starting from the terminal boundary $V^\star_{\beta,t,k}=R$. The maximization changes the action law and keeps $P$.

The unregularized values satisfy $Q^\star(s,a)=\EE_P V^\star(S')$ and $V^\star(s)=\max_a Q^\star(s,a)$ with the same terminal boundary. For a fixed positive reference, let $p_{\min}=\min_{s,a}\pisamp(a\mid s)>0$ over feasible pairs and let $H$ bound the remaining horizon. The unregularized optimal value bounds the regularized value from above, and evaluating a deterministic unregularized optimal policy in the regularized objective gives the lower bound
\[
 V^\star(s)-\beta H\log(1/p_{\min})
 \leq V^\star_{\beta,t,k}(s)\leq V^\star(s).
\]
The state values therefore converge as $\beta\downarrow0$, and taking the fixed transition expectation gives convergence of the action values.
\end{proof}

\subsection{Implicit Soft Advantages}
\label{app:implicit-proof}
\begin{proof}[Proof of Theorem~\ref{prop:implicit-qstar}]
For deterministic transitions, $Q^\star_{\beta,t,k}(s,a)=V^\star_{\beta,t,k}(sa)$. Starting from the terminal value $F=R$, backward induction gives
\[
 e^{F(s)/\beta}=\sum_a\pisamp(a\mid s)e^{F(sa)/\beta}
 =\EE_{\pisamp}[e^{R/\beta}\mid s].
\]
Hence $V^\star_{\beta,t,k}=F$ and $Q^\star_{\beta,t,k}(s,a)=F(sa)$, and the product of the optimal token conditionals telescopes to the response Gibbs distribution. For SKLPO, zero loss and coverage imply $\pi_\theta(y\mid x)=\pistar(y\mid x)$ for every feasible response, so the token conditionals agree at every prefix with positive mass. For soft-value token regression, each zero residual gives those conditionals directly. Substituting into the Bellman policy proves Equation~\eqref{eq:implicit-recovery}, and fixed positive outer weights do not affect these zero-residual arguments.
\end{proof}

\subsection{Regression against Optimal Soft Q-Values}
\label{app:q-regression-proof}
\begin{proof}[Proof of Proposition~\ref{prop:q-regression}]
The Bellman policy gives $d^Q_\theta=\beta\log(\pi_\theta/\pi^\star_{\beta,t,k})$, and Theorem~\ref{thm:common-regression} identifies the zero-loss conditional policy and the gradient. Taking the sampler expectation of the identity
\[
 \beta\log\frac{\pi^\star_{\beta,t,k}(a\mid s)}{\pisamp(a\mid s)}=Q^\star_{\beta,t,k}(s,a)-V^\star_{\beta,t,k}(s)
\]
yields Equation~\eqref{eq:qstar-normalizer}. Under deterministic transitions, Equation~\eqref{eq:qstar-softvalue} gives $d^Q_\theta=\beta\ell_\theta-F(sa)+F(s)=-\beta\epsilon_u$.
\end{proof}

\subsection{Soft-Value Token Factorization}
\label{app:soft-token-proof}
\begin{proof}[Proof of Theorem~\ref{thm:soft-token}]
Under deterministic generation, the Gibbs probability of a prefix is proportional to its sampler probability times $\EE_{\pisamp}[e^{R/\beta}\mid s]$. The ratio of the masses of $sa$ and $s$ gives Equation~\eqref{eq:softfactorization}. Summing $F(x,y_{\le u})-F(x,y_{<u})$ telescopes to $R-F(x)=R-c^*(x)$, which proves Equation~\eqref{eq:softtokenresidual}. With $F$ fixed, $\nabla_\theta\epsilon_u=-\nabla_\theta\log\pi_\theta(y_u\mid x,y_{<u})$, which gives Equation~\eqref{eq:soft-token-loss}. Positive replay mass forces all local residuals to vanish at zero loss, and the factorization then identifies $\pistar$.
\end{proof}

\subsection{Terminal-Reward Token Regression}
\label{app:token-mc-proof}
\begin{proof}[Proof of Proposition~\ref{prop:token-mc}]
Write $\xi_u=R-Q^{\mathrm{sampler}}_{t,k}(x,y_{<u},y_u)$ and $m_u=V^{\mathrm{sampler}}_{t,k}(x,y_{<u})-b(x,y_{<u})$. Then $\widehat\delta_u=\delta_u+(\xi_u+m_u)/\beta$. Given $(x,y_{\le u})$, $\xi_u$ has zero mean, and given $(x,y_{<u})$, $\delta_u$ has zero mean. Expanding the square gives
\begin{equation}
 \mathcal L_{\mathrm{MC}}-\mathcal L_{\mathrm{tok}}
 =\frac{1}{2\beta}\EE_{x,Y\sim\pisamp}\sum_u
 \left[\Var_{\pisamp}(R\mid x,y_{\le u})+(V^{\mathrm{sampler}}_{t,k}(x,y_{<u})-b(x,y_{<u}))^2\right].
 \label{eq:tokenmcgap}
\end{equation}
Every term on the right depends only on the fixed sampler and baseline, so its derivative in $\theta$ is zero. The identity differentiates the full residual, including $k_\theta$. If $b\ne V^{\mathrm{sampler}}_{t,k}$, then $\EE[\widehat\delta_u\mid x,y_{<u}]=m_u/\beta$ is generally nonzero, and the KL derivative cannot be dropped.
\end{proof}

\section{RPG--URKL and Gradient Identities}
\label{app:rpg}
\label{sec:rpg}
We compare unweighted SKLPO with RPG on unfiltered responses from a fixed historical sampler $\pisamp$, which is off-policy relative to the trainer. Equation~\eqref{eq:gradklpo} gives the unweighted SKLPO gradient, and Section~\ref{sec:length-weighting} the length-weighted version. Write $w_\theta=e^{\ell_\theta}=\pi_\theta/\pisamp$. For normalized $\pisamp$, the fully differentiable URKL surrogate of \citet[Proposition~3.6]{rpg2026} is
\begin{equation}
 \mathcal L_{\mathrm{RPG}}(\theta)
 =\EE_{\pisamp}[-w_\theta R+\beta(w_\theta\log w_\theta-w_\theta)],
 \label{eq:rpg}
\end{equation}
which equals $-J(\pi_\theta)-\beta$. No clipping is applied in the identities below.

\begin{proof}[Proof of Proposition~\ref{prop:gradients}]
Differentiating Equation~\eqref{eq:klpo} gives Equation~\eqref{eq:gradklpo}. Differentiating Equation~\eqref{eq:rpg} gives $\EE_{\pisamp}[w_\theta(\beta\log w_\theta-R)\score]$, and adding $c^*$ leaves this expectation unchanged because $\EE_{\pisamp}[w_\theta\score]=\nabla_\theta\sum_y\pi_\theta(y)=0$.
\end{proof}

For a common response, residual, and score, $g_{\mathrm{RPG}}=e^{\ell_\theta}g_{\mathrm{SKLPO}}$. The SKLPO coefficient $d_\theta$ is affine in the accumulated log-ratio, and RPG multiplies it by the product of token ratios $e^{\ell_\theta}$. At $\pi_\theta=\pisamp$ the population gradients agree, although different baseline estimates can give different finite-sample gradients. With outer weight $1/T$, the sample identity becomes $g_{\mathrm{RPG}}=w_\theta T\,g_{\mathrm{SKLPO}}^{\mathrm{len}}$, so the matched-policy equality applies only to the unweighted regression. Multiplying the weighted loss by $\sg(w_\theta T)$ would recover the RPG gradient, but KLPO does not use that multiplier.

A stop-gradient importance weight also recovers the RPG gradient:
\begin{equation}
 \mathcal L_{\mathrm{sg}}(\theta)
 =\frac{1}{2\beta}\EE_{\pisamp}[\sg(w_\theta)d_\theta^2],
 \qquad
 \nabla_\theta\mathcal L_{\mathrm{sg}}=-\nabla_\theta J(\pi_\theta).
 \label{eq:sg}
\end{equation}
This is a surrogate for automatic differentiation, not an equality of scalar objectives. Differentiating through the weight as well would add
\begin{equation}
 \frac{1}{2\beta}\EE_{\pisamp}[w_\theta d_\theta^2\,\score(y)].
 \label{eq:extra}
\end{equation}
Replacing $c^*$ by any fixed prompt baseline in Equation~\eqref{eq:sg} still gives the population RPG gradient. The unweighted regression lacks this baseline invariance, because its zero-loss target depends on the correct normalizer.

\begin{table}[H]
\centering
\small
\begin{tabularx}{\textwidth}{@{}lXX@{}}
\toprule
 & RPG--URKL & Unweighted SKLPO with exact $c^*$ \\
\midrule
Criterion, up to constants & $\beta\KL(\pi_\theta\|\pistar)$ & $\frac{\beta}{2}\EE_{\pisamp}[\log^2(\pi_\theta/\pistar)]$ \\
Gradient weight on $\pisamp$ data & $w_\theta d_\theta$ & $d_\theta$ \\
Unrestricted optimum & $\pistar$ & $\pistar$ \\
Neural approximation & Optimizes regularized return & Fits the optimality condition \\
\bottomrule
\end{tabularx}
\caption{SKLPO removes the current-policy importance weight. The two objectives share the unrestricted optimum but have different gradients and, in general, different best approximations within a model class.}
\label{tab:comparison}
\end{table}

\section{Normalizer Identities and Estimation}
\label{app:normalizer}

\subsection{The Normalizer as a KL Divergence}
\label{app:normalizer-proof}
\begin{proof}[Proof of Lemma~\ref{lem:normalizer}]
Equation~\eqref{eq:softvalue} is Equation~\eqref{eq:normalizer-kl} with $(\mu,f)=(\pisamp,R)$. For the estimator, $e^{\widehat c/\beta}=G^{-1}\sum_i e^{R_i/\beta}$ is unbiased for $Z$, and Jensen's inequality gives $\EE\widehat c\leq\beta\log Z=c^*$.
\end{proof}

With centered rewards $A^{\mathrm{sampler}}_{t,k}=R-V^{\mathrm{sampler}}_{t,k}$, the normalizer is a pure KL term:
\begin{equation}
 c_A^*(x):=\beta\log\EE_{\pisamp}[e^{A^{\mathrm{sampler}}_{t,k}(x,Y)/\beta}]
 =c^*(x)-V^{\mathrm{sampler}}_{t,k}(x)=\beta\KL(\pisamp\|\pistar).
 \label{eq:centerednormalizer}
\end{equation}
Writing $\kappa^*(x)=\KL(\pisamp\|\pistar)$, the fixed-target residual becomes
\begin{equation}
 -\frac{d_\theta(x,y)}{\beta}
 =\frac{A^{\mathrm{sampler}}_{t,k}(x,y)}{\beta}-\ell_\theta(x,y)-\kappa^*(x).
 \label{eq:exactklresidual}
\end{equation}
This removes the response partition function in the same way that BPO removes the local partition function from its PMD optimality condition \citep[Section~3.1]{bpo2026}. The identity still involves the unknown optimum; substituting the current $\pi_\theta$ for $\pistar$ gives the profiled loss and changes the intercept. For large $\beta$, $c^*=V^{\mathrm{sampler}}_{t,k}+\Var_{\pisamp}(R)/(2\beta)+O(\beta^{-2})$, which quantifies the correction that a mean-reward baseline omits.

\subsection{Binary Rewards}
For rewards taking the values $r_-$ and $r_+$, let $\rho_{\pisamp}(x)$ be the success probability under $\pisamp$. Then
\begin{equation}
 c^*(x)=\beta\log\left[(1-\rho_{\pisamp}(x))e^{r_-/\beta}
 +\rho_{\pisamp}(x)e^{r_+/\beta}\right].
 \label{eq:binary}
\end{equation}
This covers both $0/1$ rewards and negative failure rewards, and a success-rate estimate gives a plug-in normalizer. The rate must belong to the prompt and the collection sampler; a global average, or an exponential moving average over changing policies, generally estimates something else.

\subsection{Profiling the Response Intercept}
\label{app:profiling-proof}
\begin{proof}[Proof of Proposition~\ref{prop:sequence-centering}]
This is Proposition~\ref{prop:profiling} with $(\mu,f)=(\pisamp,R)$, which gives Equations~\eqref{eq:profilec} and~\eqref{eq:centered}. Since $\kappa_\theta=-\EE_{\pisamp}\ell_\theta$, the residual in Equation~\eqref{eq:klresidual} is $(h_\theta-\EE_{\pisamp} h_\theta)/\beta$. Under sampling from the sampler, the fixed-normalizer and centered losses differ by
\begin{equation}
 \mathcal L_{\mathrm{SKLPO},\pisamp}-\mathcal L_{\mathrm{ctr},\pisamp}
 =\frac{1}{2\beta}\EE_x[(c^*(x)-c_\theta(x))^2].
 \label{eq:profilegap}
\end{equation}
Since $\EE_{\pisamp}\delta_{\mathrm{KL}}=0$ for each prompt, the derivative of $\kappa_\theta$ cancels after expectation:
\begin{equation}
 \nabla_\theta\mathcal L_{\mathrm{ctr}}
 =-\beta\EE_x\EE_{Y\sim\pisamp(\cdot\mid x)}
 [\delta_{\mathrm{KL}}(x,Y;\theta)\,\score(Y)].
 \label{eq:gradcentered}
\end{equation}
Neither the loss nor the gradient contains an importance weight.
\end{proof}

For a group from $\pisamp$, the estimates $\widehat V=\bar R$ and $\widehat\kappa_\theta=-\bar\ell$ give $\widehat\delta_i=(R_i-\bar R)/\beta-(\ell_i-\bar\ell)$, which recovers Equation~\eqref{eq:pairwise}. The finite-group bias statement there assumes fixed parameters and independent samples; it is not a generalization guarantee after fitting the batch.

\subsection{Empirical Centering and Singletons}
\label{app:empirical-proof}
\begin{proof}[Proof of Lemma~\ref{lem:empirical-centering}]
Expanding the pairwise square gives $\sum_{i,j}(h_i-h_j)^2=2G\sum_i(h_i-\bar h)^2$, which proves Equation~\eqref{eq:pairwise}. In the gradient, every term containing $\nabla_\theta\bar h$ is multiplied by $\sum_i(h_i-\bar h)=0$, so detaching the mean gives the same gradient. At fixed $\theta$, independence gives $\EE[G^{-1}\sum_i(h_i-\bar h)^2]=(G-1)\Var(h_\theta)/G$.
\end{proof}

\label{app:singleton-proof}
\begin{proof}[Proof of Proposition~\ref{prop:singleton}]
For $G=1$, the log-mean-exp estimator is $\widehat c=\beta\log e^{R_1/\beta}=R_1$. At matched policies $\ell_\theta=0$, so $d_\theta=0$. Empirical centering sets $\bar h=h_1$ at every policy, and a weighted singleton also has $\bar h_a=h_1$, so the centered residuals and losses are identically zero.
\end{proof}

\subsection{Data Selection and the Normalizer}
Fix a sampler $\pisamp$ and let $\nu$ be its response law after outcome-dependent selection. Unfiltered collection gives $\nu=\pisamp$ even when $k>0$ and the current trainer differs from the sampler; selection can break this equality. With the exact $c^*$, consider
\begin{equation}
 \mathcal L_{\mathrm{SKLPO},\nu}(\theta)
 =\frac{1}{2\beta}\EE_{Y\sim\nu}[d_\theta(Y)^2].
 \label{eq:offpolicy}
\end{equation}
If $\nu$ is positive on the support of $\pisamp$ and the target is realizable, this loss has the same unique zero-loss policy $\pistar$, and it needs no importance weight to estimate \emph{its own} gradient. Under model misspecification, though, its best approximation can differ from that of the sampler-conditioned loss $\mathcal L_{\mathrm{SKLPO},\pisamp}$, and finite data do not guarantee coverage.

Estimating $c^*$ from $\nu$ is a separate problem, for which the general importance-sampling identity is
\begin{equation}
 Z=\EE_{Y\sim\nu}\left[\frac{\pisamp(Y)}{\nu(Y)}e^{R(Y)/\beta}\right].
 \label{eq:offpolicyz}
\end{equation}
This weight concerns normalizer estimation only; the SKLPO loss and algorithm do not use it. Unfiltered data from the corresponding sampler need no correction, while an uncorrected exponential average under a different $\nu$ estimates a different normalizer. Alternatively, profiling the intercept under $\nu$ gives $c_{\theta,\nu}=\EE_\nu[R-\beta\ell_\theta]$ and a $\nu$-variance loss. Under the same support and realizability assumptions, that loss identifies $\pistar$ without a density ratio, with finite-sample behavior that depends on $\nu$.

\subsection{Weighted Profiling}
\label{app:weighted-profiling}
\begin{proof}[Proof of Proposition~\ref{prop:weighted-centering}]
Fix a prompt and a sampler version. Let $\nu$ be the possibly selected response law and $a(y)>0$ a fixed weight, and set $Z_a=\EE_\nu a$, $\nu_a(y)=a(y)\nu(y)/Z_a$, and $c_{\theta,a}=\EE_{\nu_a}h_\theta$. Completing the square gives
\begin{equation}
 \EE_\nu[a(c-h_\theta)^2]
 =\EE_\nu[a(h_\theta-c_{\theta,a})^2]
   +Z_a(c-c_{\theta,a})^2,
 \label{eq:weighted-profile-gap}
\end{equation}
so the profiled loss is $Z_a\Var_{\nu_a}(h_\theta)/(2\beta)$. The pairwise variance identity under $\nu_a$ gives Equation~\eqref{eq:weighted-pairwise} for an empirical group. The reweighted law changes the regression measure, but the sampler in $\ell_\theta=\log(\pi_\theta/\pisamp)$ and the fixed Gibbs normalizer remain those of the original target. With $g_\theta=\score(Y)$, differentiation gives
\begin{equation}
 \begin{aligned}
 \nabla_\theta\mathcal L_{\mathrm{ctr},a}
 &=-\EE_\nu[a(h_\theta-c_{\theta,a})g_\theta]\\
 &=-Z_a\EE_{\nu_a}
       [h_\theta(g_\theta-\EE_{\nu_a}g_\theta)].
 \end{aligned}
 \label{eq:weighted-centered-gradient}
\end{equation}
The intercept derivative vanishes because $\EE_\nu[a(h_\theta-c_{\theta,a})]=0$; an ordinary, unweighted intercept would not cancel in the weighted loss. Positive $a$ preserves the support, so zero weighted variance still identifies the original Gibbs target under full coverage and realizability.
\end{proof}

\subsection{Learned Normalizers for Single Rollouts}
\label{app:learned-normalizers}
Fix a collection version $v=t-k$ and $\beta$. A predictor takes the prompt and collection-policy information $(x,v)$ as input, and its prediction must describe that sampler rather than the current trainer. One response per prompt supplies one training example. Sharing parameters across prompts allows generalization, but it does not turn a single observed reward into a conditional expectation. We describe two fixed-normalizer estimators and a learned-intercept alternative; all three are implementation choices rather than assumptions of the response-level objective.

\paragraph{Binary rewards: predict the success probability.}
For $R\in\{r_-,r_+\}$, train $p_\phi(x,v)\in(0,1)$ with the Bernoulli negative log-likelihood
\begin{equation}
 \mathcal L_p(\phi)=\EE_{x,Y\sim\pisamp}
 [-z\log p_\phi(x,v)-(1-z)\log(1-p_\phi(x,v))],
 \qquad z=\mathbf1\{R=r_+\}.
 \label{eq:success-predictor}
\end{equation}
The unrestricted conditional minimizer is the sampler's success probability. Substituting the prediction into Equation~\eqref{eq:binary} gives
\begin{equation}
 \widehat c_\phi(x,v)=\beta\log\left[
 (1-p_\phi(x,v))e^{r_-/\beta}+p_\phi(x,v)e^{r_+/\beta}\right],
 \label{eq:predicted-binary-normalizer}
\end{equation}
which should be evaluated in log space. A single success or failure label can update $\phi$, but it does not justify setting $p_\phi$ to that label before scoring the same response.

\paragraph{General bounded rewards: predict an exponential moment.}
Let $M$ be a known upper bound on the reward and set $q(Y)=e^{(R(x,Y)-M)/\beta}\in(0,1]$. A positive predictor $m_\phi(x,v)$ can be trained by
\begin{equation}
 \mathcal L_m(\phi)=\EE_{x,Y\sim\pisamp}
 [(m_\phi(x,v)-q(Y))^2],\qquad
 \widehat c_\phi(x,v)=M+\beta\log m_\phi(x,v).
 \label{eq:moment-predictor}
\end{equation}
Its unrestricted conditional minimizer is $\EE_{\pisamp}[q(Y)\mid x]$, which gives $c^*$. The shift by $M$ bounds the regression labels without changing the normalizer, although small $\beta$ can still make rare high rewards decisive. Regressing $c_\phi$ directly onto $R$ would learn the mean reward instead. Both predictors use unfiltered rewards under the specified sampler and do not use the policy loss's $1/T$ weight. Outcome-dependent selection, length weighting, or pooling across versions changes the conditional quantity being learned, and a numerical floor on $m_\phi$ adds a further approximation.

\paragraph{Using a fixed prediction in policy updates.}
Train the predictor on separate records and freeze a snapshot before scoring the policy-training responses. Compute $\widehat c_\phi(x,v)$ from that snapshot, detach it from the policy computation, and keep the same prediction while reusing those responses; newly observed rewards can train a later snapshot. For replayed data, the independence premise of Appendix~\ref{app:single-rollout} requires separate estimation data or held-out folds, so that the predictor used for a response was not fitted on it. Detaching a prediction, or evaluating it before the current optimizer step, does not undo earlier fitting on the same record. These routes learn a prompt-level value or normalizer: they avoid a token-value critic, not all auxiliary value estimation.

\paragraph{Learning an intercept jointly with the policy.}
Instead of estimating the fixed $c^*$, the intercept can be parameterized inside a weighted regression. With separate parameters $\theta$ and $\phi$, let
\begin{equation}
 \mathcal J(\theta,\phi)=\frac1{2\beta}
 \EE_{x,Y\sim\pisamp}[a(Y)(\beta\ell_\theta-R+c_\phi(x,v))^2],
 \qquad a(Y)=1/T(Y).
 \label{eq:joint-intercept}
\end{equation}
At fixed $\theta$, unrestricted minimization over $c_\phi$ gives $c_{\theta,a}$ of Equation~\eqref{eq:weighted-centering}. A stochastic update first computes $h_i=R_i-\beta\ell_i$ and $c_\phi(x_i,v_i)$ before either parameter moves. The policy step treats $c_\phi$ as fixed, and the auxiliary step minimizes
\begin{equation}
 \widehat{\mathcal L}_c(\phi)=\frac1{2\beta B}
 \sum_i a_i\big(c_\phi(x_i,v_i)-\sg(h_i)\big)^2
 \label{eq:intercept-update}
\end{equation}
with the pre-update $h_i$. These are the two partial gradients of $\mathcal J$ at the pre-update parameters. Auxiliary gradients should stay out of the policy parameters; a small head on fixed prompt features is one option. The intercept must be refreshed as the trainer changes. One response per prompt gives a stochastic update of this joint objective, but an imperfect predictor does not give the exactly profiled policy gradient.

Under the same full-support and realizability assumptions, a zero-loss joint fit satisfies $R-\beta\ell_\theta=c_\phi(x,v)$ on the response support, and normalization then gives $c_\phi=c^*$ and $\pi_\theta=\pistar$ for that sampler. Finite singleton data need not identify this solution: an overly flexible predictor can memorize each prompt's observed residual, and a free intercept for every response cancels its residual altogether. The learned route therefore depends on generalization across prompts; it is not empirical centering of one response.

\section{Additional Off-Policy Results}
\label{app:offpolicy}
\subsection{The Local Importance Multiplier}
\label{app:token-is-proof}
\label{sec:token-is}
Write $g_\theta(s,a)=\nabla_\theta\log\pi_\theta(a\mid s)$. For a token from response $i$, the behavior conditional and the trainer-to-sampler ratio are
\begin{equation}
 \nu_i(a\mid s)=\pisampler_{\theta_{t-k_i(t)}}(a\mid s),\qquad
 \rho_{t,i}(s,a)=\frac{\pitrainer_{\theta_t}(a\mid s)}
 {\pisampler_{\theta_{t-k_i(t)}}(a\mid s)}.
 \label{eq:lagged-token-ratio}
\end{equation}

\begin{proposition}[Local Gradient Comparison]
\label{prop:token-is}
At a fixed history $s$, use the fixed-normalizer local PMD residual $d_{\theta,s}(a)$ of Equation~\eqref{eq:token-fixed-loss}, with the collection sampler $\nu_i$ as reference. For $\rho_\theta=\pi_\theta/\nu_i$, the per-sample contributions to the negative local regularized-return gradient and to the regression gradient are
\begin{equation}
 g_{\mathrm{IS}}(s,a)=\rho_\theta(s,a)d_{\theta,s}(a)g_\theta(s,a),
 \qquad
 g_{\mathrm{KLPO}}(s,a)=d_{\theta,s}(a)g_\theta(s,a),
 \label{eq:token-gradient-comparison}
\end{equation}
and their expectations under $a\sim\nu_i(\cdot\mid s)$ are the respective gradients.
\end{proposition}
Writing $\Delta=\log\rho_{t,i}$, importance sampling multiplies each contribution by $e^\Delta$, and KLPO has neither this multiplier nor a clipping threshold. Older checkpoints can be farther from the trainer, and longer trajectories offer more chances for an outlying token ratio. Outcome-dependent admission can change the action distribution that the return-gradient identity requires. FlashREINFORCE multiplies its token ratio by a batch advantage, while the length-weighted SKLPO update puts $d_i/T_i$ on each token score with no ratio. Both average scores within a response, but their residuals differ; Appendix~\ref{app:flash} gives the comparison.

\begin{proof}[Proof of Proposition~\ref{prop:token-is}]
At a fixed history, the conditional regression loss and gradient are
\begin{equation}
 \begin{aligned}
 d_{\theta,s}(a)
 &=\beta\log\frac{\pi_\theta(a\mid s)}{\pisamp(a\mid s)}
      -A^{\mathrm{sampler}}_{t,k}(s,a)+c^*_{\mathrm{tok}}(s),\\
 \mathcal L_{\mathrm{tok}}^{\mathrm{fix}}(\theta;s)
 &=\frac1{2\beta}\EE_{a\sim\pisamp(\cdot\mid s)}[d_{\theta,s}(a)^2],\\
 \nabla_\theta\mathcal L_{\mathrm{tok}}^{\mathrm{fix}}(\theta;s)
 &=\EE_{a\sim\pisamp(\cdot\mid s)}[d_{\theta,s}(a)g_\theta(s,a)].
 \end{aligned}
 \label{eq:token-offpolicy-regression}
\end{equation}
For the local return $J_s=\EE_{\pi_\theta}A^{\mathrm{sampler}}_{t,k}-\beta\KL(\pi_\theta\|\nu_i)$, Equation~\eqref{eq:framework-return-gradient} gives $-\nabla_\theta J_s=\EE_{\pi_\theta}[d_{\theta,s}g_\theta]$. Changing the action measure to $\nu_i$ introduces $\rho_\theta=\pi_\theta/\nu_i$ and proves Equation~\eqref{eq:token-gradient-comparison}.
\end{proof}

\subsection{Comparison with FlashREINFORCE}
\label{app:flash}
Let $\nu_i=\pisampler_{\theta_{t-k_i(t)}}$ be the behavior policy that generated response $i$, and $g_{i,u}=\nabla_\theta\log\pi_\theta(y_{i,u}\mid x_i,y_{i,<u})$. With the batch advantage $A_i$, sequence mask $m_i$, and behavior probabilities held fixed, the FlashREINFORCE policy-loss gradient is
\begin{equation}
 \widehat g_{\mathrm{Flash}}
 =-\frac1B\sum_i\frac{m_iA_i}{T_i}\sum_u
 \underbrace{\frac{\pi_\theta(y_{i,u}\mid x_i,y_{i,<u})}
 {\nu_i(y_{i,u}\mid x_i,y_{i,<u})}}_{\rho_{i,u}\text{: token IS}}
 g_{i,u}.
 \label{eq:flash-token-is}
\end{equation}
Both methods use the actual collection probability: FlashREINFORCE places the ratio in a multiplicative weight, and SKLPO puts its logarithm in the response residual. Equation~\eqref{eq:mixture-regression} gives $\widehat g_{\mathrm{SKLPO}}^{\mathrm{len}}=B^{-1}\sum_i(d_i/T_i)\sum_u g_{i,u}$, so both updates average token scores within a response, but the SKLPO coefficient is the complete-response optimality residual rather than a batch advantage. Equations~\eqref{eq:token-offpolicy-regression} and~\eqref{eq:token-gradient-comparison} isolate the multiplier at a single history with a local regularized residual; FlashREINFORCE's batch advantage defines a different residual.

\subsection{Centering under the Behavior Distribution}
\label{app:behavior-centering}
Let $\nu$ be the response distribution of a specified sampler version, possibly after outcome-dependent selection as in Appendix~\ref{app:normalizer}. Unweighted profiling under $\nu$ gives an off-policy score identity, and Equation~\eqref{eq:weighted-centered-gradient} gives its weighted counterpart. At a fixed prompt and sampler version, write
\begin{equation}
 \begin{aligned}
 h_\theta&=R-\beta\log(\pi_\theta/\pisamp),
 &c_{\theta,\nu}&=\EE_{\nu}h_\theta,\\
 \mathcal L_{\mathrm{ctr},\nu}
 &=\frac{1}{2\beta}\Var_{\nu}(h_\theta),
 &g_\theta&=\score(Y).
 \end{aligned}
 \label{eq:behavior-centered-loss}
\end{equation}
Its gradient can be written in either form
\begin{equation}
 \begin{aligned}
 \nabla_\theta\mathcal L_{\mathrm{ctr},\nu}
 &=-\EE_{\nu}[(h_\theta-c_{\theta,\nu})g_\theta]\\
 &=-\EE_{\nu}[h_\theta(g_\theta-\EE_{\nu}g_\theta)],
 \end{aligned}
 \label{eq:behavior-centered-score}
\end{equation}
both obtained by differentiating the variance and using the zero mean of the centered residual. The second form connects response-level regression to score centering \citep[Equation~(4)]{scorecentering2026} under the actual replay distribution, and a prompt-dependent reward offset cancels from it. The centering mean must match the sampling measure: $c_{\theta,\nu}$ is generally not $\EE_{\nu}R+\beta\KL(\pisamp\|\pi_\theta)$, since that expression requires sampling from $\pisamp$. Neither form is in general an unbiased on-policy return gradient.

\subsection{Single Rollouts and Signed Feedback}
\label{app:single-rollout}
With a fixed exact response normalizer $c_{t,k}^*(x)$, each completed response gives an SKLPO regression gradient. For two reward levels with $0<\rho_{\pisamp}(x)<1$, the initial residuals at $\pi_\theta=\pisamp$ are
\begin{equation}
 d_+=c_{t,k}^*-r_+<0,\qquad d_-=c_{t,k}^*-r_->0.
 \label{eq:signed-initial-feedback}
\end{equation}
Gradient descent therefore pushes up the likelihood of successes and down that of failures, without a token-value critic, and the log-ratio term moderates this feedback as the policy approaches its target. At the exact response Gibbs target, the success odds satisfy
\begin{equation}
 \frac{\rho_{t,k}^*}{1-\rho_{t,k}^*}
 =\frac{\rho_{\pisamp}}{1-\rho_{\pisamp}}
   e^{(r_+-r_-)/\beta}.
 \label{eq:success-odds}
\end{equation}
The signs hold at the matched-policy starting point of one component and can fail for a stale response at the current trainer; likewise, the odds identity describes the target, which a finite update with a shared trainer need not reach.

An independently estimated prompt success rate, additional data from the same prompt and version, or a learned prompt normalizer can approximate $c_{t,k}^*$, with some estimation error. A reward mean across prompts does not provide the required soft normalizer. Condition on a prompt-only estimate $\widehat c=c_{t,k}^*+\varepsilon(x)$, obtained independently of the response being scored and held fixed during policy updates. Writing $\nu$ for the response distribution at this sampler version, the conditional gradient error is
\begin{equation}
 \nabla_\theta\mathcal L_{\nu,\widehat c}
 -\nabla_\theta\mathcal L_{\nu,c_{t,k}^*}
 =\varepsilon(x)\EE_{\nu}[g_\theta\mid x],
 \label{eq:normalizer-gradient-error}
\end{equation}
which is nonzero off-policy in general. With outer weight $1/T$, the error becomes $\varepsilon(x)\EE_\nu[g_\theta/T\mid x]$, which can be nonzero even at $\nu=\pi_\theta$, since $\EE_{\pi_\theta}[g_\theta/T]=\nabla_\theta\EE_{\pi_\theta}[1/T]$. Empirical centering of a single response gives zero loss. SKLPO can therefore train from asynchronous single rollouts only when a normalization estimator is supplied. FlashREINFORCE's batch reward baseline serves a different objective.

\paragraph{History coverage and trajectory length.}
Removing token importance sampling leaves the distribution of stored histories unchanged. Their sampling frequencies can still affect a restricted model's best fit, and hard filters can remove the coverage that Proposition~\ref{prop:offpolicy-target} requires. A fixed positive outer weight such as $a(x,y)=1/|y|$ preserves each realizable component target for nonempty responses, while changing the gradients and possibly the best fit across versions. Averaging the log-ratio \emph{inside} the residual changes the target equation itself. Response-level feedback from a terminal failure also leaves the responsible intermediate tool call unidentified.

\subsection{Feasible Updates in Interactive Environments}
\label{app:interactive}
When tools are deterministic and initial conditions are fixed, a full-history action policy can realize the response Gibbs target on the reference trajectory support. Tool outputs enter the conditioning history without contributing likelihood factors. Stochastic transition factors that are common to both policies also cancel from a trajectory likelihood ratio, but this cancellation does not ensure feasibility, because the unrestricted Gibbs tilt can change observation probabilities that the policy does not control.

A statewise regression gives a separate route for stochastic environments. At a stored history $s$, take the local advantage $A^{\mathrm{sampler}}_{t,k}(s,a)$ and the PMD target of Equation~\eqref{eq:localtarget}. For any covering behavior distribution $\nu(\cdot\mid s)$, profiling
\begin{equation}
 \min_{\pi,c}\ \EE_{a\sim\nu(\cdot\mid s)}
 \left[\beta\log\frac{\pi(a\mid s)}{\pisamp(a\mid s)}
       -A^{\mathrm{sampler}}_{t,k}(s,a)+c(s)\right]^2
 \label{eq:statewise-offpolicy}
\end{equation}
has the same realizable local PMD target without action importance weights, by the zero-residual argument. The update fits a feasible action distribution and allows covered off-policy state sampling. It does require continuation values for the same sampler as its denominator. A complete return collected under that sampler estimates its ordinary $Q^{\mathrm{sampler}}_{t,k}$, a return from another version generally does not, and neither is an optimal regularized soft-Q target. Behavior centering does not correct a mismatch of continuations across versions, and it does not extend the deterministic telescoping identity to stochastic tools.

\subsection{A Stochastic-Tool Counterexample}
\label{app:stochastic-counterexample}
Consider a single forced action followed by a fair random terminal reward in $\{0,1\}$. Every feasible policy has value $1/2$ and zero policy KL. The unrestricted trajectory normalizer is instead $\beta\log[(1+e^{1/\beta})/2]>1/2$, by strict Jensen. The tilt raises the success probability even though the only action cannot change it, whereas the regularized Bellman equation keeps the environment law and gives value $1/2$.

\subsection{Learning Soft Bellman Targets}
\label{app:learning-soft-targets}
A critic extension could fit $Q_\psi$ to one-step targets $\widehat q=V_{\bar\psi}(S')$, with $V_{\bar\psi}(s)=\beta\log\sum_a\pisamp(a\mid s)e^{Q_{\bar\psi}(s,a)/\beta}$ at nonterminal states and $V_{\bar\psi}(s_T)=R(s_T)$. With $\bar\psi$ frozen, a transition sampled conditional on $(s,a)$ gives an unbiased Bellman target regardless of the behavior policy, provided replay selection preserves the transition law. This extension needs action values or a justified estimator of the action log-sum-exp, and a separate convergence analysis under function approximation. A sampler's terminal return estimates $Q^{\mathrm{sampler}}_{t,k}$ and cannot in general replace this optimal regularized continuation target.

\section{Token-Level Identities}
\label{app:token-details}
\subsection{Response and Token Gradients}
The centered response residual multiplies every token score:
\begin{equation}
 \nabla_\theta\mathcal L_{\mathrm{seq}}
 =-\beta\EE_{x,Y\sim\pisamp}\left[\delta_{\mathrm{seq}}
 \sum_u\nabla_\theta\log\pi_\theta(y_u\mid x,y_{<u})\right].
 \label{eq:seqtokencredit}
\end{equation}
Equation~\eqref{eq:tokengradient} gives the token PMD gradient, in which the local KL derivative cancels in population because $\EE_{\pisamp}[\delta_u\mid x,y_{<u}]=0$. Neither gradient has a multiplicative importance weight; the first uses one residual per response and the second one residual per token.

\subsection{Token Reductions and Masking}
\label{app:token-reductions}
Equations~\eqref{eq:tokenloss} and~\eqref{eq:orthogonality} sum over tokens and then average over responses. In population, a global token mean divides by the constant $\EE_{\pisamp}|Y|$ and changes nothing. Dividing each response's loss by its own length instead reweights trajectories and can invalidate the cancellation of cross terms, and replacing a sum of log-ratios by a length average changes the residual itself. Termination should be counted consistently, and prompt, padding, and external tool-output tokens excluded from policy likelihoods.

\subsection{BPO Equation (25) and Score Centering}
\label{sec:scorecentering}
\citet[Equation~(25)]{bpo2026} derive the per-token gradient before approximating the full KL. Write $g_\theta(s,a)=\nabla_\theta\log\pi_\theta(a\mid s)$ and $g_{i,u,\theta}=\nabla_\theta\log\pi_\theta(y_{i,u}\mid x,y_{i,<u})$, and let $\phi(x)$ be a fixed prompt weight. With $\eta=1/\beta$, their token contribution to the trajectory-loss gradient is
\begin{equation}
 \begin{aligned}
 \nabla_\theta\mathcal L^{\mathrm{BPO}}_{i,u}
 &=-\phi(x)\beta\delta_{\mathrm{BPO}}(x,y_i)
 \big[g_{i,u,\theta}+\nabla_\theta k_\theta(x,y_{i,<u})\big]\\
 &\approx-\phi(x)\big[R(x,y_i)-V^{\mathrm{sampler}}_{t,k}(x)\big]
 \big[g_{i,u,\theta}+\nabla_\theta k_\theta(x,y_{i,<u})\big].
 \end{aligned}
 \label{eq:bpo25}
\end{equation}
The first line is exact; the second freezes the residual at $\pi_\theta=\pisamp$. Since $\pisamp$ is fixed, the bracket is a centered score:
\begin{equation}
 \begin{aligned}
 \widetilde g_\theta^{\pisamp}(s,a)
 &:=g_\theta(s,a)+\nabla_\theta k_\theta(s)\\
 &=g_\theta(s,a)-\EE_{a'\sim\pisamp(\cdot\mid s)}g_\theta(s,a'),
 \qquad \EE_{a\sim\pisamp(\cdot\mid s)}\widetilde g_\theta^{\pisamp}(s,a)=0.
 \end{aligned}
 \label{eq:scorecentering}
\end{equation}
This is the operation of \citet[Equation~(4)]{scorecentering2026}, with differentiated policy $\pi_\theta$ and sampling distribution $\nu=\pisamp$. BPO's linearized gradient weights the centered score by the reward, while its exact gradient keeps the current trajectory residual. The full-KL correction is additive and involves no importance ratio; the binary-KL approximation, smoothing, and clipping then change the update.

For KLPO, Equation~\eqref{eq:tokengradient} can equivalently use $\widetilde g_\theta^{\pisamp}$, because $\EE_{\pisamp}[\delta_u\mid x,y_{<u}]=0$, and the corresponding response-level correction centers the complete-response score by its prompt-conditional mean. These cancellations hold in population; on a single trajectory the local KL derivatives do not vanish. If the actual behavior is $\nu\ne\pisamp$, score centering uses the mean under $\nu$, which a KL correction based on $\pisamp$ does not supply. Even with the matching mean, centering does not change the sampling distribution into $\pi_\theta$ or recover its importance-weighted policy gradient.

\subsection{The Path Normalizer}
\label{app:path-normalizer}
Under Assumption~\ref{ass:deterministic}, the token target is
\begin{equation}
 \pi^+_{\mathrm{tok}}(a\mid s)
 =\frac{\pisamp(a\mid s)e^{A^{\mathrm{sampler}}_{t,k}(s,a)/\beta}}{Z_s},
 \qquad Z_s=\EE_{a\sim \pisamp(\cdot\mid s)}e^{A^{\mathrm{sampler}}_{t,k}(s,a)/\beta}.
\end{equation}
Using $\sum_u A^{\mathrm{sampler}}_{t,k}(x,y_{<u},y_u)=R(x,y)-V^{\mathrm{sampler}}_{t,k}(x)$, the induced response distribution is
\begin{equation}
 \pi^+_{\mathrm{tok}}(y\mid x)
 =\pisamp(y\mid x)\exp\left(
 \frac{R(x,y)-V^{\mathrm{sampler}}_{t,k}(x)}{\beta}-\sum_u\log Z_{(x,y_{<u})}\right).
\end{equation}
The response Gibbs update uses a single $\log Z(x)$ instead. The two agree only when the accumulated local normalization is constant across the supported responses to each prompt, which the Bellman telescope alone does not imply.

\section{Detailed Comparisons with Prior Objectives}
\label{app:related}

\paragraph{Self-Play Preference Optimization.}
\label{app:sppo}
Write $p=\pisamp$ for one frozen sampler and define its response win rate
\[
 w_p(x,y)=\EE_{Y'\sim p(\cdot\mid x)}
             [\mathbb P(y\succ Y'\mid x)].
\]
With $\eta=1/\beta$, the exact update and regression of SPPO \citep[Equations~(4.2)--(4.4)]{sppo2024} become
\begin{equation}
 \begin{aligned}
 c_w(x)&=\beta\log\EE_{Y\sim p(\cdot\mid x)}e^{w_p(x,Y)/\beta},\\
 \pi_w^+(y\mid x)&=p(y\mid x)e^{(w_p(x,y)-c_w(x))/\beta},\\
 \mathcal L_{\mathrm{SPPO}}^{\mathrm{exact}}
 &=\EE_{x,Y\sim p}
   \left[\log\frac{\pi_\theta(Y\mid x)}{p(Y\mid x)}
         -\frac{w_p(x,Y)-c_w(x)}{\beta}\right]^2.
 \end{aligned}
 \label{eq:sppo-comparison}
\end{equation}
Setting $R=w_p$ therefore gives the fixed-normalizer SKLPO loss up to a positive scale, with matched sampling and outer weights. This is a one-step identity: SPPO updates the opponent in self-play, whereas a verifier reward does not depend on the sampler. The practical SPPO loss \citep[Equation~(4.6)]{sppo2024} replaces $\log Z$ by $\eta/2$, equivalently $c_w$ by $1/2$, and estimates $w_p$ from sampled opponents. For reciprocal preferences $\EE_p w_p=1/2$, so Equation~\eqref{eq:normalizer-kl} gives
\[
 c_w(x)=\frac12+\beta\KL(p(\cdot\mid x)\|\pi_w^+(\cdot\mid x)).
\]
The half-win baseline is therefore not the exact normalizer in general. GPO instead uses preference log-odds scores, and even under Bradley--Terry preferences their expected sigmoid win rate is generally not affine in the scalar reward. SPPO therefore need not share the verifier-reward Gibbs target, and its self-play convergence guarantee does not by itself carry over to asynchronous SKLPO replay.

\paragraph{General Preference Optimization.}
GPO fits log-likelihood ratios specified by an exponential policy update \citep[Equations~(5.1)--(5.3)]{gpo2025}. For a frozen opponent and reference $\pisamp$, let
\begin{equation}
 u_{\mathrm{pref}}(x,y)=\EE_{Y'\sim \pisamp(\cdot\mid x)}[s(y\succ Y'\mid x)],
 \qquad c_{\mathrm{pref}}^*(x)=\beta\log\EE_{\pisamp}[e^{u_{\mathrm{pref}}(x,Y)/\beta}],
 \label{eq:preference}
\end{equation}
where $s$ is an antisymmetric preference score. GPO estimates this expected score with sampled opponents and fits the target $\pisamp\exp[(u_{\mathrm{pref}}-c_{\mathrm{pref}}^*)/\beta]$, which is response-level KLPO with signal $u_{\mathrm{pref}}$. Its practical implementation approximates the log-normalization term by zero; the exact comparison keeps it. For reward-difference scores $s(y\succ y'\mid x)=R(x,y)-R(x,y')$, we have $u_{\mathrm{pref}}=R-\EE_{\pisamp}R$ and $c_{\mathrm{pref}}^*=c^*-\EE_{\pisamp}R$, so $u_{\mathrm{pref}}-c_{\mathrm{pref}}^*=R-c^*$. The Gibbs targets then agree, and the exact-normalizer population losses agree up to a positive scale when their data distributions and outer weights match. This does not identify the practical GPO loss, with sampled opponents and a zero normalizer, with SKLPO. SKLPO uses the verifier directly without a preference opponent, while general pairwise scores can encode cyclic preferences that no scalar reward difference represents.

\Needspace{5\baselineskip}
\paragraph{Bellman Policy Optimization.}
BPO starts from a token-level PMD update with $A^{\mathrm{sampler}}_{t,k}(s,a)$ \citep{bpo2026}. With $\eta=1/\beta$, its exact trajectory residual is
\begin{equation}
 \begin{aligned}
 \delta_{\mathrm{BPO}}(x,y)
 &=\frac{R(x,y)-V^{\mathrm{sampler}}_{t,k}(x)}{\beta}\\
 &\quad-\sum_u\left[\log\frac{\pi_\theta(y_u\mid x,y_{<u})}{\pisamp(y_u\mid x,y_{<u})}
 +\KL(\pisamp(\cdot\mid x,y_{<u})\|\pi_\theta(\cdot\mid x,y_{<u}))\right].
 \end{aligned}
 \label{eq:bpo}
\end{equation}
Let $S_\theta(x,y)=\sum_u\KL(\pisamp(\cdot\mid x,y_{<u})\|\pi_\theta(\cdot\mid x,y_{<u}))$. The KL chain rule and Equation~\eqref{eq:sequence} give
\begin{equation}
 \kappa_\theta(x)=\EE_{Y\sim\pisamp(\cdot\mid x)}[S_\theta(x,Y)],\qquad
 \delta_{\mathrm{KL}}-\delta_{\mathrm{BPO}}=S_\theta-\kappa_\theta.
 \label{eq:bpodifference}
\end{equation}
The profiled response residual thus subtracts the expected sum of token KLs, while BPO subtracts the path sum, and the two agree only when the path sum is constant on the reference support. Separately, Theorem~\ref{thm:token-bpo} shows that the exact BPO trajectory loss equals the token PMD regression loss in population, and Appendix~\ref{sec:scorecentering} connects BPO's Equation~(25) to score centering. BPO's practical algorithm approximates the residual further and uses a smoothed complementary-probability weight, which falls outside the exact equality.

\paragraph{Reward-weighted maximum likelihood.}
Another way to fit Equation~\eqref{eq:gibbs} is to minimize $\KL(\pistar\|\pi_\theta)$, which gives
\begin{equation}
 \mathcal L_{\mathrm{MLE}}(\theta)
 =-\EE_{\pisamp}\left[\frac{e^{R/\beta}}{Z}\log\pi_\theta\right].
 \label{eq:awr}
\end{equation}
This is the reward-weighted regression behind AWR \citep{awr2019}, and the weighted maximum-likelihood step of REPS and MPO \citep{peters2010reps,abdolmaleki2018mpo} has the same form. The exponential weight is fixed within a round, whereas the RPG weight depends on the current policy. Log-ratio regression, weighted maximum likelihood, and direct RPG optimization share a realizable target and take different steps toward it.

\section{Additional Identities}
\label{app:identities}
\subsection{An Unnormalized Reference Measure}
For this algebraic extension only, let $m$ be a positive finite measure with mass $M=\sum_ym(y)$, while the optimized $\pi_\theta$ remains normalized. Set $Z=\sum_ym(y)e^{R(y)/\beta}$. Then
\begin{equation}
 \pistar(y)=\frac{m(y)e^{R(y)/\beta}}{Z},\qquad
 J(\pi_\theta)=\beta\log Z+\beta(1-M)-\beta\KL(\pi_\theta\|\pistar).
\end{equation}
Scaling $m$ only adds a constant to the objective and leaves its normalized optimizer unchanged. Without the normalization constraint, the stationary measure is $\widetilde\pi(y)=m(y)e^{R(y)/\beta}$, whose mass need not be one, so it cannot directly serve as the trainable policy.

\subsection{Coverage and Uniqueness}
Fix a prompt and let $\nu$ share the support of $\pisamp$. The nonnegative exact-normalizer loss vanishes at $\pistar$, and any zero-loss policy satisfies $\beta\log(\pi_\theta/\pisamp)-R+c^*=0$ on that support, which gives uniqueness. Profiling gives $R-\beta\log(\pi_\theta/\pisamp)=a$ instead, and normalization fixes $a=c^*$.

Full coverage is sufficient but not always necessary. With a fixed exact normalizer, the observed probabilities are fixed, and normalization determines a single omitted probability but only the total mass of two or more omitted responses. With a profiled intercept only the observed probability ratios are fixed, so even one omitted response can make the solution nonunique. Finite data can omit responses even when the population support is full.

\subsection{Profiling and Relative-Reward Regression}
For $h_\theta=R-\beta\log(\pi_\theta/\pisamp)$,
\begin{equation}
 \EE_\nu[(c-h_\theta)^2]
 =\Var_\nu(h_\theta)+(c-\EE_\nu h_\theta)^2,
\end{equation}
which gives the optimal intercept, and for independent $Y,Y'\sim\nu$,
\begin{equation}
 \Var_\nu(h_\theta)=\frac12\EE[(h_\theta(Y)-h_\theta(Y'))^2].
\end{equation}
The pairwise form uses differences of scalar rewards; it does not represent general pairwise preferences.

\Needspace{20\baselineskip}
\section{Sequence Length Weighting and Gradient Scale}
\label{app:length-gradient}
Write $g_u=\nabla_\theta\log\pi_\theta(y_u\mid x,y_{<u})$. Under length weighting, each token receives the coefficient $d_\theta/T$. At matched trainer and sampler, $d_\theta=c^*-R$ has no direct length factor. If $|\ell_{\theta,u}|\leq M$ and $\|g_u\|\leq G$, the per-response weighted gradient has norm at most $(\beta TM+|R-c^*|)G$, and the unweighted bound has one more factor of $T$. Averaging the scores does not bound the accumulated residual: a nonvanishing mean token log-ratio can still make the weighted loss and this bound grow linearly in $T$.

The return-gradient contribution is $w_\theta T$ times the weighted regression contribution, where $w_\theta=\pi_\theta/\pisamp$. When lengths vary, their population gradients generally differ even at matched policies, since $\EE_{\pi_\theta}[T^{-1}\score(Y)]=\nabla_\theta\EE_{\pi_\theta}[T^{-1}]$ need not vanish. Preserving the target therefore does not imply the first-step gradient equality of the unweighted formulation.

\section{Implementation Details}
\label{app:implementation}
This appendix collects what is needed to run the updates of Sections~\ref{sec:granularity}--\ref{sec:agentic} and Appendix~\ref{sec:sequence}. Algorithm~\ref{alg:token-klpo} gives the token-level update and Algorithm~\ref{alg:klpo} the response-level one. Code for both gradient routes and all four KL corrections is available from the project page.

\begin{algorithm}[!t]
\caption{Token-Level KLPO from a Replay Queue}
\label{alg:token-klpo}
\small
\begin{algorithmic}[1]
\Require Completed responses with sampled-token sampler log-probabilities and collection versions
\Require Route: token or sequence regression; KL correction: MC-KL, TopK-KL, binary, or full
\Require MC-KL: $M\geq1$ for token, $M\geq2$ for sequence; TopK-KL: head size $K$
\For{optimizer steps $t=0,1,\ldots$}
 \State Read $B$ responses $(x_i,y_i,T_i,R_i,v_i)$ and the sampler records the KL correction needs.
 \State Score policy tokens with $\pitrainer_{\theta_t}$; compute $\ell_{i,u}$ against version $v_i$.
 \State For MC-KL, draw $M$ fresh auxiliary tokens per prefix from version $v_i$; score their log-ratios.
 \For{responses $i=1,\ldots,B$}
  \If{token regression}
   \State Set $h_{i,u}=R_i-\beta\ell_{i,u}$; form $L_i$ by Eq.~\eqref{eq:mc-token-backward}, \eqref{eq:token-sc-topk}, \eqref{eq:sc-binary-surrogate}, or \eqref{eq:token-sc-surrogate}.
  \Else
   \State Form $L_i$ by Eq.~\eqref{eq:mc-sequence-backward}, \eqref{eq:bpo-topk-surrogate}, \eqref{eq:bpo-binary-surrogate}, or \eqref{eq:bpo-backward-surrogate}.
  \EndIf
 \EndFor
 \State Compute $G_t=\nabla_\theta^{\mathrm{AD}}[B^{-1}\sum_iL_i]$ with the specified stop-gradients.
 \State Apply the optimizer update with $G_t$ to obtain $\theta_{t+1}$.
 \State Publish checkpoints; keep the original rewards, sampler records, and collection versions.
\EndFor
\end{algorithmic}
\end{algorithm}

\subsection{Records, Reductions, and Reuse}
Each response stores its prompt, policy tokens, reward, collection version, and the sampler log-probability of every generated token. The KL correction adds the following: $M$ auxiliary token IDs per position with their sampler log-probabilities for MC-KL, stored as arrays of shape $[B,T,M]$; the head IDs and their sampler probabilities for TopK-KL; the sampler's next-token distributions for full KL; and nothing for binary KL. The rollout action's log-probability is stored separately in every case, including when the action lies outside a top-$K$ head. Correctly logged sampler probabilities make a separate reference-model scoring pass unnecessary. They should include the sampling transforms used at collection, such as temperature and truncation, because inference and training can produce different probabilities at identical parameters.

At every update, including reuse of the same batch, everything that depends on the trainer is recomputed: the log-ratios, the coefficient $h_{i,u}$ or residual $D_{i,\theta}$, and the KL correction. Sampler records, rewards, and collection versions stay fixed. The loss sums over policy tokens and averages over responses; prompt, padding, and tool-output tokens are excluded and termination is counted consistently. An outer weight $1/T_i$ breaks the exact identities.

For MC-KL, repeated token IDs are valid. Token regression subtracts the mean auxiliary log-probability. Sequence regression uses the leave-one-out coefficients of Equation~\eqref{eq:mc-sequence-backward} and reports the U-statistic of Equation~\eqref{eq:mc-sequence-ustat}, which can be negative and is not clipped. A conditionally unbiased gradient at each update requires fresh draws from the historical sampler (Section~\ref{sec:mc-kl}); the labs-molt integration consumes each set of draws once before updating the trainer.

For sequence regression with TopK-KL, $k^K$ and $D^K_\theta$ are recomputed at every reuse, and Equation~\eqref{eq:bpo-topk-surrogate} is differentiated with only the residual coefficient detached, keeping gradients through both the head and the tail terms. With binary KL, $k^{\mathrm{bin}}$, $\widehat D_\theta$, and $\omega$ are recomputed, and both $\widehat D_\theta$ and $\omega$ are detached; the token route detaches $h_{\theta,u}$ and $\omega$. No group normalization or residual linearization is applied. For exact sequence regression, the full local KL is
\begin{equation}
 k_{i,\theta}(s)=\sum_{a\in\mathcal V}\pisampler_{\theta_{v_i}}(a\mid s)
 \left[\log\pisampler_{\theta_{v_i}}(a\mid s)-\log\pitrainer_\theta(a\mid s)\right],
 \label{eq:full-vocabulary-kl}
\end{equation}
and its derivative is kept in Equation~\eqref{eq:bpo-backward-surrogate}. Token regression needs only the score mean, not the KL value. Using $-\ell_{i,u}$ from the rollout action as the entire KL estimate cancels the log-ratio and removes the learning signal; it is not an MC-KL draw, a TopK-KL head, or a binary partition.

\subsection{Numerical Safeguards}
\label{app:topk-stability}
For TopK-KL, log-probabilities and tail masses are computed in float32 or higher precision. With $\epsilon=10^{-6}$, write $\bar q_{\mathrm{tail}}=\max(q_{\mathrm{tail}},\epsilon)$ and $\bar p_{\mathrm{tail}}=\max(p_{\mathrm{tail}},\epsilon)$. Sequence regression uses the stabilized scalar
\begin{equation}
 k^{K,\epsilon}_\theta
 =\sum_{v\in H}q_v\log\frac{q_v}{p_v}
 +\bar q_{\mathrm{tail}}\log\frac{\bar q_{\mathrm{tail}}}{\bar p_{\mathrm{tail}}}
 \label{eq:topk-stabilized}
\end{equation}
in both the residual and the differentiated bracket. Below the floor, the trainer-tail derivative is zero away from the clamp boundary, and differentiating the literal stabilized scalar keeps its own squared-residual identity, although the stabilized expression need not be a KL between normalized distributions. Token regression follows the convention of \citet{scorecentering2026} and uses $\bar q_{\mathrm{tail}}/\bar p_{\mathrm{tail}}$ in the detached head coefficients. When either floor is active, this correction need not equal $-\nabla k^{K,\epsilon}_\theta$; the correspondence of Equation~\eqref{eq:topk-kl-derivative} assumes inactive floors. A full-vocabulary record uses the exact sum, with no tail bucket or ratio.

For binary KL, complementary log-probabilities are computed with stable \texttt{log1mexp} arithmetic in float32 or higher precision, rather than by subtracting a rounded probability from one. Any additional smoothing or clipping defines a further approximation.

\subsection{The SKLPO Update}
The SKLPO trainer consumes completed responses without a collection barrier. In the supplied-normalizer branch of Algorithm~\ref{alg:klpo}, $\widehat c_{v_i}$ estimates $c^*_{t,k_i(t)}$ for the response's collection policy and stays fixed across trainer updates; group means and learned intercepts instead track the current trainer. Stored denominators and lengths stay fixed in every branch. Both the log-ratio sum and $T_i$ count only generated policy tokens with a consistent termination convention.

For unfiltered responses to one prompt from one sampler, Equation~\eqref{eq:mc} estimates the normalizer without length weights. Selection or pooling across versions needs a separate justification, because a pooled estimate generally fails to recover each component's normalizer. Group centering recomputes weighted means within prompt and version groups after every trainer update; each group supplies an intercept, and the loss averages over all $B$ responses rather than over groups. With one response per prompt, use an independently supplied normalizer or the shared learned intercept of Appendix~\ref{app:learned-normalizers}, never a singleton group mean. Removing importance weights does not restore support lost through truncation or admission filters.

%% file: contents/sequence.tex
\section{Response-Level KLPO (SKLPO)}
\label{sec:sequence}
This appendix develops the response-level instance of KLPO, which we call SKLPO (Sequence KL-Regularized Policy Optimization). SKLPO regresses the optimality condition of the response Gibbs target with one normalizer per prompt and sampler version, and the response residual multiplies every token score. Given normalization information it can train from single responses, but unlike token-level KLPO it cannot do without that information (Section~\ref{sec:klpo-sklpo}). Appendix~\ref{sec:normalizer} covers normalization and Appendix~\ref{sec:qstar} the soft-value interpretation. We fix a collection version and abbreviate $c^*=c^*_{t,k}$ and $Z=Z_{t,k}$.

\subsection{The Response Gibbs Target}
For $\beta>0$, the response-level objective at step $t$ is
\begin{equation}
 \begin{aligned}
 J(\pi_\theta)&=\EE_{y\sim \pi_\theta}[R(y)]-\beta\UKL(\pi_\theta\|\pisamp),\\
 \UKL(\pi_\theta\|\pisamp)&=\sum_y\left[\pi_\theta(y)\log\frac{\pi_\theta(y)}{\pisamp(y)}+\pisamp(y)-\pi_\theta(y)\right].
 \end{aligned}
 \label{eq:objective}
\end{equation}
For normalized policies the mass correction vanishes and the unnormalized KL equals the ordinary KL; the URKL formulation keeps the normalization constraint on the policy.

\begin{theorem}[Response Gibbs Improvement]
\label{prop:gibbs}
Under Assumption~\ref{ass:finite}, define
\[
 Z(x)=\EE_{y\sim\pisamp(\cdot\mid x)}e^{R(x,y)/\beta},\qquad c^*(x)=\beta\log Z(x).
\]
The unique maximizer of Equation~\eqref{eq:objective} is
\begin{equation}
 \pistar(y\mid x)=\frac{\pisamp(y\mid x)e^{R(x,y)/\beta}}{Z(x)},
 \label{eq:gibbs}
\end{equation}
and for every feasible $\pi_\theta$,
\begin{equation}
 J(\pi_\theta)=c^*-\beta\KL(\pi_\theta\|\pistar).
 \label{eq:gap}
\end{equation}
Equivalently, $\pistar$ is the entropic mirror-ascent update
\begin{equation}
 \pistar=\argmax_{\pi\in\Delta(\mathcal Y_x)}
 \left\{\langle\pi,R\rangle-\eta^{-1}\KL(\pi\|\pisamp)\right\},
 \qquad \eta=\beta^{-1}.
 \label{eq:mirror}
\end{equation}
Subtracting a prompt-dependent baseline $b(x)$ from $R$ changes $c^*$ to $c^*-b$ and leaves the target unchanged.
\end{theorem}
The proof is in Appendix~\ref{app:sequence-proof}. The target is defined relative to the collection sampler, and a finite trainer step need not reach it.

\subsection{Regressing the Response Condition}
Taking logarithms of the Gibbs policy gives the condition that SKLPO fits. The response log-ratio is a sum of token log-ratios,
\begin{equation}
 \begin{aligned}
 \ell_{\theta,u}(x,y)
 &=\log\pi_\theta(y_u\mid x,y_{<u})-\log\pisamp(y_u\mid x,y_{<u}),\\
 \ell_\theta(x,y)
 &=\log\frac{\pi_\theta(y\mid x)}{\pisamp(y\mid x)}
 =\sum_{u=1}^{T}\ell_{\theta,u}(x,y),
 \end{aligned}
 \label{eq:sequence}
\end{equation}
and the response residual and unweighted loss are
\begin{equation}
 d_\theta(x,y)=\beta\ell_\theta(x,y)-R(x,y)+c^*(x),
 \label{eq:residual}
\end{equation}
\begin{equation}
 \begin{aligned}
 \mathcal L_{\mathrm{seq}}^{\mathrm{fix}}(\theta)
 &=\frac{1}{2\beta}\EE_{x\sim\mathcal D,\,Y\sim\pisamp(\cdot\mid x)}[d_\theta(x,Y)^2]\\
 &=\frac{1}{2\beta}\EE_{x,Y}\left[
 \left(\beta\sum_{u=1}^{|Y|}\ell_{\theta,u}(x,Y)-R(x,Y)+c^*(x)\right)^2\right].
 \end{aligned}
 \label{eq:klpo}
\end{equation}
The token log-ratios are summed \emph{before} squaring, and the reward and normalizer each enter once, so one residual multiplies every token score. We write $\kappa_\theta(x)=\KL(\pisamp(\cdot\mid x)\|\pi_\theta(\cdot\mid x))$ for the complete-response KL. Sequence regression in Section~\ref{sec:tokenbpo} also shares one residual across a response, but its correction is the path sum $\beta\sum_u k_\theta(s_u)$ rather than this prompt normalizer, and it targets local PMD rather than the response Gibbs policy (Section~\ref{sec:sr-vs-sklpo}).

\begin{theorem}[SKLPO Regression Target]
\label{thm:sequence-regression}
Under Assumption~\ref{ass:finite}, $d_\theta=\beta\log(\pi_\theta/\pistar)$. If $\pistar$ is realizable, Equation~\eqref{eq:klpo} vanishes exactly at $\pi_\theta=\pistar$. In a restricted model class, the regression minimizes a squared log-ratio under $\pisamp$ while the regularized return minimizes $\KL(\pi_\theta\|\pistar)$, and the two minimizers can differ.
\end{theorem}
The proof is in Appendix~\ref{app:sequence-proof}.

\begin{proposition}[Regression and Return Gradients]
\label{prop:gradients}
With responses drawn from $\pisamp$, and the sampler and normalizer held fixed during differentiation,
\begin{align}
 \nabla_\theta\mathcal L_{\mathrm{seq}}^{\mathrm{fix}}
 &=\EE_{\pisamp}[d_\theta\,\score(y)],
 \label{eq:gradklpo}\\
 -\nabla_\theta J(\pi_\theta)
 &=\EE_{\pisamp}\left[\frac{\pi_\theta(y)}{\pisamp(y)}d_\theta\,\score(y)\right].
 \label{eq:gradrpg}
\end{align}
The two population gradients agree at $\pi_\theta=\pisamp$ and generally differ elsewhere.
\end{proposition}
The proof is in Appendix~\ref{app:rpg}. This is Equation~\eqref{eq:framework-return-gradient} at the response level: the return gradient carries the multiplier $w_\theta=\pi_\theta/\pisamp$. Appendix~\ref{sec:token-is} makes the same comparison at a single history.

\subsection{Outer Length Weighting}
\label{sec:length-weighting}
Let $T(x,y)\geq1$ count generated policy tokens, including termination under a fixed stopping convention. A length weight changes how much each response contributes, while the summed log-ratio stays inside the residual.
\begin{proposition}[Positive Outer Weights]
\label{prop:positive-weights}
Under the assumptions of Theorem~\ref{thm:sequence-regression}, any fixed positive response weight $a(x,y)$ gives $\mathcal L_a=(2\beta)^{-1}\EE_{\pisamp}[a\,d_\theta^2]$ the same unique zero-loss policy $\pistar$ and normalizer $c^*$. For $a=1/T$,
\begin{equation}
 \begin{aligned}
 \mathcal L_{\mathrm{seq}}^{\mathrm{len}}(\theta)
 &=\frac1{2\beta}\EE_{x,Y\sim\pisamp}
       \left[\frac{d_\theta(x,Y)^2}{T(x,Y)}\right],\\
 \nabla_\theta\mathcal L_{\mathrm{seq}}^{\mathrm{len}}
 &=\EE_{x,Y\sim\pisamp}\left[
 d_\theta\,\frac1T\sum_{u=1}^T\nabla_\theta\log\pi_\theta(y_u\mid x,y_{<u})
 \right],
 \end{aligned}
 \label{eq:length-weighted}
\end{equation}
with the sampler law, recorded lengths, and normalizers held fixed during differentiation.
\end{proposition}
The proof is in Appendix~\ref{app:sequence-proof}. Each response residual now multiplies the mean token score of that response. Positive weights leave the target alone but change how a restricted model trades off errors across responses: at equal absolute residual, a response of length $100$ contributes ten times as much loss as one of length $1000$. The batch reduction is $B^{-1}\sum_i d_i^2/T_i$; dividing $\sum_i d_i^2$ by the total token count is a different reduction. The accumulated log-ratio can still grow with length, and Appendix~\ref{app:length-gradient} gives the gradient bound and explains why the weighted update differs from the return gradient even at matched policies.

\subsection{Replay across Sampler Versions}
\label{sec:sequence-batch}
In replay, each response also carries its fixed-target normalizer, or enough information to fit one.

\Needspace{10\baselineskip}
\begin{proposition}[Version-Conditioned Sequence Update]
\label{prop:sequence-batch}
Holding these quantities fixed, the empirical length-weighted loss and its gradient are
\begin{equation}
 \begin{aligned}
 d_{i,\theta}
 &=\beta\log\frac{\pitrainer_\theta(y_i\mid x_i)}{\pisampler_{\theta_{t-k_i(t)}}(y_i\mid x_i)}-R_i+c^*_{t,k_i(t)}(x_i),\\
 \widehat{\mathcal L}^{\mathrm{len}}_{\mathrm{seq},t}(\theta)
 &=\frac1{2\beta B}\sum_{i=1}^{B}\frac{d_{i,\theta}^2}{T_i},\\
 \nabla_\theta\widehat{\mathcal L}^{\mathrm{len}}_{\mathrm{seq},t}(\theta)
 &=\frac1B\sum_{i=1}^{B}\frac{d_{i,\theta}}{T_i}\,\nabla_\theta\log\pitrainer_\theta(y_i\mid x_i).
 \end{aligned}
 \label{eq:mixture-regression}
\end{equation}
\end{proposition}
The proof is in Appendix~\ref{app:batch-proof}. At $\theta=\theta_t$, collection frequencies and response lengths determine each sampler's contribution, with no extra version weight or importance multiplier. The gradient is generally not the on-policy return gradient.

\begin{proposition}[Compatibility across Sampler Versions]
\label{prop:offpolicy-target}
For a fixed version, full response support, exact normalization, and realizability give the unique zero-loss target $\pistar$, with or without positive outer weights. If several versions occur with positive frequency and have full common support, their joint population loss is zero exactly when all their Gibbs targets coincide and the trainer equals that common target. For a common reward and $\beta$, two full-support samplers have the same Gibbs target only if they are equal.
\end{proposition}
The proof is in Appendix~\ref{app:version-proof}. Distinct samplers therefore pull the trainer toward competing targets, and their relative frequencies affect the best fit even with unrestricted capacity. A finite batch does not by itself provide the population coverage assumed here. Each response should use its own $c^*_{t,k_i(t)}(x_i)$, or a weighted intercept fitted within its prompt--version group; pooling normalizers across versions changes the residual (Appendix~\ref{sec:weighted-centering}).

\subsection{The SKLPO Update}
\label{sec:sequence-update}
Algorithm~\ref{alg:klpo} applies the length-weighted loss to completed responses. For response $i$, $\ell_{i,u}$ is the token log-ratio of Equation~\eqref{eq:sequence}, evaluated with the current trainer and the stored collection version $v_i$. The normalizer can be supplied directly, learned as an intercept, or estimated by weighted centering within each prompt and sampler version. Fixed normalizers stay unchanged during reuse, while fitted intercepts track the trainer. Since a one-response group has zero centered loss, single-rollout training needs independent normalization information or a predictor shared across prompts; Appendix~\ref{sec:normalizer} gives the estimators and their assumptions.

\begin{algorithm}[htbp]
\caption{SKLPO with Each Response's Historical Sampler}
\label{alg:klpo}
\small
\begin{algorithmic}[1]
\Require Trainer $\pitrainer_{\theta_0}$, completed-response queue, $\beta>0$, normalization method
\For{optimizer steps $t=0,1,\ldots$}
 \State Read $B$ responses $(x_i,y_i,T_i,R_i,v_i)$ and their stored sampler log-probabilities.
 \State Score with $\pitrainer_{\theta_t}$; compute $\ell_i=\sum_{u=1}^{T_i}\ell_{i,u}$, $a_i=1/T_i$, $h_i=R_i-\beta\ell_i$.
 \If{using a supplied prompt-and-version normalizer}
  \State Set $d_i=\widehat c_{v_i}(x_i)-h_i$; keep $\widehat c_{v_i}$ detached and fixed during reuse.
 \ElsIf{using a learned intercept $c_\phi$}
  \State Set $d_i=\sg(c_\phi(x_i,v_i))-h_i$ before updating either model.
 \Else
  \State Group by $(x_i,v_i)$ with $G_{x,v}\geq2$.
  \State For each group, set $\bar h_a=(\sum_j a_jh_j)/(\sum_j a_j)$ and $d_i=\bar h_a-h_i$.
 \EndIf
 \State Update $\theta_t$ using the gradient of $\frac{1}{2\beta B}\sum_{i=1}^{B}a_id_i^2$.
 \If{using a learned intercept}
  \State Update $\phi$ by Equation~\eqref{eq:intercept-update}, using the pre-update $h_i$ and $a_i$.
 \EndIf
 \State Publish checkpoints on the synchronization schedule; retain historical sampler log-probabilities.
\EndFor
\end{algorithmic}
\end{algorithm}
\FloatBarrier

%% file: contents/normalization.tex
\section{Normalizing the Response-Level Target}
\label{sec:normalizer}
SKLPO needs one normalizer $c^*_{t,k}(x)$ per prompt and collection version. The fixed normalizer depends on neither the sampled response nor its length, whereas a fitted intercept changes with the trainer. This appendix compares the ways to supply it, and closes with the fixed and terminal-reward forms of the token-level objective for comparison. All statements condition on one sampler version.

\subsection{Normalization Choices}
SKLPO does not require a learned value model. Normalization information can come from repeated responses or from independent statistics, and learning a prompt normalizer is one implementation choice among several. A learned normalizer is an auxiliary scalar predictor, not a token-value critic. At the exact target, $c^*(x)$ is the prompt's soft value; a fitted intercept $c_\theta(x)$ generally takes a different value during training.

\begin{table}[ht]
\centering
\small
\renewcommand{\arraystretch}{1.15}
\begin{tabularx}{\textwidth}{@{}lXX@{}}
\toprule
Method & Information used & Policy-update schedule \\
\midrule
Direct estimate & Same-prompt, same-version rewards or independent success statistics & Estimate $\widehat c$; hold it fixed while reusing the data. \\
Learned normalizer & A shared predictor of success probability or exponentiated reward & Predict $\widehat c$ out of sample; hold that prediction fixed during reuse. \\
Group centering & At least two responses for the same prompt and version & Recompute the weighted group mean after each trainer update. \\
Learned intercept & A shared predictor of the conditional mean of $R-\beta\ell_\theta$ & Alternate policy and intercept updates, using the same outer weights. \\
\bottomrule
\end{tabularx}
\caption{Ways to normalize the response-level target. The two learned routes add a prompt-conditioned predictor; the other two do not. All four condition on the collection sampler.}
\label{tab:normalization-routes}
\end{table}

Single-rollout \emph{collection} means one response per prompt per collection event. Such data can still be normalized with independent historical data or with a predictor shared across prompts, but historical rewards estimate the right normalizer only if they came from the same collection policy. With one response and no other information, the plug-in and empirical-centering constructions below cannot produce the intended update. Appendix~\ref{app:learned-normalizers} specifies the learned routes and their update order.

\subsection{Fixed Normalizers}
Let $V^{\mathrm{sampler}}_{t,k}(x)=\EE_{\pisamp}R(x,Y)$ and $A^{\mathrm{sampler}}_{t,k}(x,y)=R(x,y)-V^{\mathrm{sampler}}_{t,k}(x)$.
\begin{lemma}[Response Normalizer and Plug-In Estimation]
\label{lem:normalizer}
The exact prompt normalizer satisfies
\begin{equation}
 c^*(x)=V^{\mathrm{sampler}}_{t,k}(x)+\beta\KL(\pisamp(\cdot\mid x)\|\pistar(\cdot\mid x)).
 \label{eq:softvalue}
\end{equation}
For $G$ independent responses to the same prompt from $\pisamp$, define
\begin{equation}
 \widehat c(x)=\beta\left[\operatorname{logsumexp}_{i=1}^G
 \left(\frac{R(x,y_i)}{\beta}\right)-\log G\right].
 \label{eq:mc}
\end{equation}
Then $e^{\widehat c/\beta}$ is an unbiased estimator of $Z$, while $\EE\widehat c\leq c^*$.
\end{lemma}
The proof is in Appendix~\ref{app:normalizer-proof}. Equation~\eqref{eq:softvalue} is Equation~\eqref{eq:normalizer-kl} with $f=R$; a mean-reward baseline omits its KL term. Smaller $\beta$ makes the exponential average more sensitive to rare high rewards. Holding $\widehat c$ fixed defines a sampled regression target, not an unbiased estimate of the exact-normalizer loss. Appendix~\ref{app:normalizer} gives the binary-reward form and the effect of data selection.

\subsection{Profiling within a Prompt}
For unweighted, unfiltered responses from one sampler, write $h_\theta=R-\beta\ell_\theta$, and let $\kappa_\theta(x)$ be the complete-response KL from Appendix~\ref{sec:sequence}. The following is Proposition~\ref{prop:profiling} with $f=R$.
\begin{proposition}[Profiled SKLPO Regression]
\label{prop:sequence-centering}
At fixed $\theta$, the least-squares prompt intercept is
\begin{equation}
 c_\theta(x)=\EE_{\pisamp}[h_\theta(x,Y)]
 =\EE_{\pisamp}R+\beta\KL(\pisamp\|\pi_\theta).
 \label{eq:profilec}
\end{equation}
The centered residual and loss are
\begin{equation}
 \delta_{\mathrm{KL}}(x,y;\theta)
 =\frac{A^{\mathrm{sampler}}_{t,k}(x,y)}{\beta}-\ell_\theta(x,y)-\kappa_\theta(x),
 \label{eq:klresidual}
\end{equation}
\begin{equation}
 \begin{aligned}
 \mathcal L_{\mathrm{ctr},\pisamp}(\theta)
 &=\frac{\beta}{2}\EE_x\EE_{Y\sim\pisamp(\cdot\mid x)}[\delta_{\mathrm{KL}}(x,Y;\theta)^2]\\
 &=\frac{1}{2\beta}\EE_x\Var_{Y\sim \pisamp(\cdot\mid x)}
 [R(x,Y)-\beta\ell_\theta(x,Y)].
 \end{aligned}
 \label{eq:centered}
\end{equation}
Under full support and realizability, the unique zero-loss policy is $\pistar$, and its fitted intercept equals $c^*$.
\end{proposition}
The proof is in Appendix~\ref{app:profiling-proof}. Away from the target, $c_\theta$ generally differs from $c^*$. For comparison with the token objectives, the same residual and loss can be written as
\begin{equation}
 \delta_{\mathrm{seq}}=\frac{R-V^{\mathrm{sampler}}_{t,k}(x)}{\beta}
 -\sum_u\ell_{\theta,u}-\kappa_\theta(x),\qquad
 \mathcal L_{\mathrm{seq}}=\frac{\beta}{2}\EE_{x,Y\sim\pisamp}[\delta_{\mathrm{seq}}^2],
 \label{eq:sequencegranularity}
\end{equation}
so that $\delta_{\mathrm{seq}}=\delta_{\mathrm{KL}}$ and $\mathcal L_{\mathrm{seq}}=\mathcal L_{\mathrm{ctr},\pisamp}$. The fixed-normalizer version replaces $\kappa_\theta$ by the KL to the Gibbs target, $\KL(\pisamp\|\pistar)$.

In practice, profiling means centering within a group of responses to the same prompt.
\begin{lemma}[Empirical Centering]
\label{lem:empirical-centering}
For $G$ responses in one prompt and sampler-version group, set $h_i=R_i-\beta\ell_i$ and $\bar h=G^{-1}\sum_i h_i$. Then
\begin{equation}
 \widehat{\mathcal L}_{\mathrm{ctr}}
 =\frac{1}{2\beta G}\sum_i(h_i-\bar h)^2
 =\frac{1}{4\beta G^2}\sum_{i,j}(h_i-h_j)^2.
 \label{eq:pairwise}
\end{equation}
Differentiating through $\bar h$ and detaching it give the same gradient. For fixed $\theta$ and i.i.d.\ responses, the expected empirical loss is $(G-1)/G$ times the conditional population loss.
\end{lemma}
The proof is in Appendix~\ref{app:empirical-proof}. The centered form costs $O(G)$ after scoring, and the pairwise form is the relative-reward regression of REBEL \citep{rebel2024}. The mean $\bar h$ must be recomputed at every optimizer step. Pooling sampler versions changes the component residuals, and outcome-dependent selection changes the centering measure, which can invalidate the KL expression in Equation~\eqref{eq:profilec}; Appendix~\ref{app:behavior-centering} gives the corresponding gradient. The empirical KL estimate $-\bar\ell$ can be negative, and clipping it changes the loss.

\begin{proposition}[Singleton Normalization]
\label{prop:singleton}
For $G=1$, Equation~\eqref{eq:mc} gives $\widehat c=R_1$, and hence a zero fixed-normalizer residual at $\pi_\theta=\pisamp$. Empirical centering of a one-response group gives zero loss at every policy, with or without a positive outer length weight.
\end{proposition}
The proof is in Appendix~\ref{app:singleton-proof}. Fixed-normalizer training on single responses therefore needs normalization information obtained independently of the response being scored. A shared learned intercept is different: it estimates a conditional mean across training examples rather than computing a singleton group mean. A reward mean across prompts cannot replace the prompt normalizer. Appendix~\ref{app:single-rollout} gives the signed-feedback and normalizer-error identities.

\subsection{Centering with Length Weights}
\label{sec:weighted-centering}
For length-weighted training, the intercept should be profiled with the same outer weight that the loss uses.
\begin{proposition}[Weighted Sequence Centering]
\label{prop:weighted-centering}
At a fixed prompt and sampler version, set $a=1/T$ and $h_\theta=R-\beta\ell_\theta$. The population intercept and profiled loss are
\begin{equation}
 c_{\theta,a}(x)=
 \frac{\EE_{\pisamp}[a\,h_\theta\mid x]}{\EE_{\pisamp}[a\mid x]},
 \qquad
 \mathcal L_{\mathrm{ctr},a}
 =\frac1{2\beta}\EE_{x,Y\sim\pisamp}
       [a(h_\theta-c_{\theta,a}(x))^2].
 \label{eq:weighted-centering}
\end{equation}
For a group of $G$ responses, let $a_i=1/T_i$, $W=\sum_i a_i$, and $\bar h_a=W^{-1}\sum_i a_i h_i$. Then
\begin{equation}
 \widehat{\mathcal L}_{\mathrm{ctr},a}
 =\frac1{2\beta G}\sum_i a_i(h_i-\bar h_a)^2
 =\frac1{4\beta GW}\sum_{i,j}a_i a_j(h_i-h_j)^2.
 \label{eq:weighted-pairwise}
\end{equation}
The empirical gradient is $G^{-1}\sum_i a_i(\bar h_a-h_i)\score(y_i)$ whether $\bar h_a$ is differentiated or detached. Under full support and realizability, the unique zero-loss policy remains $\pistar$, with intercept $c^*$.
\end{proposition}
The proof is in Appendix~\ref{app:weighted-profiling}. The weighted mean must be recomputed at every trainer update, and it generally differs from both the ordinary mean and the fixed $c^*$. Direct estimation of $c^*$ by Equation~\eqref{eq:mc} still uses the original, unweighted sampler distribution, and the unweighted $(G-1)/G$ bias formula does not carry over to random-length weights.

\subsection{Fixed Token Normalizers}
\label{sec:token-normalization}
For comparison, the token-level analogue of a fixed normalizer is the log-partition $c^*_{\mathrm{tok}}$ of Equation~\eqref{eq:localnormalizer}. Given exact sampler advantages, it can be computed over next actions or estimated from independent actions at the same history, but the advantages themselves require continuation information from the same sampler. Holding the advantage and normalizer fixed, the token residual, loss, and gradient are
\begin{equation}
 \begin{aligned}
 d^{\mathrm{tok}}_{\theta,u}
 &=\beta\ell_{\theta,u}-A^{\mathrm{sampler}}_{t,k}(x,y_{<u},y_u)+c^*_{\mathrm{tok}}(x,y_{<u}),\\
 \mathcal L_{\mathrm{tok}}^{\mathrm{fix}}(\theta)
 &=\frac1{2\beta}\EE_{x,Y\sim\pisamp}\sum_u(d^{\mathrm{tok}}_{\theta,u})^2,\\
 \nabla_\theta\mathcal L_{\mathrm{tok}}^{\mathrm{fix}}
 &=\EE_{x,Y\sim\pisamp}\sum_u d^{\mathrm{tok}}_{\theta,u}
       \nabla_\theta\log\pi_\theta(y_u\mid x,y_{<u}).
 \end{aligned}
 \label{eq:token-fixed-loss}
\end{equation}
Under realizability, this loss vanishes exactly at $\pi_\theta=\pi^+_{\mathrm{tok}}$, by Theorem~\ref{thm:common-regression}. The fixed $c^*_{\mathrm{tok}}$ and the profiled intercept $\beta k_\theta(s)$ agree at the PMD target and differ during training. A single logged action probability determines the sampled log-ratio, but neither a fixed normalizer nor the conditional score mean; Section~\ref{sec:kl-estimators} discusses what each KL estimator stores instead.

\subsection{A Terminal-Reward Form of the Token Objective}
\label{sec:token-mc}
One can also replace the advantage by the terminal reward inside a squared residual that keeps the full local KL. The resulting loss has the same population gradient as $\mathcal L_{\mathrm{tok}}$.
\begin{proposition}[Terminal-Reward Token Regression]
\label{prop:token-mc}
Let $b(s)$ be a fixed, action-independent baseline, and use complete rollouts from the specified sampler. Define
\begin{equation}
 \widehat\delta_u=\frac{R-b(x,y_{<u})}{\beta}-\ell_{\theta,u}-k_\theta(x,y_{<u}),\qquad
 \mathcal L_{\mathrm{MC}}=\frac{\beta}{2}\EE_{x,Y\sim\pisamp}\sum_u\widehat\delta_u^2.
 \label{eq:tokenmc}
\end{equation}
With exact local KLs, $\mathcal L_{\mathrm{MC}}-\mathcal L_{\mathrm{tok}}$ does not depend on $\theta$, so the population gradients agree, including for $b=0$. The identity differentiates through $k_\theta$ and holds the rollout law and the baseline fixed.
\end{proposition}
The proof is in Appendix~\ref{app:token-mc-proof}. A prompt baseline shared by all tokens preserves the population gradient while adding noise and giving up the samplewise Bellman telescope. When $b(s)\ne V^{\mathrm{sampler}}_{t,k}(s)$, discarding the KL derivative generally changes the gradient, and a baseline fitted on the same rollout needs a separate argument. The local KL requires either the full next-token distributions or an estimator built from additional independent draws. Using the sampled action itself, $\widehat k_\theta(x,y_{<u})=-\ell_{\theta,u}$, cancels the log-ratio and removes the learning signal.

%% file: contents/qstar.tex
\section{Implicit \texorpdfstring{$Q^\star$}{Q-Star} Learning}
\label{sec:qstar}
This appendix shows that the response Gibbs policy of Appendix~\ref{sec:sequence} represents an optimal regularized action-value function implicitly. For a fixed collection sampler $\pisamp$, let $Q^\star_{\beta,t,k}$ and $V^\star_{\beta,t,k}$ be the optimal values regularized toward that sampler. These differ both from the unregularized optimum and from the ordinary sampler value $Q^{\mathrm{sampler}}_{t,k}$ that drives token-level KLPO. The results apply to SKLPO and to soft-value token regression; local PMD targets a different policy and does not recover these values. They are interpretation rather than a component of the token-level updates in the main text.

\subsection{The Regularized Bellman Target}
Consider a finite-horizon process with finite state and action spaces, bounded rewards, and a positive reference $\pisamp(a\mid s)$. The state $s$ includes time and the full observed history, and $P(s'\mid s,a)$ is the environment's transition law. We keep the terminal-reward setting: intermediate rewards are zero and $R(S_T)$ is paid at termination. Define
\begin{equation}
 V^\star_{\beta,t,k}(s)
 =\max_\pi\EE_{\pi,P}\!\left[
 R(S_T)-\beta\sum_{u<T}\log\frac{\pi(A_u\mid S_u)}{\pisamp(A_u\mid S_u)}
 \,\middle|\,S_0=s\right].
 \label{eq:qstar-objective}
\end{equation}
\Needspace{8\baselineskip}
\begin{proposition}[Sampler-Regularized Bellman Equations]
\label{prop:soft-bellman}
In this setting, the optimal values and policy satisfy
\begin{equation}
 \begin{aligned}
 Q^\star_{\beta,t,k}(s,a)&=\EE_{S'\sim P(\cdot\mid s,a)}V^\star_{\beta,t,k}(S'),\\
 V^\star_{\beta,t,k}(s)&=\beta\log\sum_a\pisamp(a\mid s)
       e^{Q^\star_{\beta,t,k}(s,a)/\beta},\\
 \pi^\star_{\beta,t,k}(a\mid s)&=\pisamp(a\mid s)
       e^{(Q^\star_{\beta,t,k}(s,a)-V^\star_{\beta,t,k}(s))/\beta},
 \end{aligned}
 \label{eq:qstar-bellman}
\end{equation}
with $V^\star_{\beta,t,k}(s_T)=R(s_T)$. For a fixed positive sampler, the regularized values converge to the unregularized optimal values as $\beta\downarrow0$.
\end{proposition}
The proof is in Appendix~\ref{app:bellman-proof}. These are the regularized Bellman equations of \citet{geist2019regularized} for the penalty $\beta\KL(\pi\|\pisamp)$, and the reference-weighted form of the Boltzmann policy in soft Q-learning \citep{haarnoja2017softq}. Each state's policy is again the Gibbs solution of Lemma~\ref{lem:common-target}, now with signal $Q^\star_{\beta,t,k}$. The expectation over $P$ keeps environment outcomes conditioned on the chosen action. At finite $\beta$, future KL costs are part of the continuation value. Exponentiating the ordinary $Q^\star$ omits them, and exponentiating $Q^{\mathrm{sampler}}_{t,k}$ gives the local PMD update; neither solves this Bellman system in general.

\subsection{Soft-Value Token Regression}
\label{sec:soft-token}
In deterministic generation, soft continuation values factor the response Gibbs target into token conditionals. Regressing those local conditions gives a token-level objective with the same target as SKLPO. It needs the soft value at every prefix, which terminal returns with score centering do not provide.
\begin{theorem}[Soft-Value Token Regression]
\label{thm:soft-token}
Under Assumptions~\ref{ass:finite} and~\ref{ass:deterministic}, define
\begin{equation}
 F(s)=\beta\log\EE_{\pisamp}[e^{R/\beta}\mid s],\qquad
 \pistar(a\mid s)=\pisamp(a\mid s)e^{(F(sa)-F(s))/\beta},
 \label{eq:softfactorization}
\end{equation}
where $sa$ is the successor prefix, $F(x)=c^*(x)$, and $F=R$ at termination. The local residuals satisfy
\begin{equation}
 \epsilon_u=\frac{F(s_{u+1})-F(s_u)}{\beta}-\ell_{\theta,u},\qquad
 \sum_u\epsilon_u=\frac{R-c^*(x)}{\beta}-\ell_\theta(x,y).
 \label{eq:softtokenresidual}
\end{equation}
For fixed exact $F$ and a fixed sampler,
\begin{equation}
 \mathcal L_{\mathrm{soft\text{-}tok}}=\frac\beta2\EE_{\pisamp}\sum_u\epsilon_u^2,\qquad
 \nabla_\theta\mathcal L_{\mathrm{soft\text{-}tok}}
 =-\beta\EE_{\pisamp}\sum_u\epsilon_u\nabla_\theta\log\pi_\theta(a_u\mid s_u).
 \label{eq:soft-token-loss}
\end{equation}
Under realizability, the unique zero-loss policy is $\pistar$, the SKLPO target.
\end{theorem}
The proof is in Appendix~\ref{app:soft-token-proof}. Since $d_\theta=-\beta\sum_u\epsilon_u$, SKLPO squares the sum of the local residuals while soft-value token regression sums their squares. Unlike the PMD residuals of Theorem~\ref{thm:token-bpo}, these residuals are not conditionally centered, so the cross terms survive in expectation and the gradients differ. Soft-value token regression needs $F$ at every prefix; SKLPO needs only $R$ and $F(x)=c^*(x)$. In deterministic generation, these soft values solve Equation~\eqref{eq:qstar-bellman}. The loss sums over generated policy tokens and averages over responses, and Appendix~\ref{app:token-reductions} distinguishes a global token mean from response-length weighting.

\begin{figure}[ht]
 \centering
 \input{figures/objectives}
 \caption{\textbf{Response-level and token-level targets.} SKLPO and soft-value token regression share the response Gibbs target under exact fitting in deterministic generation. Token-level KLPO targets local PMD; its token- and sequence-regression forms have the same population gradient under the stated conditions and need no critic.}
 \label{fig:objectives}
\end{figure}

\subsection{The Implicit Action-Value Representation}
Define the policy's implicit soft advantage by
\begin{equation}
 A^{\mathrm{imp}}_\theta(s,a)
 =\beta\log\frac{\pi_\theta(a\mid s)}{\pisamp(a\mid s)}.
 \label{eq:implicit-advantage}
\end{equation}
Normalization of the policy gives $\sum_a\pisamp(a\mid s)e^{A^{\mathrm{imp}}_\theta(s,a)/\beta}=1$. The log-ratio therefore represents differences of action values, and recovering absolute Q-values requires a Bellman-consistent state baseline as well.

\begin{theorem}[Implicit Recovery of Optimal Soft Advantages]
\label{prop:implicit-qstar}
Under Assumptions~\ref{ass:finite} and~\ref{ass:deterministic}, the regularized Bellman values are
\begin{equation}
 V^\star_{\beta,t,k}(s)=F(s),\qquad
 Q^\star_{\beta,t,k}(s,a)=F(sa),\qquad
 F(s)=\beta\log\EE_{\pisamp}[e^{R/\beta}\mid s].
 \label{eq:qstar-softvalue}
\end{equation}
Suppose replay covers every feasible response and the Gibbs target is realizable. With exact $c^*(x)$ for SKLPO, or exact $F$ for soft-value token regression, any zero-loss fit satisfies, at every covered prefix,
\begin{equation}
 A^{\mathrm{imp}}_\theta(s,a)
 =Q^\star_{\beta,t,k}(s,a)-V^\star_{\beta,t,k}(s)
 =F(sa)-F(s).
 \label{eq:implicit-recovery}
\end{equation}
The same holds with fixed positive outer weights on either regression loss.
\end{theorem}
The proof is in Appendix~\ref{app:implicit-proof}. The theorem describes an exact fit, not the outcome of neural optimization. The log-ratio of an intermediate policy need not be Bellman-consistent, and absolute action values still require $Q^\star_{\beta,t,k}=V^\star_{\beta,t,k}+A^{\mathrm{imp}}_\theta$.

With stochastic tools, the deterministic equivalence can fail, because an unrestricted trajectory tilt can change outcomes outside the policy's control. Appendix~\ref{app:stochastic-counterexample} gives a one-action counterexample. The Bellman equations remain feasible because they keep $P$ fixed.

\subsection{Policy Regression with an Optimal Q-Function}
When optimal regularized Q-values are available, the regression principle applies directly to action policies, including under stochastic transitions.
\begin{proposition}[Regression against Optimal Soft Q-Values]
\label{prop:q-regression}
Hold the reference and the exact target values fixed. For a fixed replay law $\zeta$, define
\begin{equation}
 \begin{aligned}
 d^Q_\theta(s,a)&=\beta\log\frac{\pi_\theta(a\mid s)}{\pisamp(a\mid s)}
       -Q^\star_{\beta,t,k}(s,a)+V^\star_{\beta,t,k}(s),\\
 \mathcal L_Q(\theta)&=\frac{1}{2\beta}\EE_{(S,A)\sim\zeta}[d^Q_\theta(S,A)^2],\\
 \nabla_\theta\mathcal L_Q&=\EE_\zeta[d^Q_\theta(S,A)\nabla_\theta\log\pi_\theta(A\mid S)].
 \end{aligned}
 \label{eq:qstar-regression}
\end{equation}
Under state-action coverage and realizability, zero loss identifies $\pi^\star_{\beta,t,k}$. Its state normalizer satisfies
\begin{equation}
 V^\star_{\beta,t,k}(s)
 =\EE_{a\sim\pisamp(\cdot\mid s)}Q^\star_{\beta,t,k}(s,a)
  +\beta\KL(\pisamp(\cdot\mid s)\|\pi^\star_{\beta,t,k}(\cdot\mid s)).
 \label{eq:qstar-normalizer}
\end{equation}
In deterministic generation, $d^Q_\theta=-\beta\epsilon_u$, which recovers the soft-value token objective.
\end{proposition}
The proof is in Appendix~\ref{app:q-regression-proof}. Equation~\eqref{eq:qstar-normalizer} is Equation~\eqref{eq:normalizer-kl} once more. Appendix~\ref{app:learning-soft-targets} sketches a critic extension with sampled Bellman targets. A sampler's terminal return estimates its ordinary continuation value and cannot in general replace the optimal regularized target.

%% file: figures/objectives.tex
\begingroup
\input{figures/style}
\resizebox{\linewidth}{!}{%
\begin{tikzpicture}[x=1cm,y=1cm,>=Latex,font=\sffamily\small]
 \path[use as bounding box] (0,0) rectangle (16,5.45);
 \foreach \xx/\panelcolor in {0/klpoOrange,5.43/klpoBlue,10.86/klpoBlue} {
  \fill[\panelcolor!4,rounded corners=5pt] (\xx,.5) rectangle (\xx+5.14,5.45);
 }

 \node[klpo heading,text=klpoOrange] at (.22,5.08) {Token-level KLPO};
 \node[klpo heading,text=klpoBlue] at (5.65,5.08) {SKLPO};
 \node[klpo heading,text=klpoBlue] at (11.08,5.08) {Soft-value token regression};

 \node[klpo note,align=center] at (2.57,4.38) {Sampler advantage $A^{\mathrm{sampler}}_{t,k}$\\profiled PMD regression};
 \node[klpo note,align=center] at (8,4.38) {Terminal reward $R$\\prompt normalization};
 \node[klpo note,align=center] at (13.43,4.38) {Soft continuation value $F(x,y_{<u})$\\at each prefix};

 \node[klpo fit,minimum width=4.42cm,minimum height=.86cm] (tok) at (2.57,3.13)
   {Profiled token loss\\$\frac{\beta}{2}\,\EE_{\pisamp}\sum_u\delta_u^2$};
 \node[klpo box,minimum width=4.42cm,minimum height=.86cm] (seq) at (8,3.13)
   {One response residual\\$\frac{1}{2\beta}\,\EE_{\pisamp}[d_\theta^2]$};
 \node[klpo box,minimum width=4.42cm,minimum height=.86cm] (soft) at (13.43,3.13)
   {Soft-value increments\\$\frac{\beta}{2}\,\EE_{\pisamp}\sum_u\epsilon_u^2$};

 \foreach \xx/\name/\panelcolor in {2.57/tok/klpoOrange,8/seq/klpoBlue,13.43/soft/klpoBlue} {
  \draw[->,draw=\panelcolor,line width=.85pt] (\xx,3.96)--(\name.north);
  \draw[->,draw=\panelcolor,line width=.85pt] (\name.south)--(\xx,2.14);
 }

 \node[text=klpoOrange,align=center] at (2.57,1.78) {Local PMD target\\[3pt]$\pi^+_{\rm tok}$};
 \node[text=klpoBlue,align=center] at (8,1.78) {Optimal soft-Q policy\\[3pt]$\pi^\star$};
 \node[text=klpoBlue,align=center] at (13.43,1.78) {Optimal soft-Q policy\\[3pt]$\pi^\star$};

 \node[klpo note,align=center] at (2.57,.87) {Token or sequence\\regression$^\dagger$};
 \node[klpo note,align=center] at (8,.87) {Response-level Gibbs view};
 \node[klpo note,align=center] at (13.43,.87) {Same Gibbs target\\different regression loss};

 \node[anchor=west,klpo note] at (.08,.13)
  {$^\dagger$Both routes: complete matched-sampler rollouts with exact centering or full local KLs.};
\end{tikzpicture}%
}
\endgroup

%% file: contents/token-supplement.tex
\section{Additional Details for the Main Text}
\label{app:token-supplement}
This appendix complements Sections~\ref{sec:framework}--\ref{sec:agentic}. It compares the two instances of KLPO, summarizes the sequence-regression route, states the token-regression surrogate, contrasts sequence regression with the response-level instance, gives the TopK-KL and binary-KL surrogates, and collects the single-rollout requirements of the two instances.

\subsection{The Two Instances}
Table~\ref{tab:granularity} places the response-level and token-level instances of Section~\ref{sec:instances} side by side.

\begin{table}[H]
\centering
\small
\renewcommand{\arraystretch}{1.2}
\begin{tabularx}{\textwidth}{@{}l>{\raggedright\arraybackslash}X>{\raggedright\arraybackslash}X@{}}
\toprule
 & Response level (SKLPO) & Token level (local PMD) \\
\midrule
Context, outcome & prompt $x$, response $y$ & history $s$, next token $a$ \\
Signal $f$ & reward $R(x,y)$ & sampler advantage $A(s,a)$ \\
Target & $\pistar(y\mid x)\propto q(y\mid x)\,e^{R(x,y)/\beta}$ & $\pi^+_{\mathrm{tok}}(a\mid s)\propto q(a\mid s)\,e^{A(s,a)/\beta}$ \\
Fixed normalizer & $c^*(x)=\beta\log\EE_{q}[e^{R/\beta}\mid x]$ & $c^*_{\mathrm{tok}}(s)=\beta\log\EE_{q}[e^{A/\beta}\mid s]$ \\
Profiled intercept & $\EE_q[R\mid x]+\beta\KL\big(q(\cdot\mid x)\|\pi_\theta(\cdot\mid x)\big)$ & $\beta\KL\big(q(\cdot\mid s)\|\pi_\theta(\cdot\mid s)\big)$ \\
Still required & the prompt's mean reward and response KL, in practice several responses per prompt & sampler advantages and the local KL \\
\bottomrule
\end{tabularx}
\caption{The two instances of KLPO, with $q=\pisamp$ and $A=A^{\mathrm{sampler}}_{t,k}$. Profiling removes the fixed normalizer from both. The response-level intercept still needs quantities that a single response cannot supply; the token-level intercept needs only the local KL, and Section~\ref{sec:granularity} removes the advantages from the residual.}
\label{tab:granularity}
\end{table}

\subsection{The Sequence-Regression Route}
Figure~\ref{fig:score-alternative} traces sequence regression, the alternative to Steps~7--9 of Figure~\ref{fig:overview}, through the results of Section~\ref{sec:tokenbpo} and the MC-KL estimator of Section~\ref{sec:mc-kl}.

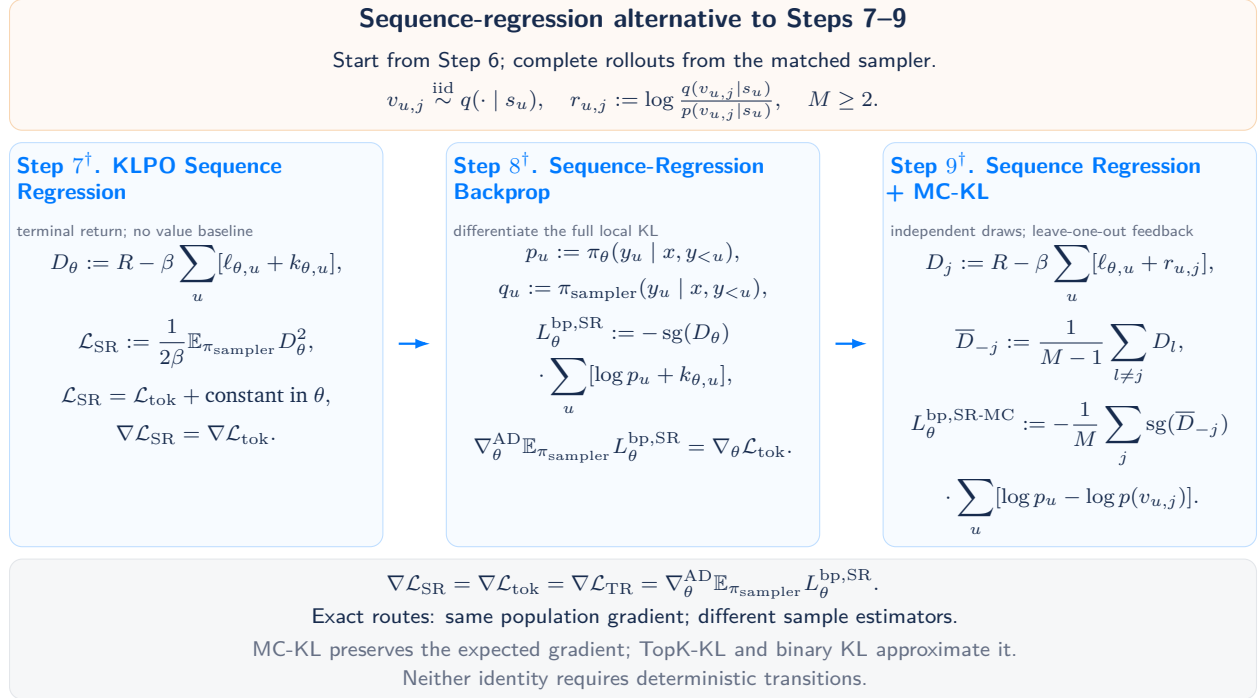
\begin{figure}[H]
 \centering
 \input{figures/score-alternative}
 \caption{\textbf{Sequence regression, an alternative to Steps~7--9 of Figure~\ref{fig:overview}.} Summing the token residuals of Step~6 along a trajectory gives one trajectory residual, which Bellman telescoping identifies with BPO's under deterministic transitions (Step~$7^\dagger$; \citealp{bpo2026}). The prompt value can be dropped, and random tool outputs add only a $\theta$-independent term, so the gradient is unchanged. Step~$8^\dagger$ differentiates through the full local KLs. Step~$9^\dagger$ estimates them with $M\geq2$ independent draws per prefix and uses leave-one-out residuals to keep the gradient unbiased. The two routes share a population gradient but differ sample by sample.}
 \label{fig:score-alternative}
\end{figure}

\subsection{The Token-Regression Surrogate}
The per-token surrogate of Equation~\eqref{eq:token-sc-surrogate} produces the gradient of Theorem~\ref{thm:token-sc} by automatic differentiation.
\begin{proposition}[Backward-Surrogate Equivalence]
\label{prop:token-sc-surrogate}
Hold the sampler fixed, detach $h_{\theta,u}$ and the sampler probabilities in Equation~\eqref{eq:token-sc-surrogate}, and differentiate both trainer log-probability terms. Then
\begin{equation}
 \nabla_\theta^{\mathrm{AD}}
 \EE_{x,Y\sim\pisamp}\sum_u L^{\mathrm{bp,TR}}_{\theta,u}
 =-\EE_{x,Y\sim\pisamp}\sum_u
 h_{\theta,u}\,\widetilde g_\theta(s_u,a_u)
 =\nabla_\theta\mathcal L_{\mathrm{TR}}
 =\nabla_\theta\mathcal L_{\mathrm{tok}}.
 \label{eq:token-sc-surrogate-equivalence}
\end{equation}
\end{proposition}
The proof is in Appendix~\ref{app:token-surrogate-proof}.

\subsection{Sequence Regression and the Response-Level Instance}
\label{sec:sr-vs-sklpo}
Sequence regression and the response-level instance (SKLPO) both square one residual per response; they differ in the correction and in the target. With full local KLs,
\begin{equation}
 \begin{aligned}
 D_{\mathrm{SR}}(x,Y)
 &=R(x,Y)-\beta\sum_u\ell_{\theta,u}-\beta\sum_u k_\theta(s_u),\\
 D_{\mathrm{SKLPO}}(x,Y)
 &=R(x,Y)-\beta\sum_u\ell_{\theta,u}-c^*(x),\qquad
 c^*(x)=\beta\log\EE_{Y\sim\pisamp(\cdot\mid x)}e^{R(x,Y)/\beta}.
 \end{aligned}
 \label{eq:klpo-sklpo-residuals}
\end{equation}
Sequence regression subtracts the path sum of current local KLs, which depends on the trainer and is differentiated. SKLPO subtracts a prompt normalizer that is fixed with respect to the trainer and must be supplied. The targets differ too: one is local PMD, the other the response Gibbs policy, and they coincide only in the special case noted in Section~\ref{sec:instances}.

\subsection{TopK-KL and Binary-KL Surrogates}
\label{app:topk-binary}
\paragraph{TopK-KL.}
TopK-KL stores the head $H_u$ with the sampler's original probabilities, together with the sampled token's own log-probability even when that token lies outside the head. With $g_v=\nabla_\theta\log p_v$ and $\rho=q_{\mathrm{tail}}/p_{\mathrm{tail}}$, the gradient of Equation~\eqref{eq:topk-kl} is
\begin{equation}
 \nabla_\theta k^K_{\theta,u}
 =-\sum_{v\in H_u}(q_v-\rho p_v)g_v,
 \label{eq:topk-kl-derivative}
\end{equation}
which amounts to modeling the sampler's tail as proportional to the trainer's, $\widehat q_v=\rho p_v$. Sequence regression substitutes TopK-KL into the residual and differentiates through it:
\begin{equation}
 D^K_\theta:=R-\beta\sum_u[\ell_{\theta,u}+k^K_{\theta,u}],
 \label{eq:bpo-topk-residual}
\end{equation}
\begin{equation}
 L^{\mathrm{bp,SRK}}_\theta:=-\sg(D^K_\theta)
 \sum_u[\log\pi_\theta(a_u\mid s_u)+k^K_{\theta,u}].
 \label{eq:bpo-topk-surrogate}
\end{equation}
Token regression uses the same correction as a score mean with a detached coefficient, following \citet[Section~4.1 and Appendix~A]{scorecentering2026}:
\begin{equation}
 C_\theta(s)=\sum_{v\in H}\sg(q_v-\rho p_v)\log p_v,\qquad
 L^{\mathrm{bp,TRK}}_{\theta,u}
 =-\sg(h_{\theta,u})[\log p_{a_u}-C_\theta(s_u)].
 \label{eq:token-sc-topk}
\end{equation}
The whole coefficient $q_v-\rho p_v$ is detached, including its dependence on the trainer through $\rho$.

\begin{proposition}[TopK-KL Surrogates]
\label{prop:bpo-topk}
Hold the sampler probabilities and the head fixed, and assume positive tail masses. Then Equation~\eqref{eq:topk-kl-derivative} holds, $\nabla_\theta C_\theta=-\nabla_\theta k^K_\theta$, and $L^{\mathrm{bp,SRK}}_\theta$ has the sample gradient of $(D^K_\theta)^2/(2\beta)$. With $\widetilde g^K_{\theta,u}=g_{\theta,u}+\nabla_\theta k^K_{\theta,u}$, the two routes produce
\begin{equation}
 \nabla_\theta^{\mathrm{AD}}L^{\mathrm{bp,SRK}}_\theta
 =-D^K_\theta\sum_u\widetilde g^K_{\theta,u},\qquad
 \nabla_\theta^{\mathrm{AD}}\sum_u L^{\mathrm{bp,TRK}}_{\theta,u}
 =-\sum_u(R-\beta\ell_{\theta,u})\widetilde g^K_{\theta,u}.
 \label{eq:sequence-token-feedback}
\end{equation}
\end{proposition}
The proof is in Appendix~\ref{app:bpo-topk-proof}. The two routes share the score correction and differ only in the coefficient; Appendix~\ref{app:topk-stability} gives numerical safeguards for small tail masses.

\paragraph{Binary KL.}
When only the sampled token's probability is stored, the coarsest partition separates that token from everything else. Following BPO's binary partition \citep[Equation~(27)]{bpo2026}, with $q_u=\pisamp(a_u\mid s_u)$ and $p_u=\pi_\theta(a_u\mid s_u)$, define
\begin{equation}
 k^{\mathrm{bin}}_{\theta,u}
 :=q_u\log\frac{q_u}{p_u}
 +(1-q_u)\log\frac{1-q_u}{1-p_u},\qquad
 \widehat D_\theta:=R-\beta\sum_u[\ell_{\theta,u}+k^{\mathrm{bin}}_{\theta,u}],\qquad
 \omega_{\theta,u}:=\frac{1-q_u}{1-p_u}.
 \label{eq:bpo-binary-residual}
\end{equation}
\begin{proposition}[Binary-KL Surrogates]
\label{prop:bpo-binary}
For $0<p_u,q_u<1$, with the sampler held fixed,
\begin{equation}
 \nabla_\theta[\log p_u+k^{\mathrm{bin}}_{\theta,u}]
 =\omega_{\theta,u}\nabla_\theta\log p_u.
 \label{eq:bpo-binary-derivative}
\end{equation}
Consequently the sequence surrogate
\begin{equation}
 L^{\mathrm{bp,SR\text{-}bin}}_\theta
 :=-\sg(\widehat D_\theta)\sum_u\sg(\omega_{\theta,u})\log p_u
 \label{eq:bpo-binary-surrogate}
\end{equation}
has the sample gradient of $\widehat D_\theta^2/(2\beta)$, and the token surrogate
\begin{equation}
 L^{\mathrm{bp,TR\text{-}bin}}_{\theta,u}
 :=-\sg(R-\beta\ell_{\theta,u})\,\sg(\omega_{\theta,u})\log p_u
 \label{eq:sc-binary-surrogate}
\end{equation}
applies the same corrected score with the per-token coefficient.
\end{proposition}
Here $k^{\mathrm{bin}}$ is Equation~\eqref{eq:topk-kl} with head $\{a_u\}$, which need not be the sampler's most probable token, and $\omega_{\theta,u}$ is a complementary-probability correction, not the importance ratio $p_u/q_u$. Binary KL depends on which action was sampled. The proof is in Appendix~\ref{app:bpo-binary-proof}, which also gives a three-token example in which the binary-corrected log-ratio has nonzero sampler mean. BPO's practical loss uses this partition after linearizing the residual and adds group normalization, smoothing, and clipping \citep[Section~3.2]{bpo2026}. The surrogates above keep the residual and the coefficient dynamic.

\subsection{Reusing Auxiliary Draws}
\label{app:mc-reuse}
Proposition~\ref{prop:mc-gradients} conditions on a trainer chosen independently of the auxiliary draws. Once the trainer has been updated with a set of draws, the draws are no longer independent of it. A conditionally unbiased gradient at every update therefore requires fresh draws from the historical sampler, which in turn requires its stored conditionals or the ability to rerun it. Reusing a fixed set of draws still defines a sensible training objective, a fixed sample-average surrogate, but the expectation identities of Section~\ref{sec:mc-kl} no longer apply to it.

\subsection{Single-Rollout Requirements}
Table~\ref{tab:single-rollout} lists what each instance needs to train from one rollout per prompt, as discussed in Section~\ref{sec:klpo-sklpo}.

\begin{table}[H]
\centering
\small
\renewcommand{\arraystretch}{1.2}
\begin{tabularx}{\textwidth}{@{}l>{\raggedright\arraybackslash}X>{\raggedright\arraybackslash}X@{}}
\toprule
With one response per prompt & Token-level KLPO & Response-level KLPO (SKLPO) \\
\midrule
Token-level critic & not needed & not needed \\
Prompt normalizer & not needed & supplied, estimated from independent data, or learned as a prompt-level predictor \\
Same-prompt group & not needed & needed for group centering or a log-mean-exp estimate, both of which degenerate for a single response \\
Per-token records & sampled-token log-probability, plus the KL estimator's records & sampled-token log-probability \\
Unbiased gradient & with full KL or MC-KL & given an exact normalizer \\
\bottomrule
\end{tabularx}
\caption{What each instance needs to train from one rollout per prompt. The comparison concerns information requirements.}
\label{tab:single-rollout}
\end{table}

%% file: figures/score-alternative.tex
\begingroup
\input{figures/style}

\tikzset{
  klpo fig flow/.style={
    -{Latex[length=2.0mm,width=1.35mm]},
    draw=klpoBlue,
    line width=.85pt
  }
}
\newcommand{\klpostepheading}[1]{%
  {\raggedright\sffamily\bfseries\fontsize{8.5}{10}\selectfont #1\par}%
}
\tcbset{
  klpo theory step/.style={
    enhanced,colback=klpoOrange!3,colframe=klpoOrange!30,
    boxrule=.4pt,arc=1.6mm,left=.9mm,right=.9mm,
    top=1.2mm,bottom=1.2mm,boxsep=0pt,before skip=0pt,after skip=0pt,
    valign=top
  },
  klpo practice step/.style={
    enhanced,colback=klpoBlue!3,colframe=klpoBlue!25,
    boxrule=.4pt,arc=1.6mm,left=.9mm,right=.9mm,
    top=1.2mm,bottom=1.2mm,boxsep=0pt,before skip=0pt,after skip=0pt,
    valign=top
  }
}

\begin{tcolorbox}[klpo practice step,colback=klpoOrange!4,colframe=klpoOrange!30]
\centering\sffamily\small\color{klpoInk}
\textbf{Sequence-regression alternative to Steps 7--9}\\[1mm]
{\scriptsize Start from Step 6; complete rollouts from the matched sampler.}\\[1mm]
{\scriptsize \(v_{u,j}\overset{\mathrm{iid}}{\sim}q(\cdot\mid s_u),\quad r_{u,j}:=\log\frac{q(v_{u,j}\mid s_u)}{p(v_{u,j}\mid s_u)},\quad M\geq2.\)}
\end{tcolorbox}
\vspace{1.5mm}
\noindent
\begin{minipage}[c]{.30\linewidth}
\begin{tcolorbox}[klpo practice step,equal height group=klpo-alternative-row]
\klpostepheading{\color{klpoBlue}Step \(7^\dagger\). KLPO Sequence Regression}
{\sffamily\tiny\color{klpoGray}terminal return; no value baseline}\par
\vspace{.3mm}
\centering\scriptsize\color{klpoInk}
\(\displaystyle
\begin{gathered}
D_\theta:=R-\beta\sum_u[\ell_{\theta,u}+k_{\theta,u}],\\[2pt]
\mathcal L_{\mathrm{SR}}:=\frac1{2\beta}\EE_{\pi_{\mathrm{sampler}}}D_\theta^2,\\[2pt]
\mathcal L_{\mathrm{SR}}
=\mathcal L_{\mathrm{tok}}+\text{constant in }\theta,\\[2pt]
\nabla\mathcal L_{\mathrm{SR}}=\nabla\mathcal L_{\mathrm{tok}}.
\end{gathered}
\)
\end{tcolorbox}
\end{minipage}%
\hfill
\begin{minipage}[c]{.035\linewidth}
\centering
\begin{tikzpicture}[baseline=-.5ex]
\draw[klpo fig flow] (0,0)--(.42,0);
\end{tikzpicture}
\end{minipage}%
\hfill
\begin{minipage}[c]{.30\linewidth}
\begin{tcolorbox}[klpo practice step,equal height group=klpo-alternative-row]
\klpostepheading{\color{klpoBlue}Step \(8^\dagger\). Sequence-Regression Backprop}
{\sffamily\tiny\color{klpoGray}differentiate the full local KL}\par
\vspace{.3mm}
\centering\scriptsize\color{klpoInk}
\(\displaystyle
\begin{gathered}
p_u:=\pi_\theta(y_u\mid x,y_{<u}),\\[1pt]
q_u:=\pi_{\mathrm{sampler}}(y_u\mid x,y_{<u}),\\[2pt]
L^{\mathrm{bp,SR}}_\theta:=-\sg(D_\theta)\\[-1pt]
{}\cdot\sum_u[\log p_u+k_{\theta,u}],\\[2pt]
\nabla_\theta^{\mathrm{AD}}\EE_{\pi_{\mathrm{sampler}}}L^{\mathrm{bp,SR}}_\theta
=\nabla_\theta\mathcal L_{\mathrm{tok}}.
\end{gathered}
\)
\end{tcolorbox}
\end{minipage}%
\hfill
\begin{minipage}[c]{.035\linewidth}
\centering
\begin{tikzpicture}[baseline=-.5ex]
\draw[klpo fig flow] (0,0)--(.42,0);
\end{tikzpicture}
\end{minipage}%
\hfill
\begin{minipage}[c]{.30\linewidth}
\begin{tcolorbox}[klpo practice step,equal height group=klpo-alternative-row]
\klpostepheading{\color{klpoBlue}Step \(9^\dagger\). Sequence Regression\\+ MC-KL}
{\sffamily\tiny\color{klpoGray}independent draws; leave-one-out feedback}\par
\vspace{.2mm}
\centering\scriptsize\color{klpoInk}
\(\displaystyle
\begin{gathered}
D_j:=R-\beta\sum_u[\ell_{\theta,u}+r_{u,j}],\\[2pt]
\overline D_{-j}:=\frac{1}{M-1}\sum_{l\ne j}D_l,\\[2pt]
L^{\mathrm{bp,SR\text{-}MC}}_\theta
:=-\frac1M\sum_j\sg(\overline D_{-j})\\[-1pt]
{}\cdot\sum_u[\log p_u-\log p(v_{u,j})].
\end{gathered}
\)
\end{tcolorbox}
\end{minipage}\par

\vspace{1.5mm}
\begin{tcolorbox}[klpo practice step,colback=klpoInk!4,colframe=klpoInk!12]
\centering\scriptsize\color{klpoInk}
\(\displaystyle
\nabla\mathcal L_{\mathrm{SR}}
=\nabla\mathcal L_{\mathrm{tok}}
=\nabla\mathcal L_{\mathrm{TR}}
=\nabla_\theta^{\mathrm{AD}}\EE_{\pi_{\mathrm{sampler}}}L^{\mathrm{bp,SR}}_\theta.
\)\\[1mm]
{\sffamily Exact routes: same population gradient; different sample estimators.}\\[1mm]
{\sffamily\color{klpoGray}MC-KL preserves the expected gradient; TopK-KL and binary KL approximate it.}\\[.6mm]
{\sffamily\color{klpoGray}Neither identity requires deterministic transitions.}
\end{tcolorbox}

\endgroup

%% file: contents/mc-proof.tex
\subsection{Monte Carlo KL and Leave-One-Out Residuals}
\label{app:mc-kl-proof}
\begin{proof}[Proof of Proposition~\ref{prop:mc-gradients}]
Condition on a complete rollout $Y$ and on a trainer $\theta$ chosen independently of the auxiliary draws. The visited histories, rollout actions, reward, and sampler conditionals are then fixed. Write $g_u(v)=\nabla_\theta\log p_\theta(v\mid s_u)$, and for each auxiliary column define
\[
 Z_j=\sum_u\log\frac{q(v_{u,j}\mid s_u)}{p_\theta(v_{u,j}\mid s_u)},\qquad
 G_j=\sum_u[g_u(a_u)-g_u(v_{u,j})].
\]
Their conditional means are $\EE_{\mathrm{MC}}Z_j=\sum_u k_\theta(s_u)$ and $\EE_{\mathrm{MC}}G_j=\sum_u\widetilde g_{\theta,u}$, and $D_j=R-\beta(\sum_u\ell_{\theta,u}+Z_j)$ has mean $D_\theta$. Token regression multiplies its MC score estimate by $h_{\theta,u}$, which is fixed given the rollout, so its conditional expectation is the full-KL sample gradient.

For sequence regression, the auxiliary columns are independent given $Y$, so $\overline D_{-j}$ is independent of $G_j$ and
\[
 \EE_{\mathrm{MC}}\left[-\frac1M\sum_j\overline D_{-j}G_j\right]
 =-D_\theta\sum_u\widetilde g_{\theta,u}.
\]
Since $\nabla_\theta D_j=-\beta G_j$, differentiating the sum of ordered cross-products in Equation~\eqref{eq:mc-sequence-ustat} gives exactly this estimator before taking expectations. For $j\ne l$, independence gives $\EE_{\mathrm{MC}}D_jD_l=D_\theta^2$, which proves the loss identity. Averaging over sampler rollouts then completes the proof under the original full-KL assumptions.
\end{proof}

\paragraph{The plug-in square is biased.}
Let $\overline D=M^{-1}\sum_jD_j$. With draws independent across prefixes,
\begin{equation}
 \EE_{\mathrm{MC}}\frac{\overline D^2}{2\beta}
 =\frac{D_\theta^2}{2\beta}
 +\frac{\beta}{2M}\sum_u
   \Var_{v\sim q(\cdot\mid s_u)}\left[\log\frac{q(v\mid s_u)}{p_\theta(v\mid s_u)}\right].
 \label{eq:mc-naive-bias}
\end{equation}
The extra variance term depends on the trainer, so differentiating the plug-in square does not give an unbiased full-KL gradient. For $M\geq2$, the U-statistic subtracts an unbiased estimate of that term:
\[
 \widehat{\mathcal L}_{\mathrm{SR,MC}}
 =\frac1{2\beta}\left[\overline D^2
 -\frac1{M-1}\frac1M\sum_j(D_j-\overline D)^2\right].
\]
Neither this quantity nor a single MC-KL estimate has to be nonnegative. The leave-one-out coefficient is $\overline D_{-j}=\overline D-(D_j-\overline D)/(M-1)$, so no explicit enumeration of pairs is needed. Splitting the draws into independent groups would also estimate the product, but the leave-one-out construction uses all $M$ draws symmetrically. With $M=1$ there is no independent second residual, so sequence regression is undefined, while token regression remains valid.

\paragraph{Reuse.}
All conditional identities above assume a trainer chosen independently of the auxiliary records. After the trainer has been updated with a fixed set of records, conditioning on the updated trainer can destroy their original i.i.d.\ law. Recomputing the current log-probabilities on the same records gives a fixed-sample surrogate, and drawing fresh auxiliary tokens from the same historical sampler restores the conditional unbiasedness. Sampling from the current trainer, deduplicating token IDs, clipping negative estimates, or reusing the rollout action each changes the estimator.

%% file: reference.bib
@inproceedings{kostrikov2022iql,
  title = {Offline Reinforcement Learning with Implicit {Q}-Learning},
  author = {Kostrikov, Ilya and Nair, Ashvin and Levine, Sergey},
  booktitle = {International Conference on Learning Representations},
  year = {2022},
  url = {https://arxiv.org/abs/2110.06169}
}

@inproceedings{garg2021iqlearn,
  title = {{IQ-Learn}: Inverse soft-{Q} Learning for Imitation},
  author = {Garg, Divyansh and Chakraborty, Shuvam and Cundy, Chris and Song, Jiaming and Geist, Matthieu and Ermon, Stefano},
  booktitle = {Advances in Neural Information Processing Systems},
  volume = {34},
  year = {2021},
  url = {https://arxiv.org/abs/2106.12142}
}

@inproceedings{rafailov2024qfunction,
  title = {From {$r$} to {$Q^*$}: Your Language Model is Secretly a {Q}-Function},
  author = {Rafailov, Rafael and Hejna, Joey and Park, Ryan and Finn, Chelsea},
  booktitle = {Conference on Language Modeling},
  year = {2024},
  url = {https://arxiv.org/abs/2404.12358}
}

@inproceedings{geist2019regularized,
  title = {A Theory of Regularized {Markov} Decision Processes},
  author = {Geist, Matthieu and Scherrer, Bruno and Pietquin, Olivier},
  booktitle = {Proceedings of the 36th International Conference on Machine Learning},
  volume = {97},
  pages = {2160--2169},
  year = {2019},
  publisher = {PMLR},
  url = {https://proceedings.mlr.press/v97/geist19a.html}
}

@inproceedings{haarnoja2017softq,
  title = {Reinforcement Learning with Deep Energy-Based Policies},
  author = {Haarnoja, Tuomas and Tang, Haoran and Abbeel, Pieter and Levine, Sergey},
  booktitle = {Proceedings of the 34th International Conference on Machine Learning},
  volume = {70},
  pages = {1352--1361},
  year = {2017},
  publisher = {PMLR},
  url = {https://proceedings.mlr.press/v70/haarnoja17a.html}
}

@inproceedings{rpg2026,
  title = {On the Design of {KL}-Regularized Policy Gradient Algorithms for {LLM} Reasoning},
  author = {Zhang, Yifan and Liu, Yifeng and Yuan, Huizhuo and Yuan, Yang and Gu, Quanquan and Yao, Andrew Chi-Chih},
  booktitle = {International Conference on Learning Representations},
  year = {2026},
  url = {https://arxiv.org/abs/2505.17508}
}

@inproceedings{gpo2025,
  title = {Beyond {Bradley--Terry} Models: A General Preference Model for Language Model Alignment},
  author = {Zhang, Yifan and Zhang, Ge and Wu, Yue and Xu, Kangping and Gu, Quanquan},
  booktitle = {Proceedings of the 42nd International Conference on Machine Learning},
  year = {2025},
  url = {https://arxiv.org/abs/2410.02197}
}

@article{bpo2026,
  title = {Bellman Policy Optimization},
  author = {Song, Zhuoqing and Xu, Haotian and Zhang, Xikun and Bing, Lidong},
  journal = {arXiv preprint arXiv:2609.15987},
  year = {2026},
  url = {https://arxiv.org/abs/2609.15987v1}
}

@article{scorecentering2026,
  title = {Score Centering Stabilizes Off-policy Reinforcement Learning},
  author = {Marek, Martin and Ryabinin, Max},
  journal = {arXiv preprint arXiv:2609.20807},
  year = {2026},
  url = {https://arxiv.org/abs/2609.20807v1}
}

@article{rebel2024,
  title = {{REBEL}: Reinforcement Learning via Regressing Relative Rewards},
  author = {Gao, Zhaolin and Chang, Jonathan D. and Zhan, Wenhao and Oertell, Owen and Swamy, Gokul and Brantley, Kiant{\'e} and Joachims, Thorsten and Bagnell, J. Andrew and Lee, Jason D. and Sun, Wen},
  journal = {arXiv preprint arXiv:2404.16767},
  year = {2024},
  url = {https://arxiv.org/abs/2404.16767}
}

@article{awr2019,
  title = {Advantage-Weighted Regression: Simple and Scalable Off-Policy Reinforcement Learning},
  author = {Peng, Xue Bin and Kumar, Aviral and Zhang, Grace and Levine, Sergey},
  journal = {arXiv preprint arXiv:1910.00177},
  year = {2019},
  url = {https://arxiv.org/abs/1910.00177}
}

@misc{flashreinforce2026,
  title = {{FlashREINFORCE}: Critic-Free Single-Rollout Asynchronous {RL} for Agentic Language Models},
  author = {Hu, Jian and Zhang, Yifan and Zhang, Hao and Xu, Binfeng and Zhang, Shaokun and Peng, Hongqing and Yu, Zhiding and Molchanov, Pavlo and Kautz, Jan and Dong, Yi},
  year = {2026},
  howpublished = {Manuscript},
  note = {Repository revision 774e4ccd, accessed September 19, 2026},
  url = {https://github.com/yifanzhang-pro/FlashREINFORCE-Overleaf/tree/774e4ccd79435e8251a5b1aff8366575467de4ea}
}

@article{sppo2024,
  title = {Self-Play Preference Optimization for Language Model Alignment},
  author = {Wu, Yue and Sun, Zhiqing and Yuan, Huizhuo and Ji, Kaixuan and Yang, Yiming and Gu, Quanquan},
  journal = {arXiv preprint arXiv:2405.00675},
  year = {2024},
  url = {https://arxiv.org/abs/2405.00675}
}

@inproceedings{peters2010reps,
  title = {Relative Entropy Policy Search},
  author = {Peters, Jan and M{\"u}lling, Katharina and Alt{\"u}n, Yasemin},
  booktitle = {Proceedings of the AAAI Conference on Artificial Intelligence},
  volume = {24},
  pages = {1607--1612},
  year = {2010}
}

@inproceedings{abdolmaleki2018mpo,
  title = {Maximum a Posteriori Policy Optimisation},
  author = {Abdolmaleki, Abbas and Springenberg, Jost Tobias and Tassa, Yuval and Munos, Remi and Heess, Nicolas and Riedmiller, Martin},
  booktitle = {International Conference on Learning Representations},
  year = {2018},
  url = {https://arxiv.org/abs/1806.06920}
}

@article{lan2023pmd,
  title = {Policy Mirror Descent for Reinforcement Learning: Linear Convergence, New Sampling Complexity, and Generalized Problem Classes},
  author = {Lan, Guanghui},
  journal = {Mathematical Programming},
  volume = {198},
  pages = {1059--1106},
  year = {2023}
}

@inproceedings{rafailov2023dpo,
  title = {Direct Preference Optimization: Your Language Model is Secretly a Reward Model},
  author = {Rafailov, Rafael and Sharma, Archit and Mitchell, Eric and Manning, Christopher D. and Ermon, Stefano and Finn, Chelsea},
  booktitle = {Advances in Neural Information Processing Systems},
  volume = {36},
  year = {2023},
  url = {https://arxiv.org/abs/2305.18290}
}

@article{schulman2017ppo,
  title = {Proximal Policy Optimization Algorithms},
  author = {Schulman, John and Wolski, Filip and Dhariwal, Prafulla and Radford, Alec and Klimov, Oleg},
  journal = {arXiv preprint arXiv:1707.06347},
  year = {2017},
  url = {https://arxiv.org/abs/1707.06347}
}

@article{shao2024deepseekmath,
  title = {{DeepSeekMath}: Pushing the Limits of Mathematical Reasoning in Open Language Models},
  author = {Shao, Zhihong and Wang, Peiyi and Zhu, Qihao and Xu, Runxin and Song, Junxiao and Bi, Xiao and Zhang, Haowei and Zhang, Mingchuan and Li, Y. K. and Wu, Y. and Guo, Daya},
  journal = {arXiv preprint arXiv:2402.03300},
  year = {2024},
  url = {https://arxiv.org/abs/2402.03300}
}
